%% file: main_arxiv.tex
\documentclass{article}
\usepackage{general_def}
\usepackage{specific_def}
\usepackage{arxiv/arxiv_pkg}
\usepackage{multirow}
\usepackage{amsmath,amssymb,amsthm, bbm}

\conffalse
\anonfalse

\usepackage[round]{natbib}

\newenvironment{compact}
    {
    \begin{normalsize}
    }
    { 
    \end{normalsize}
}

\title{Sharp Bounds for Federated Averaging (Local SGD) \\ and Continuous Perspective}

\begin{document}
\date{}

\author{%
  Margalit Glasgow\thanks{Equal contribution.} \\
  Stanford University \\
  \texttt{mglasgow@stanford.edu} 
  \and
  Honglin Yuan\footnotemark[1] \\
  Stanford University \\
  \texttt{yuanhl@cs.stanford.edu}
  \and
  Tengyu Ma \\
  Stanford University \\
  \texttt{tengyuma@stanford.edu}
  }
\maketitle

\input{sections/abstract}
\input{sections/intro}

\input{sections/related_work}

\input{sections/organization}

\input{sections/intuition} %
\input{sections/main_results} %
\input{sections/third_order_results} %
\input{sections/conclusion} %
\input{sections/ack}
\bibliography{refs}

\begin{appendices}
\listofappendices
\input{appendices/bias_bdd}
\input{appendices/lb_proof}
\input{appendices/ub_3o}

\end{appendices}
\end{document}

%% file: sections/abstract.tex
\begin{abstract}
\ifconf
    \vspace{-0.2cm}
\fi
Federated Averaging (\fedavg), also known as Local SGD, is one of the most popular algorithms in Federated Learning (FL). 
Despite its simplicity and popularity, the convergence rate of \fedavg has thus far been undetermined. 
Even under the simplest assumptions (convex, smooth, homogeneous, and bounded covariance), the best known upper and lower bounds do not match, and it is not clear whether the existing analysis captures the capacity of the algorithm. 
In this work, we first resolve this question by providing a lower bound for \fedavg that matches the existing upper bound, which shows the existing \fedavg upper bound analysis is not improvable.
Additionally, we establish a lower bound in a heterogeneous setting that nearly matches the existing upper bound. 
While our lower bounds show the limitations of \fedavg, under an additional assumption of third-order smoothness, 
we prove more optimistic state-of-the-art convergence results
in both convex and non-convex settings. 
Our analysis stems from a notion we call \emph{iterate bias}, which is defined by the deviation of the expectation of the SGD trajectory from the noiseless gradient descent trajectory with the same initialization. We prove novel sharp bounds on this quantity, and show intuitively how to analyze this quantity from a Stochastic Differential Equation (SDE) perspective\ifconf\footnote{The first two authors contributed equally.}\fi.

\end{abstract}

%% file: sections/intro.tex
\ifdefined\conf \section{INTRODUCTION} \else \section{Introduction}\fi
Federated Learning (FL) 
is an emerging distributed learning paradigm in which a massive number of clients collaboratively participate in the training process without disclosing their private local data to the public \citep{Konecny.McMahan.ea-NeurIPS15}. Typically, federated learning is orchestrated by a central server who oversees the clients, e.g. mobile devices or a group of organizations. The training process combines local training of a model at the clients with infrequent aggregation of the locally trained models at the central server.

Reflecting the goal of minimizing a loss function aggregated across clients, we consider the distributed optimization problem
$
    \min F(\x) := \frac{1}{M} \sum_{m=1}^M F_m(\x), 
$
where each client $m \in [M]$ holds a local objective $F_m$ realized by its local data distribution $\dist_m$, namely
$
F_m(\x) := \mathbb{E}_{\xi \sim \mathcal{D}_m} f(\x; \xi).
$
Federated Learning is \em heterogeneous \em by design as $\mathcal{D}_m$ can vary across clients. In the special case when $\mathcal{D}_m \equiv \mathcal{D}$ for all clients $m$, the problem is called \em homogeneous\em.

Federated Averaging (\fedavg, \citealt{McMahan.Moore.ea-AISTATS17}), also known as Local SGD (\citealt{Stich-ICLR19}), is one of the most popular algorithms applied in Federated Learning. 
In its simplest form,\footnote{We discuss other extensions of \fedavg in \cref{sec:related-work}.} \fedavg  proceeds in $R$ communication rounds, where at the beginning of each round, a central server sends the current iterate to each of the $M$ clients. Each client then locally takes $K$ steps of SGD, and then returns its final iterate to the central server. The central server averages these iterates to obtain the first iterate of the next round. We state the \fedavg algorithm formally in \cref{alg:fedavg}.

\input{tables/table_rates}

While the \fedavg algorithm is popular in practice, a thorough theoretical understanding of \fedavg has not been established. 
Even under the simplest setting (convex, smooth, homogeneous and bounded covariance, see \cref{asm:main}), the state-of-the-art upper bounds for \fedavg due to \citet{Khaled.Mishchenko.ea-AISTATS20} and \citet{Woodworth.Patel.ea-ICML20} do not match the state-of-the-art lower bound due to \citet{Woodworth.Patel.ea-ICML20}, see \cref{table:complexity}.
This suggests that at least one side of the analysis is not sharp. 
Therefore a fundamental question remains:

\emph{Does the current convergence analysis of \fedavg fully capture the capacity of the algorithm?}

Our first contribution is to answer this question definitively under the standard smoothness and convexity assumptions. 
We establish a sharp lower bound for \fedavg that matches the existing upper bound (\cref{lb:homo}), showing that the existing \fedavg analysis is \emph{not} improvable.
Moreover, we establish a stronger lower bound in the \emph{heterogeneous} setting, \cref{thm:lb}, which suggests the best known \emph{heterogeneous} upper bound analysis \citep{Woodworth.Patel.ea-NeurIPS20} is also (almost)\footnote{Up to a minor variation of the definition of heterogeneity measure, see \cref{table:complexity}.} not improvable.

\input{sections/alg_fedavg}

Our proofs highlight exactly what can go wrong in \fedavg, yielding these slow convergence rates. 
Specifically, our lower bound analysis stems from a notion we call \emph{iterate bias}, which is defined by the deviation of the expectation of the SGD trajectory from the (noiseless) gradient descent trajectory with the same initialization (see \cref{def:bias} for details).
We show that even for convex and smooth objectives, the mean of SGD initialized at the optimum can drift away from the optimum at the rate of $\Theta(\eta^2 k^{\frac{3}{2}})$ after $k$ steps,\footnote{This rate is also sharp according to our matching upper and lower bounds, see \cref{thm:2o:bias:ub,thm:2o:bias:lb} for details.} for sufficiently small learning rate $\eta$.
We depict this phenomenon in \cref{fig:bias}.\ifanon\else\footnote{Code repository see \url{https://github.com/hongliny/Sharp-Bounds-for-FedAvg-and-Continuous-Perspective}.}\fi 
The iterate bias thus quantifies the fundamental difficulty encountered by \fedavg: 

\emph{Even with infinite number of homogeneous clients, \fedavg can drift away from the optimum even if initialized at the optimum. }

Indeed, we show in \cref{sec:bias_to_lb} that the sharp lower bound of SGD iterate bias leads directly to our sharp lower bound of \fedavg convergence rate.
\begin{figure*}[t]
    \centering
    \begin{subfigure}{0.2\textwidth}
        \centering
        \vspace{0.1cm}
        \includegraphics[width=\textwidth]{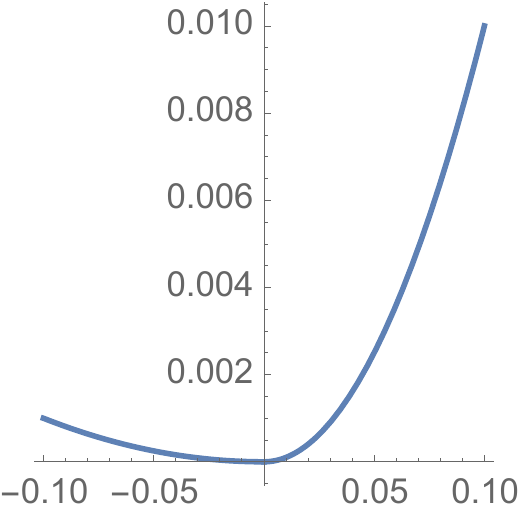}
        \vspace{0.05cm}
        \caption{Function plot}
        \label{fig:piecewise:quadratic}
    \end{subfigure}
    \hfill
    \begin{subfigure}{0.78\textwidth}
        \includegraphics[width=\textwidth]{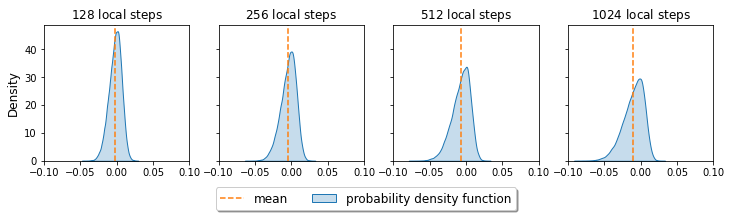}
        \caption{Probability density function after various SGD steps}
    \end{subfigure}
    \caption{\textbf{Illustration of the iterate bias of SGD.}  Consider the objective $F(x) = \begin{cases}x^2 & x \geq 0 \\
        \frac{1}{10} x^2 & x < 0\end{cases}$ as shown in (a), and $f(x; \xi) := \xi x + F(x)$ where $\xi \sim \mathcal{N}(0, 0.01)$.
        We initialize the SGD at optimum $x^{\star}=0$, and run 1024 steps of SGD with step size $10^{-2}$. We repeat this random process for 65536 times, and estimate the density function after 128, 256, 512 and 1024 steps. 
        Observe that the density function and the average gradually move to the left (away from the optimum, where the curvature is smaller). This figure explains the intrinsic difficulty for \fedavg to handle objective with drastic Hessian change.}
    \label{fig:bias}
\end{figure*}

The discouraging lower bound of $\fedavg$ under a standard smoothness assumption does not conform well with its empirical efficiency observed in practice \citep{Lin.Stich.ea-ICLR20}. This motivates us to consider whether additional modeling assumptions could better explain the empirical performance of \fedavg. The aforementioned lower bound is attained by a special piece-wise quadratic function with a sudden curvature change, which is smooth (has bounded second-order derivatives) but has unbounded third-order derivatives. A natural assumption to exclude this corner case is third-order smoothness, which has been considered before in the context of federated learning~\citep{Yuan.Ma-NeurIPS20}, and may be representative of objectives in practice. For instance, loss functions used to learn many generalized linear models, such as logistic regression, often exhibit third-order smoothness \citep{Hastie.Tibshirani.ea-09}.

With this third-order smoothness assumption, we show that the iterate bias reduces to $\Theta(\eta^3 k^2)$, one order higher in $\eta$ than the rate under only second-order smoothness.\footnote{This rate is sharp according to our matching upper and lower bounds, see \cref{thm:3o:bias:ub,thm:3o:bias:lb}.} While the proofs for bounding the iterate bias are quite technical, we show that it is easy to analyze the bias via a continuous approach. More specifically, by studying the stochastic differential equation (SDE) corresponding to the continuous limit of SGD, one can derive the limit of the iterate bias of generic objectives by using the Kolmogorov backward equation of the SDE, see \cref{sec:sde}.

Leveraging this intuition from the bias, we prove state-of-the-art rates for \fedavg under third-order smoothness in \emph{both} convex and non-convex settings (Theorems~\ref{thm:3o_convex} and \ref{thm:non_convex}). 
In non-convex settings, our convergence rate scales with $1/R^{\frac{4}{5}}$, which improves upon the best known rate of $1/R^{\frac{2}{3}}$ \citep{yu2019parallel} if we do not assume third-order smoothness.

%% file: tables/table_rates.tex
\begin{table*}[ht]
    \centering
    \caption{\textbf{Convergence Rates of \fedavg}. Some lower order terms as $R \rightarrow \infty$ omitted.
$H$: smoothness, $R$: number of rounds, $K$: local iterations per round, $M$: number of clients, $\sigma$: noise, $D: \|\x^{(0, 0)} - \x^{\star}\|.$ The lower and upper bound use a slightly different metric of heterogeneity ($\zeta$ and $\zeta_*$), see  \cref{rem:hetero} for details. 
We bold the terms where our analysis improves upon previous work.
}\label{table:complexity}
\begin{compact}
\begin{tabular}{@{}lll@{}} \toprule
             & Homogeneous (Assumption~\ref{asm:main}) & Heterogeneous (Assumption~\ref{asm:main} and \ref{asm:hetero})\\ \midrule
    Previous Upper Bound &  $\frac{HD^2}{KR} + \frac{ \sigma D}{\sqrt{MKR}} + \frac{H^{\frac{1}{3}} \sigma^{\frac{2}{3}} D^{\frac{4}{3}}} {K^{\frac{1}{3}} R^{\frac{2}{3}}}$   &  $\frac{HD^2}{KR} + \frac{ \sigma D}{\sqrt{MKR}} + \frac{H^{\frac{1}{3}} \sigma^{\frac{2}{3}} D^{\frac{4}{3}}} {K^{\frac{1}{3}} R^{\frac{2}{3}}} + \frac{H^{\frac{1}{3}}\zeta^{\frac{2}{3}}D^{\frac{4}{3}}}{R^{\frac{2}{3}}}$  \vspace{0.15cm}  \\ 
                         & \citep{Khaled.Mishchenko.ea-AISTATS20} & \citep{Khaled.Mishchenko.ea-AISTATS20,Woodworth.Patel.ea-NeurIPS20} \\ \midrule
    \textbf{Our Lower Bound} &    ${\bf{\frac{HD^2}{KR}}} + \frac{ \sigma D}{\sqrt{MKR}} + \frac{H^{\frac{1}{3}} \sigma^{\frac{2}{3}} D^{\frac{4}{3}}} {{\bf{K^{\frac{1}{3}}}} R^{\frac{2}{3}}}$ 
    & $\frac{HD^2}{KR} + \frac{ \sigma D}{\sqrt{MKR}} + \frac{H^{\frac{1}{3}} \sigma^{\frac{2}{3}} D^{\frac{4}{3}}} {{\bf{K^{\frac{1}{3}}}} R^{\frac{2}{3}}} + {\bf{\frac{H^{\frac{1}{3}} \zeta^{\frac{2}{3}}_*D^{\frac{4}{3}}}{R^{\frac{2}{3}}}}}$   \vspace{0.15cm}  \\
                & \cref{lb:homo} & \cref{thm:lb} \\ \midrule
    Previous Lower Bound &  $\frac{ \sigma D}{\sqrt{MKR}} + \frac{H^{\frac{1}{3}} \sigma^{\frac{2}{3}} D^{\frac{4}{3}}} {{\bf{K^{\frac{2}{3}}}} R^{\frac{2}{3}}} $ & $\frac{ \sigma D}{\sqrt{MKR}} +  \frac{H^{\frac{1}{3}} \sigma^{\frac{2}{3}} D^{\frac{4}{3}}} {{\bf{K^{\frac{2}{3}}}} R^{\frac{2}{3}}} + {\bf{\min\left(\frac{HD^2}{R},\frac{H^{\frac{1}{3}} \zeta^{\frac{2}{3}}_*D^{\frac{4}{3}}}{R^{\frac{2}{3}}}\right)}}$   \vspace{0.15cm}  \\
                & \citep{Woodworth.Patel.ea-ICML20} & \citep{Woodworth.Patel.ea-NeurIPS20} \\ 
    \bottomrule
\end{tabular}
\end{compact}
\end{table*}

%% file: sections/alg_fedavg.tex
\begin{algorithm}[t]
  \caption{\fedavgfull (\fedavg)}
  \label{alg:fedavg}
  \begin{algorithmic}[1]
  \STATE {\textbf{procedure}} \fedavg ($\x^{(0, 0)}, \eta$)
  \FOR {$r=0, \ldots, R-1$}
      \FORALL {$m \in [M]$ {\bf in parallel}}
        \STATE $\x_m^{(r,0)} \gets \x^{(r,0)}$ \COMMENT{broadcast current iterate}
        \FOR {$k = 0, \ldots, K-1$}
          \STATE $\xi^{(r,k)}_m \sim \mathcal{D}_m$
          \STATE $\g^{(r,k)}_m \gets \nabla f(\x^{(r,k)}_m; \xi^{(r,k)}_m)$
          \STATE $\x^{(r, k+1)}_m \gets \x^{(r, k)}_m - \eta \cdot \g_m^{(r,k)}$
          \COMMENT{client update} 
        \ENDFOR
      \ENDFOR
      $\x^{(r+1, 0)} \gets \frac{1}{M} \sum_{m=1}^M \x^{(r,K)}_m$
      \COMMENT{server averaging}
  \ENDFOR
\end{algorithmic}
\end{algorithm}

%% file: sections/related_work.tex
\subsection{Related Work}
\label{sec:related-work}
\paragraph{\fedavg Analysis.}
The understanding of local updates algorithm such as \fedavg is one of the most important topics in distributed optimization.
The early analysis of $\fedavg$ preceded the proposal of Federated Learning, typically under the name of Local SGD or parallel SGD \citep{Mcdonald.Mohri.ea-NIPS09,Zinkevich.Weimer.ea-NIPS10,Shamir.Srebro-Allerton14,Rosenblatt.Nadler-16,Jain.Kakade.ea-COLT18,Zhou.Cong-IJCAI18}. 
The primary focus of this literature is the special case of one-shot averaging, in which only one round of averaging (communication) is conducted at the end of the algorithm. 
The first upper bound of $\fedavg$ (with multiple averaging rounds) was established by  \citet{Stich-ICLR19} in the convex homogeneous setting, which imposes uniform gradient bound assumption. 
The result was further improved by \citet{Khaled.Mishchenko.ea-AISTATS20,Woodworth.Patel.ea-ICML20} with improved rates and relaxed assumptions. Note that the result of \citet{Khaled.Mishchenko.ea-AISTATS20} preceded that of \citet{Woodworth.Patel.ea-ICML20}, but the step-size was not properly optimized.
In the convex heterogeneous setting, the first upper bound of $\fedavg$ is due to \citet{Li.Huang.ea-ICLR20}. This result was improved by \citet{Khaled.Mishchenko.ea-AISTATS20,Woodworth.Patel.ea-NeurIPS20}. 
For non-convex objectives, a series of recent works \citep{Zhou.Cong-IJCAI18,Haddadpour.Kamani.ea-ICML19,Wang.Joshi-18,Yu.Jin-ICML19,Yu.Jin.ea-ICML19} has established various upper bounds of \fedavg in homogeneous and heterogeneous settings. 
On the lower bound side, the best known result for convex $\fedavg$ is established by \citet{Woodworth.Patel.ea-ICML20} (for homogeneous) and \citet{Woodworth.Patel.ea-NeurIPS20} (for heterogeneous). To the best of our knowledge, we are unaware of any lower bound for \fedavg in non-convex settings.

The existence and effect of iterate bias has been observed in various forms in the current literature \citep{Dieuleveut.Durmus.ea-20,Charles.Konecny-20,Woodworth.Patel.ea-ICML20}, yet our paper is the first to sharply characterize the rate of the bias, both in the second-order smooth case and third-order smooth case.

\paragraph{Other Extensions of \fedavg.}
Throughout this work, we study the simplest form of $\fedavg$ (\cref{alg:fedavg}) to keep our efforts focused.
There are many other extensions of $\fedavg$ applied in practice. 
For example, instead of letting all the clients to participate in computation, one may randomly draw a subset of clients to participate every round. 
Most of our results (e.g., all of the homogeneous results) can be directly extended to this sub-sampling variant. 
Other variants of \fedavg include letting clients run different number of steps per round, or average the client states non-uniformly.
We refer readers to \citet{Wang.Charles.ea-21} for a more comprehensive survey of these extensions.

One special extension of \fedavg introduces a ``server learning rate'', which we name \textsc{FedAvg-SLR} to distinguish it. Instead of taking the client averaging as the initialization of the next round, \textsc{FedAvg-SLR} extrapolates (or interpolates) between the round initialization and the clients averaging \citep{Charles.Konecny-20,Reddi.Charles.ea-ICLR21}, or formally
$
\x^{(r+1,0)} \gets \x^{(r,0)} + \eta_{s} \cdot \frac{1}{M} \sum_{m=1}^M \left( \x_m^{(r,K)} - \x^{(r,0)} \right).
$
By definition, \textsc{FedAvg-SLR} reduces to the classic \fedavg when $\eta_{s} = 1$. 
Notably, \textsc{FedAvg-SLR} also reduces to mini-batch SGD when the client learning rate goes to 0, and the server learning rate goes to infinity. 
Therefore, \textsc{FedAvg-SLR} can at least attain the best convergence rate of the classic \fedavg and the mini-batch SGD.
However, to the best of our knowledge, it is not known whether \textsc{FedAvg-SLR} can outperform the best of the two, and there are \emph{no} results that characterize \textsc{FedAvg-SLR} beyond the two special regimes.
While our lower bounds do \emph{not} apply to the generic form of \textsc{FedAvg-SLR}, we  anticipate that our techniques (e.g., iterate bias) can be applicable to the study of \textsc{FedAvg-SLR} in the future work.

\paragraph{Other Federated Learning Algorithms.}
Besides the $\fedavg$ framework, there are many other federated optimization algorithms that aims to improve communication efficiency \citep{Yuan.Ma-NeurIPS20,Reddi.Charles.ea-ICLR21} or tackle the heterogeneity in FL \citep{Li.Sahu.ea-MLSys20,Karimireddy.Kale.ea-ICML20}. 
We expect the techniques developed in this work can shed light on the analysis of broader existing federated algorithms and promote the design of more efficient federated algorithms. 

The past half decade has witnessed a booming interest in various aspects of Federated Learning. 
Data heterogeneity is one of the most important patterns in Federated Learning, and is known to cause performance degradation in practice \citep{Hsu.Qi.ea-arXiv19}. Numerous existing works have aimed to understand and mitigate the negative effect of heterogeneity in various ways \citep{Mohri.Sivek.ea-ICML19,Liang.Shen.ea-19,Chen.Chen.ea-NeurIPS20,Deng.Kamani.ea-20,Li.Sahu.ea-20,Reisizadeh.Mokhtari.ea-AISTATS20,Wang.Liu.ea-AAAI19,Pathak.Wainwright-NeurIPS20,Zhang.Hong.ea-20,Yuan.Zaheer.ea-ICML21,Acar.Zhao.ea-ICLR21,Al-Shedivat.Gillenwater.ea-ICLR21,Yuan.Morningstar.ea-21}. 
In practice, the system heterogeneity will also affect the performance of Federated Learning \citep{Smith.Chiang.ea-NIPS17,Diao.Ding.ea-ICLR21}. 
In deep learning context, a recent array of works has studied the alternative approaches of model ensembling beyond averaging in parameter space \citep{Bistritz.Mann.ea-NeurIPS20,He.Annavaram.ea-NeurIPS20,Lin.Kong.ea-NeurIPS20,Chen.Chao-ICLR21,Yoon.Shin.ea-ICLR21}.

This paper mainly considers the classic FL settings in which the same model is learned from and deployed to all the clients. 
There is an alternative setup in FL, known as the \emph{personalized} setting, which aims to learn a different (personalized) model for different clients. 
Numerous recent papers have proposed Federated Learning models and algorithms to accommodate personalization  
\citep{Smith.Chiang.ea-NIPS17,Jiang.Konecny.ea-19,Chen.Luo.ea-19,Fallah.Mokhtari.ea-NeurIPS20,Hanzely.Hanzely.ea-NeurIPS20,London-SpicyFL20,T.Dinh.Tran.ea-NeurIPS20,Hanzely.Richtarik-20,Agarwal.Langford.ea-20,Lin.Yang.ea-20,Deng.Kamani.ea-20,pmlr-v139-acar21a}. 
We anticipate the techniques developed in this work can be applied to personalized FL algorithms, especially the ones that applied local updates approach.

We refer readers to \citep{Kairouz.McMahan.ea-arXiv19,Wang.Charles.ea-21} for a more comprehensive survey on the recent progress of Federated Learning.

\paragraph{Connection to Implicit Bias.} 
It is possible to view the iterate bias as an implicit bias of the \fedavg algorithm, which pushes the iterate towards flatter regions of the objective. This effect is similar to other instances of implicit bias observed for stochastic gradient descent, which has drawn connections between noise in the gradients and flat minima~\citep{hochreiter1997flat,jastrzkebski2017three,blanc2020implicit,damian2021label}. While in many instances, implicit bias has been linked to choosing favorable optima that generalize well \citep{neyshabur2017implicit}, in our setting, the bias affects the convergence rate.

%% file: sections/organization.tex
\subsection{Organization and Notation} 
In Section 2, we formally define the iterate bias of SGD, and state sharp bounds on its rate.
In Section 3, we state our lower bounds for \fedavg, and show how the iterate bias can be used to achieve our sharp bounds.
In Section 4, we state our convergence results for \fedavg under third-order smoothness. All proofs are deferred to the appendix.

We use bold lower case character to denote vectors (e.g., $\x$). We use $\|\cdot\|$ to denote the  $\ell_2$-norm of a vector, $[n]$ to denote the set $\{1, \ldots, n\}$. 
Throughout the paper, we use $O, \Omega, \Theta$ notation to hide absolute constants only.

%% file: sections/intuition.tex
\ifdefined\conf \section{SETUP AND TECHNICAL OVERVIEW: INTUITION FROM ITERATE BIAS}\else \section{Setup and Technical Overview: Intuition From Iterate Bias of SGD}\fi
\label{sec:bias}

The intuition from our lower bound comes from studying the behaviour of \fedavg when the number of clients, $M$, tends to infinity. In this case, the averaged iterate $\x^{(r + 1, 0)}$ is precisely the \em expected \em iterate after $K$ iterations of SGD starting from the last averaged iterate, $\x^{(r, 0)}$. This motivates the following definition.

\begin{definition}[Iterate Bias of SGD]
\label{def:bias}
Let $\{\x_{\sgd}^{(k)}\}_{k=0}^{\infty}$ and $\{\z_{\gd}^{(k)}\}_{k=0}^{\infty}$ be the trajectories of SGD and GD initialized at the same point $\x$, formally
\begin{compact}
\begin{alignat}{3}
    & \x^{(k+1)}_{\sgd} \gets \x_{\sgd}^{(k)} - \eta \nabla f(\x_{\sgd}^{(k)}; \xi^{(k)}), \qquad & \x_{\sgd}^{(0)} = \x;
    \\
    & \z^{(k+1)}_{\gd} \gets \z_{\gd}^{(k)} - \eta \nabla F(\z_{\gd})
    , \qquad  & \z_{\gd}^{(0)} = \x.
\end{alignat}
\end{compact}
The \textbf{iterate bias} (or in short ``bias'') from $\x$ at the $k$-th step is defined as 
\begin{compact}
\begin{equation}
    \expt \x_{\sgd}^{(k)} - \z_{\gd}^{(k)},
\end{equation}
\end{compact}
the difference between the mean of SGD trajectory and the (deterministic) GD trajectory.
\end{definition}

One important special case of Definition~\ref{def:bias} is the iterate bias from a stationary point  $\x^{\star}$. In this case, the gradient descent trajectory $\z_{\gd}^{(k)}$ will stay at the optimum since $\nabla F(\z_{\sgd}^{(k)}) \equiv \nabla F(\x^{\star}) = \mathbf{0}$. 
The iterate bias then reduces to $\expt [\x_{\sgd}^{(k)}] - \x^{\star}$. 
Notably, even for convex smooth objectives $f$, the expected iterate $\expt[\x_{\sgd}^{(k)}]$ may drift away from the optimum $\x^{\star}$, even if initialized at the $\x^{\star}$.
This occurs because of a difference between the gradient of the expectation of an iterate, $\nabla f(\mathbb{E}[\cdot])$, and the expectation of the gradient of the iterate, $\expt [\nabla f(\cdot)]$.

In \cref{fig:bias}, we illustrate this phenomenon via a one-dimensional objective. This figure, and our formal results below, illustrate that for sufficiently small step sizes, the bias increases in $k$. For this reason, doing more than one local step can sometimes be counterproductive (when $k=1$, the bias is always zero).
This phenomenon is key to the poor dependence on $K$ in the convergence rate we prove for \fedavg.

\subsection{The Bias Under Second-Order Smoothness}

In this subsection, we provide sharp bounds on the iterate bias under standard assumptions, formally given below.
\begin{assumption}
    \label{asm:main}
    Assume $f(\x; \xi)$ is second-order differentiable w.r.t. $\x$, and
    \begin{enumerate}[(a)]
        \item Convexity: $f(\x;\xi)$ is convex with respect to $\x$ for any $\xi$.
        \item Smoothness: $f(\x; \xi)$ is $H$-smooth with respect to $\x$. That is, for any $\xi$, for any $\x, \y$, we have $\|\nabla f(\x; \xi) - \nabla f(\y; \xi)\|_2 \leq H\|\x - \y\|_2$.
        \item Bounded covariance: for any $\x$, 
        $\expt_{\xi \sim \mathcal{D}} \|\nabla f(\x, \xi) - \nabla F(\x)\|_2^2 \leq \sigma^2.$
    \end{enumerate}
\end{assumption}

We establish the following upper bound on the bias.\footnote{Throughout this section, we mainly focus on the iterate bias bound in the regime of sufficiently small $\eta$ for simplicity and easy comparison. Our complete theorem in appendix covers the case of general $\eta$ choice.}
\begin{theorem}[Simplified from \cref{thm:2o:bias:ub:complete}]
    \label{thm:2o:bias:ub}
    Under \cref{asm:main}, there exists an absolute constant $\bar{c}$ such that for any initialization $\x$, for any $\eta \leq \frac{1}{H}$, the iterate bias satisfies
    $\left \| 
            \expt \x_{\sgd}^{(k)} - \z_{\gd}^{(k)}
        \right\|_2
        \leq
        \bar{c}
        \cdot
        \eta^2 k^{\frac{3}{2}}  H \sigma.
    $
\end{theorem}
In fact, we show in the following theorem that this upper bound of iterate bias is sharp.
\begin{theorem}[Simplified from \cref{thm:2o:bias:lb:complete}]
    \label{thm:2o:bias:lb}
    There exists an absolute constant $\underline{c}$ such that for any $H, \sigma$,  there exists an objective $f(\x; \xi)$ and distribution $\xi \sim \dist$ satisfying \cref{asm:main} such that for any integer $K$, for any $\eta \leq \frac{1}{2KH}$, and integer $k \in [2, K]$,  the iterate bias from the optimum $\x^{\star}$ of $F$ is lower bounded as
    $
        \left \| 
            \expt \x_{\sgd}^{(k)} - \z_{\gd}^{(k)}
        \right\|_2
        \geq
        \underline{c} \cdot \eta^2 k^{\frac{3}{2}} H \sigma.
    $
\end{theorem}
\cref{thm:2o:bias:lb} shows that the SGD trajectory can indeed drift away (in expectation) from the optimum $\x^{\star}$ despite being initialized at $\x^{\star}$.
Our lower bound improves over the best known lower bound $\Omega(\eta^2 k H \sigma)$
due to \citet{Woodworth.Patel.ea-ICML20}. 
The lower bound is attained by running SGD with Gaussian noise on the piecewise quadratic function $f(x) := \begin{cases}\frac{1}{2} H x^2 & x \geq 0, \\
     \frac{1}{4}Hx^2 & x < 0.\end{cases}$, first analyzed in \citet{Woodworth.Patel.ea-ICML20}.

Recall that the bias originates from the difference between $\nabla f(\mathbb{E}[\x_{\sgd}^{(k)} ])$ and $\mathbb{E}[\nabla f(\x_{\sgd}^{(k)})]$. 
This piecewise quadratic function has an unbounded third order derivative at $0$, which causes this difference to be large whenever the distribution of $\x_{\sgd}^{(k)}$ spans both sides of $0$. This worst case construction motivates our further study of the bias under a third-order derivative bound.

\subsection{The Bias Under Third-Order Smoothness}
We formally state our third-order smoothness condition in the following assumption.
\begin{assumption}\label{asm:third_order}
Assume $f(\x;\xi)$ is third-order differentiable w.r.t. $\x$ for any $\xi$, and
\begin{enumerate}[(a)]
    \item $f(\x; \xi)$ is $Q$-3rd-order-smooth, i.e. for any $\xi$, for any $\x, \y$, $
        \| \nabla^2 f(\x; \xi) - \nabla^2 f(\y; \xi) \|_2 \leq Q \| \x - \y\|_2$.
    \item $\nabla f(\x, \xi)$ has $\sigma^4$-bounded 4th order central moment, i.e. for all $\x$, $\mathbb{E}_{\xi}\left[\left\|\nabla f(\x, \xi)- \nabla F(\x)]\right\|^4\right] \leq \sigma^4$.
 \end{enumerate}
\end{assumption}

We show that under this additional assumption, the iterate bias reduces to $\bigo(\eta^3 k^2 Q \sigma^2)$, which scales on the order of $\eta^3$ (rather than $\eta^2$) as $\eta$ goes to 0.
\begin{theorem}[Simplified from \cref{thm:3o:bias:ub:complete}]
\label{thm:3o:bias:ub}
    Under \cref{asm:main,asm:third_order}, there exists an absolute constant $\bar{c}$ such that for any initialization $\x$, for any $\eta \leq \frac{1}{2H}$, the iterate bias satisfies 
    $
        \left \| 
            \expt \x_{\sgd}^{(k)} - \z_{\gd}^{(k)}
        \right\|_2
        \leq
        \bar{c}
        \cdot
        \eta^3 k^2 Q \sigma^2.
    $
\end{theorem}
\cref{thm:3o:bias:ub} also reveals the dependency on the third-order smoothness $Q$. In the extreme case where $Q = 0$ ($f$ is quadratic), the iterate bias will disappear. It is worth noting that since \cref{asm:main} is still required in \cref{thm:3o:bias:ub}, the original upper bound $\bigo(\eta^2 k^{\frac{3}{2}} H \sigma)$ from \cref{thm:2o:bias:ub} still applies, and one can formulate the upper bound as the minimum of the two. 

The following lower bound shows that the upper bound in \cref{thm:3o:bias:ub} is sharp.
\begin{theorem}[Simplified from \cref{thm:3o:bias:lb:complete}]
    \label{thm:3o:bias:lb}
      There exists an absolute constant $\underline{c}$ such that for any $H, \sigma, K$, for any sufficiently small $Q$ (polynomially dependent on $H, \sigma, K$), there exists an objective $f(\x; \xi)$ and distribution $\xi \sim \dist$ satisfying \cref{asm:main,asm:third_order} such that for any $\eta \leq \frac{1}{2HK}$ and integer $k \in [2, K]$, the iterate bias from the optimum $\x^{\star}$ is lower bounded as
    $
        \left \| 
           \expt \x_{\sgd}^{(k)} - \z_{\gd}^{(k)}
        \right\|_2
        \geq
        \underline{c} \cdot \eta^3 k^2 Q \sigma^2.
    $
\end{theorem}

\subsection{Revealing Iterate Bias Via Continuous Perspective}
\label{sec:sde}
While the proofs of the results above are quite technical, the intuition for these bounds is much easier to see in a continuous view of SGD. As an example, we demonstrate how the $\Theta(\eta^3 k^2 Q \sigma^2)$ term shows up in \cref{thm:3o:bias:ub,thm:3o:bias:lb}. 

Consider a one-dimensional instance of SGD with Gaussian noise, where $f(x; \xi) = F(x) - \xi x$, and $\xi \sim \mathcal{N}(0, \sigma^2)$.
The SGD then follows
\begin{compact}
    \begin{equation}
        x^{(k+1)}_{\sgd} = x^{(k)}_{\sgd} - \eta \nabla F(x^{(k)}_{\sgd}) + \eta \xi^{(k)}, ~~ \text{where } \xi^{(k)} \sim \mathcal{N}(0, \sigma^2).
        \label{eq:sde:sgd}
    \end{equation}
\end{compact}
The continuous limit of \eqref{eq:sde:sgd} corresponds to the following SDE, with the scaling $t = \eta k$:
\begin{compact}
\begin{equation}
    \diff X(t) = - F'(X(t)) \diff t + \sqrt{\eta} \sigma \diff B_t,
    \label{eq:sde}
\end{equation}
\end{compact}
where $B_t$ denotes the Brownian motion (also known as the Wiener process).\footnote{To justify the relation of \cref{eq:sde:sgd} and \cref{eq:sde}, note that \cref{eq:sde:sgd} can be viewed as a numerical discretization (Euler-Maruyama discretization \citep{Kloeden.Platen-92}) of the SDE \eqref{eq:sde} with time step-size $\eta$.}

To get a handle of the iterate bias, our goal is to study $\expt [X(t) | X(0) = x]$, the expectation of the SDE solution $X(t)$ initialized at $x$. We view this quantity as a multivariate function $u(t, x)$ of $t$ and $x$, with the objective to Taylor expand $u(t, x)$ around $u(0, x)$ in $t$:
\begin{compact}
 \begin{equation}
     u(t, x) = u(0, x) + u_t(0,x) t + \frac{1}{2}u_{tt}(0,x) t^2 + o(t^2).
 \end{equation}
\end{compact}
For brevity, we use subscript notation to denote partial derivatives, e.g, $u_x$ denotes $\frac{\partial u(t, x)}{\partial x}$.
The relationship of $u(t,x)$ and the SDE \eqref{eq:sde} is established by the Kolmogorov backward equation as follows. 
\begin{claim}[Kolmogorov backward equation \citep{Oksendal-03}]
    Let $u(t,x) = \expt [X(t) | X(0) = x]$, then $u(t,x)$ satisfies the following  partial differential equation:
    \begin{equation}
        u_t = - F_x u_x + \eta \sigma^2 u_{xx},\quad \text{with $u(0,x) = x$.}
        \label{eq:pde:1d}
    \end{equation}
\end{claim}

Using this claim, we can compute the first two derivatives of $u(t, x)$ in $t$, as follows:
\begin{lemma}
    \label{lem:pde:utt}
    Suppose $u(t,x)$ satisfies the PDE \eqref{eq:pde:1d}, then $
        u_t(0,x) = -F_x,  u_{tt}(0,x) = F_x F_{xx} - \eta \sigma^2 F_{xxx}$.
\end{lemma}
\begin{proof}[Proof sketch of Lemma~\ref{lem:pde:utt}]
    The first equation follows from equation~\eqref{eq:pde:1d} and the fact that $u_x(0, x) \equiv 1$ and $u_{xx}(0, x) \equiv 0 $ since $u(0,x) = x$. 
    To see the second equation, we take $\partial_t$ on both sides of \eqref{eq:pde:1d}, which gives
    \begin{equation}
        u_{tt} = - F_x u_{xt} + \eta \sigma^2 u_{xxt}.
        \label{eq:proof:lem:pde:utt}
    \end{equation}
    Since  $u_{xt} = u_{tx} = (u_t)_x$, one has (by \cref{eq:pde:1d})
    \begin{compact}
    \begin{equation}
        u_{xt} = (- F_x u_x + \eta \sigma^2 u_{xx})_x = - F_{xx} u_x + - F_x u_{xx} + \eta \sigma^2 u_{xxx}.
    \end{equation}
    \end{compact}
    For $t = 0$ we have $u_{xt}(0, x) = -F_{xx}$ since $u_{xx}(0,x) \equiv u_{xxx}(0,x) \equiv 0$. 
    Taking another $\partial_x$ yields $u_{xxt}(0, x) = -F_{xxx}$. Plugging back to \cref{eq:proof:lem:pde:utt} yields the second equation of the lemma \ref{lem:pde:utt}.
\end{proof}
With Lemma~\ref{lem:pde:utt} we can expand $u(t,x)$ around $(0,x)$:
\begin{compact}
\begin{equation}
    u(t,x) = x - F_x t + \frac{1}{2} \left( F_x F_{xx} - \eta \sigma^2 F_{xxx} \right) t^2 + o(t^2).
\end{equation}
\end{compact}
Ignoring higher order terms in $t$, the term $-\frac{1}{2}\eta \sigma^2 F_{xxx}$ reflects the difference between the noiseless GD trajectory from $x$ and $\mathbb{E}[X(t)| X(0) = x]$, that is, the iterate bias. 
Converting back to the discrete trajectory (\cref{eq:sde:sgd}) via the scaling $t = \eta k$, we obtain
\begin{compact}
\begin{equation}
    \mathbb{E}[x_{\sgd}^{(k)}] - z_{\gd}^{(k)} \approx -\frac{1}{2}\eta^3k^2\sigma^2F_{xxx}(x).
\end{equation}
\end{compact}
When the third derivative of $F$ is bounded by $Q$, this recovers the upper bound of $O(\eta^3k^2 Q \sigma^2)$ in Theorem~\ref{thm:3o:bias:ub}. The lower bound of Theorem~\ref{thm:3o:bias:lb} follows by choosing a function with third derivative $Q$ at $x^{\star}$. 

While it is possible to derive these results via more-involved discrete approaches, we believe the SDE approach may be promising for understanding more general objectives and algorithms. For instance, for multi-dimensional objectives, one can apply the same techniques to derive the \emph{direction} of the iterate bias via a multi-dimensional SDE, which is difficult to derive in the discrete setting.

%% file: sections/main_results.tex
\ifdefined\conf \section{LOWER BOUND RESULTS}\else \section{Lower Bound Results}\fi
\label{sec:lb}
In this section, we present our lower bounds for \fedavg in both convex homogeneous and heterogeneous settings, and discuss its implications. We then show how use the lower bound on the bias of SGD from Section~\ref{sec:bias} to establish a lower bound on the convergence of \fedavg.

Our main result for the homogeneous setting is the following theorem.
\begin{theorem}[Lower bound for homogeneous \fedavg (see \cref{full_lb})]\label{lb:homo}
For any $K \geq 2$, $R$, $M$, $\sigma$, and $D$, there exists $f(\x; \xi)$ and distribution $\xi \sim \mathcal{D}$ satisfying Assumption~\ref{asm:main} with optimum $\x^{\star}$, such that 
for some initialization $\x^{(0, 0)}$ with $\|\x^{(0, 0)} - \x^{\star}\|_2 < D$, the final iterate of \fedavg with any step size satisfies:
\ifconf
\begin{compact}
\begin{align}
  & \mathbb{E}\left[F(\x^{(R, 0)})\right] - F(\x^{\star}) \geq 
  \\
  & \Omega\left(
  \frac{HD^2}{KR} + \frac{\sigma D}{\sqrt{MKR}} 
  + 
  \min\left\{\frac{\sigma D}{\sqrt{KR}}, \frac{H^\frac{1}{3} \sigma^{\frac{2}{3}} D^{\frac{4}{3}}}{K^{\frac{1}{3}} R^{\frac{2}{3}}} \right\}\right).
\end{align}
\end{compact}
\else
\begin{equation}
  \mathbb{E}\left[F(\x^{(R, 0)})\right] - F(\x^{\star}) \geq 
  \Omega\left(
  \frac{HD^2}{KR} + \frac{\sigma D}{\sqrt{MKR}} 
  + 
  \min\left\{\frac{\sigma D}{\sqrt{KR}}, \frac{H^\frac{1}{3} \sigma^{\frac{2}{3}} D^{\frac{4}{3}}}{K^{\frac{1}{3}} R^{\frac{2}{3}}} \right\}\right).
\end{equation}
\fi
\end{theorem}
This lower bound matches the best known upper bound given by the theorem 2 of \citep{Woodworth.Patel.ea-ICML20}.

We extend our results to \fedavg in the heterogeneous setting. Recall that in this setting, we allow each client $m$ to draw $\xi$ from its own distribution $\mathcal{D}_m$. We prove our results under the following assumption on heterogeneity of the gradient at the optimum.
\begin{assumption}[Bounded gradient heterogeneity at optimum]\label{asm:hetero}
    $
            \frac{1}{M} \sum_{m=1}^M \|\nabla F_m(\x^{\star})\|_2^2 \leq \zeta_*^2.
    $
\end{assumption}
\begin{remark}\label{rem:hetero}
While the right measure of heterogeneity is a subject of significant debate in the FL community, the most popular are either a bound on gradient heterogeneity at $\x^{\star}$ (Assumption~\ref{asm:hetero}), or a stronger assumption of uniform gradient heterogeneity: for any $\x$, $\frac{1}{M} \sum_{m=1}^M \|\nabla F_m(\x) - \nabla F(\x)\|_2^2 \leq \zeta^2.$ The best-known lower bound, due to \cite{Woodworth.Patel.ea-NeurIPS20}, considers the weaker Assumption~\ref{asm:hetero}. We remark however that the strongest upper bounds use the stronger uniform assumption (e.g., \citep{, Woodworth.Patel.ea-ICML20} \footnote{While \citet{Khaled.Mishchenko.ea-AISTATS20} studies a relaxed assumption (optimum-heterogeneity like Assumption~\ref{asm:hetero} instead of uniform-heterogeneity), these results only hold with a much smaller step-size range $\eta \lesssim \frac{1}{KH}$ (in our notation, c.f. Theorem 3, 4 and 5 in their work), instead of $\eta \lesssim \frac{1}{H}$ as in the uniform setting. Under this restricted step-size range, one cannot recover the same upper bounds as in uniform-heterogeneity by optimizing $\eta$.}).
\end{remark}

\begin{theorem}[Lower bound for heterogeneous \fedavg (see \cref{full_lb})]\label{thm:lb}
For any $K \geq 2, R, M$, $H$, $D$, $\sigma$, and $\zeta_*$, there exist $f(\x; \xi)$ and distributions $\{\mathcal{D}_m\}$, each satisfying Assumption~\ref{asm:main}, and together satisfying Assumption~\ref{asm:hetero}, such that for some initialization $\x^{(0, 0)}$ with $\|\x^{(0, 0)} - \x^{\star}\|_2 < D$, the final iterate of \fedavg with any step size satisfies:
\ifconf
\begin{compact}
\begin{align}
    & \mathbb{E}\left[F(\x^{(R, 0)})\right] - F(\x^{\star}) \geq
    \Omega\left(\frac{HD^2}{KR} + \frac{\sigma D}{\sqrt{MKR}} 
    \right.
    \\
    & \left. + \min\left\{\frac{\sigma D}{\sqrt{KR}},  \frac{H^\frac{1}{3} \sigma^{\frac{2}{3}} D^{\frac{4}{3}}}{K^{\frac{1}{3}} R^{\frac{2}{3}}} \right\}
    + \min\left\{\frac{\zeta_*^2}{H}, \frac{H^{\frac{1}{3}}\zeta_*^{\frac{2}{3}}D^{\frac{4}{3}}}{R^{\frac{2}{3}}}\right\}\right)
\end{align}
\end{compact}
\else
\begin{align}
    \mathbb{E}\left[F(\x^{(R, 0)})\right] - F(\x^{\star}) \geq
    \Omega\left(\frac{HD^2}{KR} + \frac{\sigma D}{\sqrt{MKR}} 
    + \min\left\{\frac{\sigma D}{\sqrt{KR}},  \frac{H^\frac{1}{3} \sigma^{\frac{2}{3}} D^{\frac{4}{3}}}{K^{\frac{1}{3}} R^{\frac{2}{3}}} \right\}
    + \min\left\{\frac{\zeta_*^2}{H}, \frac{H^{\frac{1}{3}}\zeta_*^{\frac{2}{3}}D^{\frac{4}{3}}}{R^{\frac{2}{3}}}\right\}\right)
\end{align}
\fi
\end{theorem}
Theorem~\ref{thm:lb} is nearly tight, up to a difference in the definitions of heterogeneity (See Remark~\ref{rem:hetero}). We compare our result to existing lower bounds and upper bounds in Table~\ref{table:complexity}. 

\subsection{Interpretation of \cref{thm:lb}} 
To better understand the convergence rates in the Theorems above, first observe that the first two terms in both rates, $\frac{HD^2}{KR} + \frac{\sigma D}{\sqrt{MKR}}$ is familiar from the standard SGD convergence rate. The term $\frac{HD^2}{KR}$ corresponds to the deterministic convergence, which appears even when there is no noise. The term $\frac{\sigma D}{\sqrt{MKR}}$ is a standard statistical noise term that applies to any algorithm which accesses $MKR$ total stochastic gradients.

The third term in both theorems, $\frac{H^{\frac{1}{3}} \sigma^{\frac{2}{3}} D^{\frac{4}{3}}} {K^{\frac{1}{3}} R^{\frac{2}{3}}}$ depends on the variance of the noise, and arises due to the iterate bias of SGD. This term appears even in the homogeneous setting where all clients access the same distribution. Our main contribution is proving the appearance of this term in the lower bound. The previous best lower bound, due to \cite{Woodworth.Patel.ea-ICML20}, achieved in comparison the term $\frac{H^{\frac{1}{3}} \sigma^{\frac{2}{3}} D^{\frac{4}{3}}} {\mathbf{K^{\frac{2}{3}}} R^{\frac{2}{3}}}$, which is a factor of $K^{\frac{1}{3}}$ weaker. We expand on how we achieve this term in subsection~\ref{sec:bias_to_lb}.

The last term of \cref{thm:lb}, $\frac{H^{\frac{1}{3}} \zeta_*^{\frac{2}{3}} D^{\frac{4}{3}}} {R^{\frac{2}{3}}}$ is due to another bias that scales with the heterogeneity of the data among the clients. In comparison, the best known lower bound on the dependence on the heterogeneity is $\min\left(\frac{HD^2}{R}, \frac{H^{\frac{1}{3}} \zeta_*^{\frac{2}{3}} D^{\frac{4}{3}}} {R^{\frac{2}{3}}}\right)$. Note that as $R$ becomes large, the minimum is achieved by $\frac{HD^2}{R}$, yielding a significantly weaker lower bound which doesn't depend at all on the heterogeneity.

Our lower bound shows that under only and assumption of second order smoothness and convexity (Assumptions~\ref{asm:main}), \fedavg may achieve a rate as slow as ${K^{-\frac{1}{3}}R^{-\frac{2}{3}}}$. Prior work has pointed out that this rate can be beat by alternative algorithms that use the same (or less) communication and gradient computation. One such algorithm is \em minibatch SGD \em, which replaces the $K$ iterations of local SGD at each client with a single iteration. This results in the same outcome as $R$ iterations of SGD with minibatch size $M$. A second such algorithm, \em single-machine SGD \em ignores all but one client, and results the same outcome as $KR$ iterations of SGD. Under Assumption~\ref{asm:main}, the best of these two algorithms (minibatch SGD and single-machine SGD) achieves a rate of
\begin{compact}
    \begin{equation}
        \frac{HD^2}{KR} + \frac{\sigma D}{\sqrt{MKR}} + \min\left(\frac{HD^2}{R},  \frac{\sigma D}{\sqrt{KR}}\right).
    \end{equation}
\end{compact}
It turns out that this rate always dominates the the sharp rate we have shown for \fedavg.
Further, when $\sigma$ and $K$ are large, this rate is dominated by $\frac{H D^2}{R}$, while the rate of \fedavg is dominated by $\frac{H^{\frac{1}{3}}\sigma^{\frac{2}{3}}D^{\frac{4}{3}}}{K^{\frac{1}{3}}R^{\frac{2}{3}}}$. In this regime, the rate of this ``naive'' algorithm may improve on the rate of \fedavg by a factor of $\left(\frac{R\sigma^2H^2}{K}\right)^{1/3}$. 

\subsection{Constructing Lower Bound from Iterate Bias}\label{sec:bias_to_lb}
In this subsection, we theoretically establish the relationship between the iterate bias (Definition~\ref{def:bias}) and the lower bound on the function error of \fedavg.

Recall that in \cref{thm:2o:bias:lb}, we proved a lower bound on bias from the optimum $x^{\star}$, which came from analyzing SGD with Gaussian noise on the the piecewise quadratic function, which we abbreviate ``$\psi(x)$'':
\begin{compact}
\begin{equation}\label{def:pqf}
    f(x; \xi) = \psi(x) + x \xi, \quad 
    \psi(x) = 
    \begin{cases}
        \frac{1}{2} H x^2 & x \geq 0, \\
        \frac{1}{4} H x^2 & x < 0.
    \end{cases}
\end{equation}
\end{compact}
where $\xi \sim \mathcal{N}(0, 1)$. 

To construct our lower bound, we show when we run \fedavg on the function above, this same bias, $\eta^2k^{\frac{3}{2}}H\sigma$, persists more generally from any $x$ which is not too far from the optimum $x^{\star} = 0$. Loosely speaking, we can achieve this same bias whenever a constant fraction of the mass of the iterate $x_{\sgd}^{(k)}$ lies on each side of $x^{\star}$. Since the variance of $x_{\sgd}^{(k)}$ is on the order of $\eta^2 k \sigma^2$, we can prove that the bias from $x$ will continue at the rate given in \cref{thm:2o:bias:lb} from any $x$ with $|x| \leq \Theta(\eta \sqrt{k} \sigma)$. In fact, we can extend this observation to the case when the initial iterate $x_{\sgd}^{(0)}$ is a random variable, and its expectation is bounded, yielding the following lemma:

\begin{lemma}[Simplified from \cref{full_lem:expected_step}]\label{lemma:expected_step}
Let $f(x, \xi)$ be as in \ref{def:pqf}. If $\eta \leq \frac{1}{2kH}$,\footnote{For simplicity, in this section we focus on the regime where $\eta \leq \frac{1}{2kH}$, though our proofs in the Appendix we consider any $\eta \leq {O}(\frac{1}{H})$.} then there exist constants $c_1$ and $c_2 > 0$ such that for any random variable $x$ with $\mathbb{E}[x]^2 \leq c_1k\eta^2\sigma^2$ and $\mathbb{E}[x] \leq 0$, we have
\begin{compact}
\begin{equation}
\mathbb{E}_{x}\mathbb{E}_{\sgd}[x_{\sgd}^{(k)} | x_{\sgd}^{(0)}  = x] \leq \mathbb{E}_{x}[z_{\gd}^{(k)}| z_{\gd}^{(0)}  = x] - c_2\eta^2k^{\frac{3}{2}}H\sigma.
\end{equation}
\end{compact}
\end{lemma}

Directly applying Lemma~\ref{lemma:expected_step}, we can show that the expectation of the \fedavg iterate $x^{(r, 0)}$ moves in the negative direction each round:
\begin{compact}
\begin{equation}\label{fed_avg_bias}
\mathbb{E}[x^{(r + 1, 0)} | x^{(r, 0)}] \leq (1 - \eta H/2)^K \mathbb{E}[x^{(r, 0)}]- c_2\eta^2K^{3/2}H\sigma,
\end{equation}
\end{compact}
so long as $\eta \leq \frac{1}{2KH}$ and $-\sqrt{c_1K}\eta\sigma \leq \mathbb{E}[x^{(r, 0)}] \leq 0$.

Of course, when $\mathbb{E}[x^{(r, 0)}]$ becomes too negative, the force of the gradient in the positive direction exceeds the negative bias. Once this occurs, we are in the \em mixing \em regime. One can check from \cref{fed_avg_bias} that this occurs roughly when $\mathbb{E}[x^{(r, 0)}] \approx -\eta K^{1/2}\sigma$. 
Combining these observations, we obtain the following lemma, stated to include the more general case when $\eta > \frac{1}{2kH}$.

\begin{lemma}%
\label{shortlem:steplb} Let $f(x, \xi)$ be as in \ref{def:pqf}. There exists a universal constant $c$ such that for $\eta \leq \frac{1}{6H}$, if $x^{(0, 0)} = 0$, then 
$\mathbb{E}[F(x^{(R, 0)})] \geq \frac{c^2}{4}\eta\sigma^2\min\left(R(\eta H K)^{3}, 1, \eta H K\right).$
\end{lemma}
With this bound on the function error of the step function, proving our lower bound for the homogeneous case follows by considering the function on $\mathbb{R}^3$ used in the lower bound construction of \cite{Woodworth.Patel.ea-ICML20}. We state this function fully in the appendix.

Our lower bound for the heterogeneous setting is similar, but involves the following additional ingredient:

\begin{lemma}%
\label{lem:hetero_short}
Consider \fedavg with $M$ clients with 
\begin{compact}
$$f^{(3)}(x; (\xi_1, \xi_2)) = \begin{cases}Hx^2 - x\xi_2 & \xi_1 = 1\\ \frac{H}{2}x^2 - x\xi_2& \xi_1 = 2\\ \end{cases},$$
\end{compact}
and for all the odd $m \in [M]$, we have $(\xi_1, \xi_2) = (1, \zeta_*)$ always, while for all the even $m \in [M]$ we have $(\xi_1, \xi_2) = (2, -\zeta_*)$. There exists a universal constant $c_h$ such that for $\eta \leq \frac{1}{H}$, if $x^{(0, 0)} \leq 0$, then \fedavg with $R$ rounds and $K$ steps per round results in 
\begin{compact}
$$x^{(R, 0)} \leq -\frac{c_h}{H} \min(1, \eta H K, (\eta H K)^2R)\zeta_*.$$ 
\end{compact}
\end{lemma}
The functions studied in this lemma appear in the heterogeneous lower bound construction in \cite{Woodworth.Patel.ea-NeurIPS20}, but the analysis we give in this lemma is much tighter than theirs. 

%% file: sections/third_order_results.tex
\ifdefined\conf \section{FEDAVG UPPER BOUNDS  UNDER THIRD-ORDER SMOOTHNESS}\else\section{Upper Bounds for \fedavg Under Third-Order Smoothness}\fi
\label{sec:third}

In light of the limitations of \fedavg discussed in Section~\ref{sec:lb}, it is natural to ask if there are additional assumptions under which \fedavg may perform better. Several classes of additional assumptions have been suggested for studying the performance of \fedavg. Perhaps the most common, and the one supported from our intuition on the bias, is an assumption of third-order smoothness, stated formally in Assumption~\ref{asm:third_order}. Previously it has been shown that under such an assumption, \fedavg may converge faster. We present several state-of-the-art bounds for \fedavg under Assumption~\ref{asm:third_order}, including for the non-convex case.

\begin{theorem}[Upper bound for \fedavg under $3$rd order smoothness (see \cref{full_thm:ub:3o})]\label{thm:3o_convex}

Suppose $f(\x, \xi)$ satisfies Assumptions~\ref{asm:main} and Assumptions~\ref{asm:third_order}. Then for some step size, \fedavg satisfies
\begin{compact}
\begin{equation}
    \mathbb{E}\left[\left\|\nabla f(\hat{\x})\right\|^2\right]
    \leq
    \bigo 
    \left(
        \frac{HB}{KR}
        +
        \frac{\sigma \sqrt{BH}}{\sqrt{MKR}}
        +
        \frac{B^{\frac{4}{5}}\sigma^{\frac{4}{5}}Q^{\frac{2}{5}}}{K^{\frac{2}{5}}R^{\frac{4}{5}}}
    \right)
\end{equation}
\end{compact}
where $\hat{\x} := \frac{1}{M}\sum_m{\x^{(r, k)}_m}$  for a  random choice of $k \in [K]$, and $r \in [R]$, and $B := F(\x^{(0, 0)}) - \inf_{\x} F(\x)$.
\end{theorem}

In the non-convex setting, akin to some other work in FL literature~\citep{yu2019parallel,Reddi.Charles.ea-ICLR21}, we require an assumption bounding moments of the stochastic gradients. Note that this is stronger that Assumption~\ref{asm:main} which bounds the \em variance \em of the stochastic gradients. We remark that several other works impose weaker assumptions, though the algorithms they consider are different, or their results are weaker.~\citep{Stich-ICLR19, Koloskova.Loizou.ea-ICML20, Wang.Joshi-18}.
\begin{assumption}[Bounded gradients]\label{asm:bounded_grad}
For any $\x$, we have $\mathbb{E}_{\xi}\left[\|\nabla f(\x, \xi)\|^4\right] \leq G^4.$
\end{assumption}
\begin{theorem}[Upper bound for \fedavg with non-Convex objectives under third-order smoothness, see \cref{full_thm:3o_non_convex}]\label{thm:non_convex}
Suppose $F(\x)$ is $H$-smooth and $f(\x, \xi)$ satisfies Assumptions~\ref{asm:third_order} and \ref{asm:bounded_grad}. Then for some step size, we have
\begin{compact}
\begin{equation}
  \mathbb{E}\left[\left\|\nabla f(\hat{\x})\right\|^2\right]\leq
  \bigo \left(
  \frac{HB}{KR}
  + 
  \frac{G\sqrt{BH}}{\sqrt{MKR}}
  +
  \frac{B^{\frac{4}{5}}G^{\frac{4}{5}}Q^{\frac{2}{5}}}{R^{\frac{4}{5}}}
  \right),
\end{equation}
\end{compact}
where $\hat{\x} := \frac{1}{M}\sum_m{\x_m^{(r, k)}}$ for a  random choice of $k \in [K]$, and $r \in [R]$, and $B := F(\x^{(0, 0)}) - \inf_{\x} F(\x)$.
\end{theorem}
\begin{remark}
In \cref{full_thm:3o_non_convex}, we weaken Assumption~\ref{asm:bounded_grad} to a uniform bound on $\|\nabla F(\x)\|$.%
\end{remark}

This theorem shows that the convergence rate of \fedavg improves substantially under third order smoothness. In comparison, the best known rate for \fedavg with non-convex objectives (under second-order smoothness alone) is $\frac{HB}{KR} 
  + \frac{G \sqrt{BH}}{\sqrt{MKR}} 
  +  \frac{B^{\frac{2}{3}}G^{\frac{2}{3}}H^{\frac{2}{3}}}{R^{\frac{2}{3}}}$,\footnote{There are other extensions of \fedavg that can outperform this rate, e.g., \fedavg with server learning rate discussed in \cref{sec:related-work},
  since it includes mini-batch SGD as a special case.} due to \cite{yu2019parallel}.\footnote{This rate is not explicitly given in their paper, but can be proved from their work by setting the step size appropriately. For completeness, we prove this rate in the Appendix \ref{sec:proof:fedavg:2o:ub}.} Observe that we improve the dependence from $R^{\frac{2}{3}}$ in the third term to $R^{\frac{4}{5}}$.

%% file: sections/conclusion.tex
\ifdefined\conf \section{ CONCLUSION}\else\section{Conclusion}\fi
In this work we provided sharp lower bounds for homogeneous and heterogeneous $\fedavg$ that matches the existing upper bound. 
By solving this open problem, we highlight the obstacles to $\fedavg$, and show how a third-order smoothness assumption can lead to faster convergence. We expect the proposed techniques can shed light on the analysis of other federated algorithms and aid design of more efficient federated algorithms. 

%% file: sections/ack.tex
\section*{Acknowledgements}
We would like to thank Aaron Sidford for helpful discussions. 
MG acknowledges the support of NSF award DGE-1656518. 
HY is partially supported by the TOTAL Innovation Scholars program. 
TM acknowledges the support of Google Faculty Award, NSF IIS 2045685, the Sloan Fellowship, and JD.com. 
We would like to thank the anonymous reviewers for their suggestions and comments.

%% file: appendices/bias_bdd.tex
\ifdefined\conf\section{\large FORMAL THEOREMS AND PROOFS ON THE BOUNDS OF ITERATE BIAS}\else \section{Formal Theorems and Proofs on the Bounds of Iterate Bias}\fi

In this section, we list and prove the complete theorems on the lower and upper bounds of iterate bias discussed in \cref{sec:bias}.

\subsection{Formal Theorems Statement}
\begin{theorem}[Upper bound of iterate bias under second-order smoothness, complete version of \cref{thm:2o:bias:ub}]
    \label{thm:2o:bias:ub:complete}
    Assume $F(\x) := \expt_{\xi} f(\x; \xi)$ satisfies \cref{asm:main}.
    Let $\{\x_\sgd^{(k)}\}_{k=0}^{\infty}$ be the trajectory of SGD initialized at $\x_{\sgd}^{(0)} = \x$,  and $\{\z_{\gd}^{(k)}\}_{k=0}^{\infty}$ be the trajectory of GD initialized at $\z_{\gd}^{(0)} = \x$, namely
    namely
    \begin{equation}
        \x^{(k+1)}_{\sgd} := \x_{\sgd}^{(k)} - \eta \nabla f(\x_{\sgd}^{(k)}; \xi^{(k)}), \quad 
        \z^{(k+1)}_{\gd} := \z_{\gd}^{(k)} - \eta \nabla F(\z_{\gd}^{(k)}), \quad \text{for } k = 0, 1,\ldots
    \end{equation}
    Then for any $\eta \leq \frac{1}{H}$, the following inequality holds
    \begin{equation}
        \left \| 
            \expt \x_{\sgd}^{(k)} - \z_{\gd}^{(k)}
        \right\|_2
        \leq
        \min \{4 \eta^2  k^{\frac{3}{2}} H \sigma, \eta k^{\frac{1}{2}} \sigma \}.
        \label{eq:thm:bias:2o:ub}
    \end{equation}
\end{theorem}
The proof of \cref{thm:2o:bias:ub:complete} is provided in \cref{sec:proof:2o:bias:ub}.

\begin{theorem}[Lower bound of iterate bias under second-order smoothness, complete version of \cref{thm:2o:bias:lb}]
    \label{thm:2o:bias:lb:complete}
For any $H, \sigma, K$, there exists a function $f(\x; \xi)$ and a distribution $\dist$ satisfying \cref{asm:main} such that for any $\eta \leq \frac{1}{2H}$, for any $k \leq K$ the following iterate bias inequality holds for SGD and GD initialized at the optimum
    \begin{equation}
        \left\| \expt [\x_{\sgd}^{(k)}] - \z_{\gd}^{(k)} \right\|  \geq
        0.002\min\left\{ \eta^2  k^{\frac{3}{2}} H \sigma,  \eta^{\frac{1}{2}}H^{-\frac{1}{2}} \sigma \right\}.
    \end{equation}
\end{theorem}
\cref{thm:2o:bias:lb:complete} is proved in Section~\ref{sec:pf_exp_drift} as a special case of Lemma~\ref{lemma:masterlb}, by taking $x^{(0)} = 0$ to be the optimum.

\begin{theorem}[Upper bound of iterate bias under third-order smoothness, complete version of \cref{thm:3o:bias:ub}]
    \label{thm:3o:bias:ub:complete}
    Assume $F(\x) := \expt_{\xi} f(\x; \xi)$ satisfies \cref{asm:main,asm:third_order}.
    Let $\{\x_\sgd^{(k)}\}_{k=0}^{\infty}$ be the trajectory of SGD initialized at $\x_\sgd^{(0)} = \x$,  and $\{\z_{\mathtt{GD}}^{(k)}\}_{k=0}^{\infty}$ be the trajectory of GD initialized at $\z_{\gd}^{(0)} = \x$, namely
    namely
    \begin{equation}
        \x^{(k+1)}_{\sgd} := \x_{\sgd}^{(k)} - \eta \nabla f(\x_{\sgd}^{(k)}; \xi^{(k)}), \quad 
        \z^{(k+1)}_{\gd} := \z_{\gd}^{(k)} - \eta \nabla F(\z_{\gd}^{(k)}), \quad \text{for } k = 0, 1,\ldots
    \end{equation}
    Then for any $\eta \leq \frac{1}{H}$, the following inequality holds
    \begin{equation}
        \left \| 
            \expt \x_{\sgd}^{(k)} - \z_{\gd}^{(k)}
        \right\|_2
        \leq
        \min \left\{
        \frac{1}{4} \eta^3 k^2 Q \sigma^2
        , 4 \eta^2 k^{\frac{3}{2}} H \sigma, \eta k^{\frac{1}{2}} \sigma \right\}.
        \label{eq:thm:3o:bias:ub:complete}
    \end{equation}
\end{theorem}
The proof of \cref{thm:3o:bias:ub:complete} is provided in \cref{sec:proof:bias:3o:ub:complete}.

\begin{theorem}[Lower bound of iterate bias under third-order smoothness, complete version of \cref{thm:3o:bias:lb}]
    \label{thm:3o:bias:lb:complete}
    For any $H, \sigma, K$, for any $Q \leq \frac{H^2}{12 K \sigma}$, there exists a function $f(\x; \xi)$ and a distribution $\dist$ satisfying \cref{asm:main,asm:third_order} such that for any $\eta \leq \frac{1}{2H}$, for any $k < K$, the following iterate bias inequality holds for SGD and GD initialized at the optimum
    \begin{equation}
        \left\| \expt [\x_{\sgd}^{(k)}] - \z_{\gd}^{(k)} \right\|  \geq
        0.005 \eta^3 \sigma^2 Q  \min \left\{ \frac{k-1}{ \eta H}, k(k-1) \right\}.
        \label{eq:3o:bias:lb:complete}
    \end{equation}
\end{theorem}
The proof of \cref{thm:3o:bias:lb:complete} is provided in \cref{sec:proof:bias:3o:lb:complete}.

\subsection{Proof of \cref{thm:2o:bias:ub:complete}: Upper Bound of Iterate Bias Under 2nd-Order Smoothness}
\label{sec:proof:2o:bias:ub}
The proof of \cref{thm:2o:bias:ub:complete} is based on the following two lemmas: \cref{lem:bias:2o:ub:1,lem:bias:2o:ub:2}.
\begin{lemma}
    \label{lem:bias:2o:ub:1}
    Under the same settings of \cref{thm:2o:bias:ub:complete}, for any $\eta \leq \frac{1}{H}$, the following inequality holds 
    \begin{equation}
        \left \| 
            \expt \x_{\sgd}^{(k)} - \z_{\gd}^{(k)}
        \right\|_2
        \leq
        \frac{(1 + \eta H)^k - 1}{\eta H} \cdot 2\eta^2 H k^{\frac{1}{2}} \sigma.
    \end{equation}
\end{lemma}

\begin{proof}[Proof of \cref{lem:bias:2o:ub:1}]
    By definition of $\x_{\sgd}^{(k+1)}$ and $\z_{\gd}^{(k+1)}$ we obtain
    \begin{align}
     \left\| \expt \x_{\sgd}^{(k+1)} - \z_{\gd}^{(k+1)} \right\|_2
    & = \left\| \left( \expt \x_{\sgd}^{(k)} - \z_{\gd}^{(k)} \right) 
    - 
    \eta \left( \expt \nabla F (\x_{\sgd}^{(k)}) - \nabla F (\z_{\gd}^{(k)}) \right) \right\|_2
    \\
    &   \leq
    \left\| \expt \x_{\sgd}^{(k)} - \z_{\gd}^{(k)}  \right\|_2
    + 
    \eta \left\| \expt \nabla F (\x_{\sgd}^{(k)}) - \nabla F (\z_{\gd}^{(k)}) \right\|_2.
    \end{align}
    Now we seek an upper bound for $\left\| \expt \nabla F (\x_{\sgd}^{(k)}) - \nabla F (\z_{\gd}^{(k)}) \right\|_2$. Observe that
    \begin{align}
        & \left\| \expt \nabla F (\x_{\sgd}^{(k)}) - \nabla F (\z_{\gd}^{(k)}) \right\|_2 
        \\
    \leq & \expt \left\| \nabla F (\x_{\sgd}^{(k)}) - \nabla F (\z_{\gd}^{(k)}) \right\|_2
        \tag{Jensen's inequality}
        \\
    \leq & \eta H  \expt \left\| \x_{\sgd}^{(k)} - \z_{\gd}^{(k)} \right\|_2
        \tag{by $H$-smoothness of $F$}
        \\
    \leq & \eta H  \left( \left\| \expt \x_{\sgd}^{(k)} - \z_{\gd}^{(k)} \right\|_2 + \expt \left\| \x_{\sgd}^{(k)} - \expt \x_{\sgd}^{(k)} \right\|_2 \right)
        \tag{by triangle inequality}
        \\
    \leq & \eta H  \left( \left\| \expt \x_{\sgd}^{(k)} - \z_{\gd}^{(k)} \right\|_2 + \sqrt{ \expt \left\| \x_{\sgd}^{(k)} - \expt \x_{\sgd}^{(k)} \right\|_2^2}\right).
        \tag{by Holder's inequality}
    \end{align}
    By standard convex stochastic analysis (e.g. \citep{Khaled.Mishchenko.ea-AISTATS20}) one can show that  $\expt \left\| \x_{\sgd}^{(k)} - \expt \x_{\sgd}^{(k)} \right\|_2^2 \leq 2 \eta^2 k \sigma^2$. 
    Consequently
    \begin{equation}
        \left\| \expt \x_{\sgd}^{(k+1)} - \z_{\gd}^{(k+1)} \right\|_2
        \leq
        (1 + \eta H) \left(  \left\| \expt \x_{\sgd}^{(k)} - \z_{\gd}^{(k)} \right\|_2 \right) + 2 \eta^2 H k^{\frac{1}{2}} \sigma.
        \label{eq:proof:thm:bias:2o:ub:1}
    \end{equation}
    Telescoping \cref{eq:proof:thm:bias:2o:ub:1} completes the proof.
\end{proof}

\begin{lemma}
    \label{lem:bias:2o:ub:2}
    Under the same settings of \cref{thm:2o:bias:ub:complete}, or any $\eta \leq \frac{1}{H}$, the following inequality holds
    \begin{equation}
        \left \| 
            \expt \x_{\sgd}^{(k)} - \z_{\gd}^{(k)}
        \right\|_2
        \leq
        \eta k^{\frac{1}{2}} \sigma.
    \end{equation}
\end{lemma}

\begin{proof}[Proof of \cref{lem:bias:2o:ub:2}]
     By definition of $\x_{\sgd}^{(k+1)}$ and $\z_{\gd}^{(k+1)}$ we obtain
    \begin{align}
        & \expt \left\| \x_{\sgd}^{(k+1)} - \z_{\gd}^{(k+1)} \right\|_2^2
        =
        \expt \left\| \left(\x_{\sgd}^{(k)} - \z_{\gd}^{(k)} \right)
        - 
        \eta \left( \nabla f(\x_{\sgd}^{(k)}; \xi^{(k)}) - \nabla F(\z_{\gd}^{(k)}) \right)
        \right\|_2^2
        \\
    \leq & \expt \left\| \left(\x_{\sgd}^{(k)} - \z_{\gd}^{(k)} \right)
        - 
        \eta \left( \nabla F(\x_{\sgd}^{(k)}) - \nabla F(\z_{\gd}^{(k)}) \right)
        \right\|_2^2 + \eta^2 \sigma^2. \tag{by independence and $\sigma^2$-bounded covariance}
    \end{align}
    Note that 
    \begin{align}
        &  \left\| \left(\x_{\sgd}^{(k)} - \z_{\gd}^{(k)} \right)
        - 
        \eta \left( \nabla F(\x_{\sgd}^{(k)}) - \nabla F(\z_{\gd}^{(k)}) \right)
        \right\|_2^2
        \\
    = & \left\| \x_{\sgd}^{(k)} - \z_{\gd}^{(k)}
        \right\|_2^2
        -
        2 \eta \left \langle \nabla F(\x_{\sgd}^{(k)}) - \nabla F(\z_{\gd}^{(k)}), \x_{\sgd}^{(k)} - \z_{\gd}^{(k)}  \right \rangle
        +
        \eta^2 \left\| \nabla F(\x_{\sgd}^{(k)}) - \nabla F(\z_{\gd}^{(k)})  \right\|_2^2 
        \\
    \leq & \left\| \x_{\sgd}^{(k)} - \z_{\gd}^{(k)}
        \right\|_2^2
        -
        \left( \frac{2 \eta}{H} - \eta^2 \right) \left\| \nabla F(\x_{\sgd}^{(k)}) - \nabla F(\z_{\gd}^{(k)})  \right\|_2^2 
        \tag{by convexity and $H$-smoothness}
        \\
    \leq & \left\| \x_{\sgd}^{(k)} - \z_{\gd}^{(k)}
        \right\|_2^2 \tag{since $\eta \leq \frac{2}{H}$}.
    \end{align}
    Therefore
    \begin{equation}
         \expt \left\| \x_{\sgd}^{(k+1)} - \z_{\gd}^{(k+1)} \right\|_2^2
         \leq
         \expt \left\| \x_{\sgd}^{(k)} - \z_{\gd}^{(k)} \right\|_2^2 + \eta^2 \sigma^2.
    \end{equation}
    Telescoping yields
    \begin{equation}
        \expt \|\x_{\sgd}^{(k)} - \z_{\gd}^{(k)}\|_2^2 \leq \eta^2 k \sigma^2.
    \end{equation}
    and thus by Jensen's inequality and Holder's inequality
    \begin{equation}
        \left \|\expt  \x_{\sgd}^{(k)} - \z_{\gd}^{(k)} \right\|_2 \leq 
       \expt  \left\| \x_{\sgd}^{(k)} - \z_{\gd}^{(k)} \right\|_2
       \leq 
       \sqrt{ \expt  \left\| \x_{\sgd}^{(k)} - \z_{\gd}^{(k)} \right\|_2^2}
       \leq
       \eta k^{\frac{1}{2}} \sigma.
    \end{equation}
\end{proof}

With \cref{lem:bias:2o:ub:1,lem:bias:2o:ub:2} at hands we are ready to prove \cref{thm:2o:bias:ub:complete}.
\begin{proof}[Proof of \cref{thm:2o:bias:ub:complete}]
    We consider the case of $\eta \leq \frac{1}{Hk}$ and $\eta > \frac{1}{Hk}$ separately. In either case we have $ \left \| 
            \expt \x_{\sgd}^{(k)} - \z_{\gd}^{(k)}
        \right\|_2
        \leq
        \eta k^{\frac{1}{2}} \sigma$ by \cref{lem:bias:2o:ub:2}.
    
    If $\eta \leq \frac{1}{Hk}$, by \cref{lem:bias:2o:ub:1}, we have
    \begin{equation}
        \left \| 
            \expt \x_{\sgd}^{(k)} - \z_{\gd}^{(k)}
        \right\|_2
        \leq
        \frac{(1 + \eta H)^k - 1}{\eta H} 2 \eta^2 H k^{\frac{1}{2}} \sigma
        \leq
        \frac{e^{\eta HK} - 1}{\eta H} 2 \eta^2 H k^{\frac{1}{2}} \sigma
        \leq
        4 \eta^2 Hk^{\frac{3}{2}} \sigma,
    \end{equation}
    where the last inequality is due to $e^{\eta Hk} - 1 \leq 2 \eta Hk$ since $\eta Hk \leq 1$. Therefore \cref{eq:thm:bias:2o:ub} is satisfied.
    
    If $\eta > \frac{1}{Hk}$, then $\eta k^{\frac{1}{2}} \sigma < \eta^2 Hk^{\frac{3}{2}} \sigma$. Hence \cref{eq:thm:bias:2o:ub} is also satisfied.
\end{proof}

\subsection{Proof of \cref{thm:3o:bias:ub:complete}: Upper Bound of Iterate Bias Under 3rd-order Smoothness}
\label{sec:proof:bias:3o:ub:complete}
The proof of \cref{thm:3o:bias:ub:complete} is based on the following lemma.
\begin{lemma}
\label{lem:bias:3o:ub}
Consider the same settings of \cref{thm:3o:bias:ub:complete}. For any $k$, define vector-valued function 
\begin{equation}
    \u^{(k)}(\x) = \expt \left[ \x_{\sgd}^{(k)} \mid \x^{(0)} = \x \right].
\end{equation}
Then the following results hold.
\begin{enumerate}[(a)]
    \item For any $k$, $\u^{(k+1)}(\x) = \expt_{\xi} \left[ \u^{(k)}(\x - \eta \nabla f(\x; \xi) ) \right]$.
    \item For any $k$,  $\Diff \u^{(k+1)}(\x) = \expt_{\xi} \left[ 
  \Diff  \u^{(k)} (\x - \eta \nabla f(\x; \xi)) 
    \left( \I - \eta \nabla^2 f(\x; \xi) \right) 
    \right]$. Here $\Diff$ denotes the Jacobian operator.
    \item For any $k$,  $\sup_{\x} \|\Diff \u^{(k)}(\x)\| \leq 1$.
    \item For any $k$,  $\sup_{\x} \|\Diff^2 \u^{(k)}(\x)\| \leq \eta k Q$.
    \item For any $k$,  $\left\| \u^{(k+1)}(\x) - \u^{(k)} (\x - \eta \nabla F(\x)) \right\|_2 \leq \frac{1}{2} \eta^3 k Q \sigma^2$.
\end{enumerate}
\end{lemma}
\begin{proof}[Proof of \cref{lem:bias:3o:ub}]
\begin{enumerate}[(a)]
    \item Holds by time-homogeneity of the SGD sequence as 
    \begin{align}
        \u^{(k+1)}(\x) & = \expt \left[\x_{\sgd}^{(k+1)} \middle| \x^{(0)}_{\sgd} = \x \right]
        =
        \expt_{\xi} \expt \left[\x_{\sgd}^{(k+1)} \middle| \x^{(1)}_{\sgd} = \x - \eta \nabla f(\x; \xi) \right]
        \\
        & =
         \expt_{\xi} \expt \left[\x_{\sgd}^{(k)} \middle| \x^{(0)}_{\sgd} = \x - \eta \nabla f(\x; \xi) \right]
         = 
         \expt_{\xi} \left[ \u^{(k)}(\x - \eta \nabla f(\x; \xi) ) \right].
    \end{align}
    \item Holds by taking derivative on both sides of (a). Indeed, for any $i \in [d]$, one has
        \begin{equation}
            \nabla u_i^{(k+1)}(\x)^\top = \expt_{\xi} \left[
            \nabla u_i^{(k)} (\x - \eta \nabla f(\x; \xi))^\top \left( \I - \eta \nabla^2 f(\x; \xi) \right) 
            \right],
        \end{equation}
        where $u_i^{(k)}$ denotes the $i$-th coordinate of the vector-valued function $\u^{(k)}$.
    \item By (b) one has 
    \begin{equation}
        \left\|\Diff \u^{(k+1)}(\x)\right\| \leq \expt_{\xi} \left[ 
        \left\|\Diff  \u^{(k)} (\x - \eta \nabla f(\x; \xi)) \right\| 
        \left\| \I - \eta \nabla^2 f(\x; \xi) \right\|
        \right].
    \end{equation}
    Since $f(\x; \xi)$ is convex and $H$-smooth w.r.t. $\x$, and $\eta \leq \frac{1}{H}$, one has $\sup_{\x, \xi} \left\| \I - \eta \nabla^2 f(\x; \xi) \right\| \leq 1$. Therefore
    \begin{equation}
        \sup_{\x} \left\|\Diff \u^{(k+1)}(\x)\right\|
        \leq
        \sup_{\x} \left\|\Diff \u^{(k)}(\x)\right\|.
    \end{equation}
    By definition of $\u^{(0)}(\x) = \Diff \u^{(0)}(\x) = \I$. Telescoping the above inequality yields (c).
    \item Taking twice derivatives w.r.t. $\x$ on both sides of (a) gives (for any $i$)
    \begin{small}
    \begin{equation}
    \nabla^2 u_i^{(k+1)}(\x) 
    =
    \expt_{\xi} \left[ (\I - \eta \nabla^2 f(\x; \xi)) \nabla^2 u_i^{(k)} (\x - \eta \nabla f(\x; \xi))
        (\I - \eta \nabla^2 f(\x; \xi)) 
        - \eta \nabla^3 f (\x; \xi) [\nabla u_i^{(k)}(\x - \eta \nabla f(\x; \xi))] \right]
    \end{equation}
    \end{small}
    Therefore
    \begin{equation}
        \sup_{\x}\| \Diff^2 \u^{(k+1)} (\x)\|_2 \leq  \sup_{\x}\| \Diff^2 \u^{(k)}(\x) \|_2 
        \sup_{\x, \xi} \| \I- \eta \nabla^2 f(\x; \xi) \|^2_2
        + \eta \cdot \left(  \sup_{\x, \xi}\|\nabla^3 f(\x; \xi)\|_2  \right)
        \cdot
        \left( \sup_{\x} \|\Diff \u^{(k)} (\x)\|_2 \right).
    \end{equation}
    Since $f(\x; \xi)$ is convex and $H$-smooth w.r.t. $\x$ and $\eta \leq \frac{1}{H}$, one has $\sup_{\x, \xi} \left\| \I - \eta \nabla^2 f(\x; \xi) \right\| \leq 1$. Also by (c), we arrive at
    \begin{equation}
        \sup_{\x} \| \Diff^2 \u^{(k+1)} (\x)\|_2
        \leq
        \sup_{\x} \| \Diff^2 \u^{(k)} (\x)\|_2
        +
        \eta Q
    \end{equation}
    Telescoping from $0$ to $k$ yields (d).
    \item By (a)
    \begin{align}
        & \left\| \u^{(k+1)}(\x) - \u^{(k)} (\x - \eta \nabla F(\x)) \right\|_2
        = \left\| \expt_{\xi} \left[ \u^{(k)}(\x - \eta \nabla f(\x; \xi) ) \right] - \u^{(k)} (\x - \eta \nabla F(\x)) \right\|_2 
        \tag{by (a)}
        \\
        = & \left\| \expt_{\xi} \left[ \u^{(k)}(\x - \eta \nabla f(\x; \xi) ) - \u^{(k)} (\x - \eta \nabla F(\x)) 
        - \Diff \u^{(k)} (\x - \eta \nabla F(\x)) \left(\eta \nabla f(\x; \xi) - \eta \nabla F(\x)  \right) \right] 
        \right\|_2 
        \tag{Since $\expt_{\xi} \nabla f(\x; \xi) = \nabla F(\x)$}
        \\
        \leq & \expt_{\xi} \left\| \u^{(k)}(\x - \eta \nabla f(\x; \xi) ) - \u^{(k)} (\x - \eta \nabla F(\x)) 
        - \Diff \u^{(k)} (\x - \eta \nabla F(\x)) \left(\eta \nabla f(\x; \xi) - \eta \nabla F(\x)  \right) 
        \right\|_2 
        \tag{By Jensen's inequality}
        \\
        \leq & \frac{1}{2} \sup_{\x} \|\Diff^2 \u^{(k)}(\x) \|_2 
        \expt_{\xi} \| \eta \nabla F(\x) - \eta \nabla f(\x; \xi) \|_2^2
        \tag{By Taylor's expansion}
        \\
        \leq & \frac{1}{2} \eta k Q \eta^2 \cdot \sigma^2
        =
        \frac{1}{2} \eta^3 k Q \sigma^2.
    \end{align}
\end{enumerate}
\end{proof}

We are now ready to finish the proof of \cref{thm:3o:bias:ub:complete}.
\begin{proof}[Proof of \cref{thm:3o:bias:ub:complete}]
    By \cref{lem:bias:3o:ub}(e), for any $j \in \{0, 1, \ldots, k\}$
    \begin{equation}
        \left\|\u^{(k-j)}(\z_{\gd}^{(j)}) - \u^{(k-j-1)} (\z_{\gd}^{(j+1)}) \right\|_2 \leq \frac{1}{2} \eta^3 (k-j - 1) Q \sigma^2
    \end{equation}
    Consequently
    \begin{equation}
        \left\| \expt \x_{\sgd}^{(k)} - \z_\gd^{(k)} \right\|
        =
        \left\| \u^{(k)}(\z_{\gd}^{(0)}) - \u^{(0)} (\z_\gd^{(k)}) \right\|
        \leq
        \sum_{j=0}^{k-1} \left\|\u^{(k-j)}(\z_{\gd}^{(j)}) - \u^{(k-j-1)} (\z_{\gd}^{(j+1)}) \right\|_2
        \leq
        \frac{1}{4} \eta^3 k^2 Q \sigma^2.
    \end{equation}
\end{proof}

\subsection{Proof of \cref{thm:3o:bias:lb:complete}: Lower Bound of Iterate Bias Under 3rd-order Smoothness}
\label{sec:proof:bias:3o:lb:complete}
Before we state the proof of \cref{thm:3o:bias:lb:complete}, let us first describe the following helper function used to construct the lower bound instance. Define
\begin{equation}
    \varphi(x) = \int_0^x \log (\cosh (x)) \diff x.
    \label{eq:3o:phi}
\end{equation}
In the following lemma, we show that this $\varphi(x)$ satisfies the following properties
\begin{lemma}
\label{lem:bias:3o:lb:helper}
The following properties hold for the $\varphi(x)$ defined in  \cref{eq:3o:phi}.
\begin{enumerate}[(a)]
    \item $\varphi'(x) = \log (\cosh (x))$. Therefore $\varphi'(x) \leq |x|$. In particular $\varphi(0) = 0$.
    \item $\varphi''(x) = \tanh (x)$. In particular $\varphi''(0) = 0$, $\lim_{x \to +\infty} \varphi''(x) = 1$, $\lim_{x \to -\infty} \varphi''(x) = -1$, and $\varphi''(x) \in [-1, 1]$ for any $x \in \reals$.
    \item $\varphi'''(x) = \sech^2 (x)$. In particular $\varphi'''(0) = 1$, $\lim_{x \to +\infty} \varphi'''(x) = 0$, $\lim_{x \to -\infty} \varphi'''(x) = 0$, and $\varphi'''(x) \in [0, 1]$ for any $x \in \reals$. 
    Also $\varphi'''(x) \geq \frac{1}{2}$ for any $x \in [-\frac{1}{2}, +\frac{1}{2}]$
    \item $\varphi''''(x) = -2 \sech^2(x) \tanh(x)$. In particular $\varphi''''(x) \in (-1, 1)$ for any $x \in \reals$.
\end{enumerate}
    
\end{lemma}
\begin{proof}[Proof of \cref{lem:bias:3o:lb:helper}]
    All results follow by standard trigonometry analysis.
\end{proof}

Next we establish the following lemma
\begin{lemma}
    \label{lem:bias:3o:lb:instance}
    Consider 
    \begin{equation}
        f(x; \xi) = \frac{3}{8}Hx^2 + \frac{H^3}{64Q^2} \varphi \left( \frac{4Q}{H} x \right) +  \xi, \qquad
        F(x) := \expt_{\xi \sim \mathcal{U}[-\sigma, \sigma]} f(x; \xi).
        \label{def:bias:3o:lb:instance}
    \end{equation}
    where $\varphi$ is defined in \cref{eq:3o:phi}. 
    Then
    \begin{enumerate}[(a)]
        \item $f''(x; \xi) = F''(x) = \frac{3}{4}H + \frac{1}{4}H \varphi'' \left( \frac{4Q}{H} x \right)$. Therefore $F''(x)  \in [\frac{1}{2}H, H]$ for any $x \in \reals$.
        \item $f'''(x; \xi) = F'''(x) = Q \varphi'''(\frac{4Q}{H} x )$. Therefore $F'''(x) \in [0,Q]$ for any $x \in \reals$. In particular $F'''(0) = Q$, and $F'''(x) \geq \frac{1}{2}Q$ for any $x \in [-\frac{H}{8Q}, +\frac{H}{8Q}]$.
        \item $f(x; \xi)$ satisfies \cref{asm:main,asm:third_order}.
    \end{enumerate}
\end{lemma}
\begin{proof}[Proof of \cref{lem:bias:3o:lb:instance}]
    (a,b) follow from \cref{lem:bias:3o:lb:helper}. (c) follows by (a, b) and the fact that the variance of $\mathcal{U}[-\sigma, +\sigma] \leq \sigma^2$.
\end{proof}

The following lemma studies the SGD trajectory on $f$ defined in \cref{def:bias:3o:lb:instance}.
\begin{lemma}
\label{lem:bias:3o:lb:enum}
Let $\{\x_{\sgd}^{(k)}\}_{k=0}^{\infty}$ be the SGD trajectory on the function $f$ defined in \cref{def:bias:3o:lb:instance}, with learning rate $\eta$, that is 
\begin{equation}
    x_{\sgd}^{(k+1)} \gets x_{\sgd}^{(k)} - \eta \cdot f'(x_{\sgd}^{(k)}; \xi^{(k)}), \qquad \xi^{(k)} \sim \mathcal{U}[-\sigma, +\sigma].
\end{equation}
Define
\begin{equation}
    u_k(x) := \expt [x_{\sgd}^{(k)} | x_{\sgd}^{(0)} = x].
\end{equation}
Then the following results hold
\begin{enumerate}[(a)] 
    \item  $u_{k+1}(x) = \expt_{\xi} \left[ u_k (x - \eta f'(x; \xi)) \right]$
    \item $u_{k+1}'(x) = \expt_{\xi} \left[(1 - \eta F''(x)) \cdot u_k'(x - \eta f'(x; \xi))  \right]$.
    \item $u_{k+1}''(x) = \expt_{\xi} \left[ (1 - \eta F''(x))^2 u_k''(x - \eta f'(x; \xi)) - \eta F'''(x) u_k'(x - \eta f'(x;\xi))\right]$. 
    \item For any $k$, $\inf_{x} \{u_k'(x) \} \geq (1- \eta H)^k$  holds.
    \item For any $k$, $\sup_{x} \{u_k''(x)\} \leq 0$.
    \item For any $x \in \reals$ and $k$, it is the case that $u''_{k+1}(x) \leq (1 - \eta H)^2 \expt_{\xi}[u_k''(x - \eta f'(x; \xi))] - \eta (1 - \eta H) F'''(x)$.
\end{enumerate}
\end{lemma}
\begin{proof}[Proof of \cref{lem:bias:3o:lb:enum}]
\begin{enumerate}[(a)]
    \item Proved in \cref{lem:bias:3o:ub}(a).
    \item Proved in \cref{lem:bias:3o:ub}(b).
    \item Holds by taking derivative with respect to $x$ on both sides of (b).
    \item Since $F''(x) \in [\frac{1}{2}H, H]$, by (b), we have
    \begin{equation}
        \inf_{x} \{u_{k+1}'(x)\} \leq (1 - \eta H) \inf_x \{ u_k'(x)\}.
    \end{equation}
    By definition of $u_0$ we have $u_0(x) \equiv x$ and thus $u_0'(x) \equiv 1$. Telescoping the above inequality gives (d).
    \item 
    We prove by induction. For $k = 0$ we have $u_0''(x) \equiv 0$ which clearly satisfies (e). Now assume (e) holds for the case of $k$, and we study the case of $k+1$.

    Since $F''(x) \in [\frac{1}{2}H, H]$ and $F'''(x) \geq 0$, by (c) and (d), we have
    \begin{equation}
        \sup_{x} \{ u_{k+1}''(x) \}
        \leq
        (1 - \eta H)^2 \sup_{x} \{ u_{k}''(x) \}
        -
        \eta  \inf_{x} \{F'''(x)\} (1 - \eta H)^k
        \leq 0,
    \end{equation}
    completing the induction.
    
    \item Holds by (c-e).
    \end{enumerate}
\end{proof}

We further have the following lemma.
\begin{lemma}
    \label{lem:bias:3o:lb:enum:2}
    Under the same setting of \cref{lem:bias:3o:lb:enum}, the following results hold.
    \begin{enumerate}[(a)]
    \item For any $x \in [-\frac{H}{8Q}, \frac{H}{8Q}]$ and $k$, 
    \begin{equation}
         u''_{k+1}(x) \leq (1 - \eta H)^2 \sup_{z \in [x - \eta\sigma, x+ \eta \sigma]} \{ u_k''(z) \} - \eta (1 - \eta H) \frac{Q}{2}.
    \end{equation}
    \item Assuming $Q \leq \frac{H}{24 \eta K \sigma}$, then for any $k < K$, the following inequality holds
        \begin{equation}
         \sup_{x \in [-\frac{H}{12Q}, +\frac{H}{12Q}]} u''_{k}(x) \leq - \sum_{j=0}^{k-1} (1-\eta H)^{2j
         +1} \cdot \frac{\eta Q}{2}.
        \end{equation}
    \item Assuming $\eta \leq \frac{1}{2H}$ and $Q \leq \frac{H}{24 \eta K \sigma}$, then for any $k < K$, for any $x \in [-\frac{H}{24Q}, +\frac{H}{24Q}]$, one has
        \begin{equation}
            u_{k+1}(x) \leq  u_k(x - \eta F'(x)) -  \frac{1}{12} \eta^3 \sigma^2 Q \sum_{j=0}^{k-1} (1 - \eta H)^{2j+1}.
        \end{equation}
    \end{enumerate}
\end{lemma}
\begin{proof}[Proof of \cref{lem:bias:3o:lb:enum:2}]
    \begin{enumerate}[(a)]
        \item   Holds by (f) and the fact that $|f'(x;\xi) - F'(x)| \leq \eta \sigma$ and $\inf_{x \in [-\frac{H}{8Q}, \frac{H}{8Q}]} F'''(x) \geq \frac{Q}{2}$.
        \item Since $ \frac{H}{12Q} + \eta \sigma K \leq \frac{H}{8Q}$ (due to the assumption that $Q \leq \frac{H}{24 \eta K \sigma}$), we can repeatedly apply (a) for $K$ times. Therefore
        \begin{align}
                & \sup_{x \in [-\frac{H}{12Q}, +\frac{H}{12Q}]} \{u''_{k}(x) \} 
                \\
            \leq &
            (1 - \eta H)^2 \sup_{x \in [-\frac{H}{12Q} - \eta \sigma, \frac{H}{12Q} + \eta \sigma]}
             \{u''_{k-1}(x) \} - \eta (1-\eta H)\frac{Q}{2}
             \\
             \leq & (1 - \eta H)^{2k} \sup_{x \in [-\frac{H}{12Q} - \eta k \sigma, \frac{H}{12Q} + \eta k \sigma]}
             \{u''_{0}(x) \} - \eta \sum_{j=0}^{k-1} (1-\eta H)^{2j} (1-\eta H)\frac{Q}{2}.
        \end{align}
        Plugging in $u_0''(x) \equiv 0$ gives (b).
        \item By \cref{lem:bias:3o:lb:enum}(a),
        \begin{align}
            & u_{k+1}(x) - u_k(x - \eta F'(x)) = \expt_{\xi} \left[ u_k (x - \eta f'(x; \xi)) - u_k(x - \eta F'(x))\right] 
            \\
        \leq & \expt_{\xi} \left[ - \eta \cdot u_k'(x - \eta F'(x)) \cdot (f'(x; \xi) - F'(x))
        + \frac{1}{2} \sup_{z \in [x - \eta F'(x) - \eta \sigma, x - \eta F'(x) + \eta \sigma]} u''(z) \cdot \eta^2 (f'(x; \xi) - F'(x))^2 \right] 
        \\
        \leq & \frac{1}{6} \eta^2 \sigma^2 \sup_{z \in [x - \eta F'(x) - \eta \sigma, x - \eta F'(x) + \eta \sigma]} u''(z)
        \end{align}
    \end{enumerate}
    Since $x \in [-\frac{H}{24Q}, \frac{H}{24Q}]$, we know that $x - \eta F'(x) \in [-\frac{H}{24Q}, \frac{H}{24Q}]$ by construction of $F$. Since $Q \leq \frac{H}{24 \eta K \sigma}$ we know that 
    $[x - \eta F'(x) - \eta \sigma, x - \eta F'(x) + \eta \sigma] \subset [-\frac{H}{12Q}, \frac{H}{12Q}]$. Therefore (b) is applicable, which suggets
    \begin{equation}
        u_{k+1}(x) - u_k(x - \eta F'(x))  \leq - \frac{1}{12} \eta^3 \sigma^2 Q \sum_{j=0}^{k-1} (1 - \eta H)^{2j+1}.
    \end{equation}
\end{proof}

We are ready to finish the proof of \cref{thm:3o:bias:lb:complete} now.
\begin{proof}[Proof of \cref{thm:3o:bias:lb:complete}]
    For $k=1$ the bound trivially holds. From now on assume $k \geq 2$.

    Consider the one-dimensional instance $f$ defined in \cref{def:bias:3o:lb:instance}. The optimum of $F = \expt_{\xi} f(x; \xi)$ is clearly 0.
    We will actually show a stronger result that \cref{eq:3o:bias:lb:complete} holds for any $x \in [-\frac{H}{24Q}, +\frac{H}{24Q}]$, in addition to 0.

    Since $\eta \leq \frac{1}{2H}$, for any $x \in [-\frac{H}{24Q}, +\frac{H}{24Q}]$, one has $x - \eta F'(x) \in  [-\frac{H}{24Q}, +\frac{H}{24Q}]$. Therefore one can repeatedly apply \cref{lem:bias:3o:lb:enum:2}(c), which yields
    \begin{equation}
        \expt [x_{\sgd}^{(k)}] - z_{\gd}^{(k)} \leq - \frac{1}{12} \eta^3 \sigma^2 Q \sum_{j=0}^{k-1} \sum_{i=0}^{j-1} (1 - \eta H)^{2i+1}.
    \end{equation}
    If $k \leq \frac{1}{\eta H}$ then
    \begin{equation}
        \sum_{j=1}^{k-1} \sum_{i=0}^{j-1} (1 - \eta H)^{2i+1}
        \geq
        k (k-1) (1 - \eta H)^{2k-3}
        \geq
        k (k-1) \left(1 - \frac{1}{k} \right)^{2k-3}
        \geq
        \frac{1}{e^2} k (k-1)
        \geq 
        \frac{k(k-1)}{16}.
    \end{equation}
    If $k > \frac{1}{\eta H}$ then
    \begin{equation}
        \sum_{j=1}^{k-1} \sum_{i=0}^{j-1} (1 - \eta H)^{2i+1} 
        =
        \frac{(1 - \eta H) ((1- \eta H)^{2k} + \eta H (2- \eta H) k - 1)}{\eta^2 H^2 (2 - \eta H)^2}
        \geq
        \frac{ \frac{3}{2} \eta H k - 1 }{8 \eta^2 H^2}
        \geq
        \frac{\eta H k}{16 \eta^2 H^2}
        \geq
        \frac{k - 1}{16 \eta H},
    \end{equation}
    where in the second from the last inequality we used the asssumption that $\eta H \leq \frac{1}{2}$.
    In either case we have
    \begin{equation}
        \sum_{j=1}^{k-1} \sum_{i=0}^{j-1} (1 - \eta H)^{2i+1} 
        \geq 
        \min \left\{ \frac{k - 1}{16 \eta H}, \frac{1}{16} k(k-1) \right\},
    \end{equation}
    and hence
    \begin{equation}
        \expt [x_{\sgd}^{(k)}] - z_{\gd}^{(k)} \leq
        - 0.005 \eta^3 \sigma^2 Q (k-1) \min \left\{ \frac{1}{ \eta H}, k \right\}.
    \end{equation}

\end{proof}

%% file: appendices/lb_proof.tex
\ifdefined\conf\section{\large PROOF OF THEOREMS~\ref{lb:homo} AND \ref{thm:lb}: LOWER BOUNDS OF \fedavg UNDER 2ND-ORDER SMOOTHNESS}\else\section{Proof of \cref{lb:homo,thm:lb}: Lower Bounds of \fedavg Under 2nd-Order Smoothness}\fi
\label{lb:proofs}

The main objective of this section is to prove the following Theorem, which implies both \cref{lb:homo,thm:lb}.
\begin{theorem}
\label{full_lb}
For any $K \geq 2, R, M$, $H$, $D$, $\sigma$, and $\zeta_*$, there exist $f(\x; \xi)$ and distributions $\{\mathcal{D}_m\}$, each satisfying Assumption~\ref{asm:main}, and together satisfying Assumption~\ref{asm:hetero}, such that for some initialization $\x^{(0, 0)}$ with $\|\x^{(0, 0)} - \x^{\star}\|_2 < D$, the final iterate of \fedavg with \textbf{any} step size satisfies:
\begin{equation}
\mathbb{E}\left[F(\x^{(R, 0)})\right] - F(\x^{\star}) \geq
    \Omega\left(\frac{HD^2}{KR} + \frac{\sigma D}{\sqrt{MKR}}  + \min\left\{\frac{\sigma D}{\sqrt{KR}},  \frac{H^\frac{1}{3} \sigma^{\frac{2}{3}} D^{\frac{4}{3}}}{K^{\frac{1}{3}} R^{\frac{2}{3}}} \right\}
    + \min\left\{\frac{\zeta_*^2}{H}, \frac{H^{\frac{1}{3}}\zeta_*^{\frac{2}{3}}D^{\frac{4}{3}}}{R^{\frac{2}{3}}}\right\}\right).
\end{equation}
In particular, if $\zeta_* = 0$, then it is possible to choose all distributions $\mathcal{D}_m$ to be the same distribution $\mathcal{D}$.
\end{theorem}
\begin{remark}
Note that $\zeta_* = 0$ does not necessarily imply that the distributions are homogeneous. Hence for this theorem to imply  \cref{lb:homo}, we ensure that in the case when $\zeta_* = 0$, all of the distributions $\mathcal{D}_m$ are equal.
\end{remark}

In \cref{sec:overview}, we provide a high level overview of the proof techniques of this theorem in the homogeneous case. The main technical lemmas from this overview are proved in Sections~\ref{sec:pf_exp_drift} and \ref{sec:pf_steplb_1}. In \cref{sec:hetero}, we provide the additional technical lemmas for the heterogeneous case. Finally, we finish the proof of Theorem~\ref{full_lb} in \cref{sec:pf_lb}.

\subsection{Proof Overview of Theorem~\ref{full_lb} in Homogeneous Case}\label{sec:overview}
To prove the lower bound in Theorem~\ref{full_lb} for the homogenous case, we construct the following function.
\begin{equation}\label{lb:f}
    f(\x; \xi) = f_1(x_1; \xi) + f_2(x_2) + f_3(x_3),
\end{equation}
where
\begin{equation}
    f_1(x; \xi) = \frac{L}{2}\psi(x) + \xi x, \quad 
    f_2(x) = \mu x^2, \quad
    f_3(x) = Hx^2,
\end{equation}
and 
where $$\psi(x) := \begin{cases}
    \frac{1}{2}x^2 & x < 0 \\
    x^2 & x \geq 0
    \end{cases},$$ $\xi \sim \mathcal{N}(0, \sigma^2),$ $L = H/6$, 
    and $\mu$ is some function of $K, R, D, H$, and $\sigma$.
    
We will analyze the convergence of \fedavg starting at $\x^{(0, 0)} = \left(0, D/2, D/2\right)$ such that the initial distance to optimum $\|\x^{(0, 0)} - \x^{\star}\|_2 < D$. Observe that the only noise is Gaussian noise in the gradient of the first coordinate.
    
The sole objective of including $f_3$ is to ensure that we can limit our analysis to cases with small step size, $\eta$. Indeed, by standard arguments, if $\eta \geq \frac{1}{H}$, then the third coordinate of $\x$ would diverge.

The role of the function $f_2$ is to provide a direction (the $x_2$-axis) in which $f$ is only slightly convex. Indeed, this term requires that $\eta$ is sufficiently large for convergence, which we formalize in Lemma~\ref{f2}. 

The novelty in our analysis stems from our sharp analysis of the bias $\mathbb{E}[x^{(R, 0)}_1]$ from running SGD on the piecewise quadratic function, pictured in Figure~\ref{fig:step}.

\begin{figure}
    \centering
    \begin{tabular}{ccc}
             \includegraphics[width=5cm]{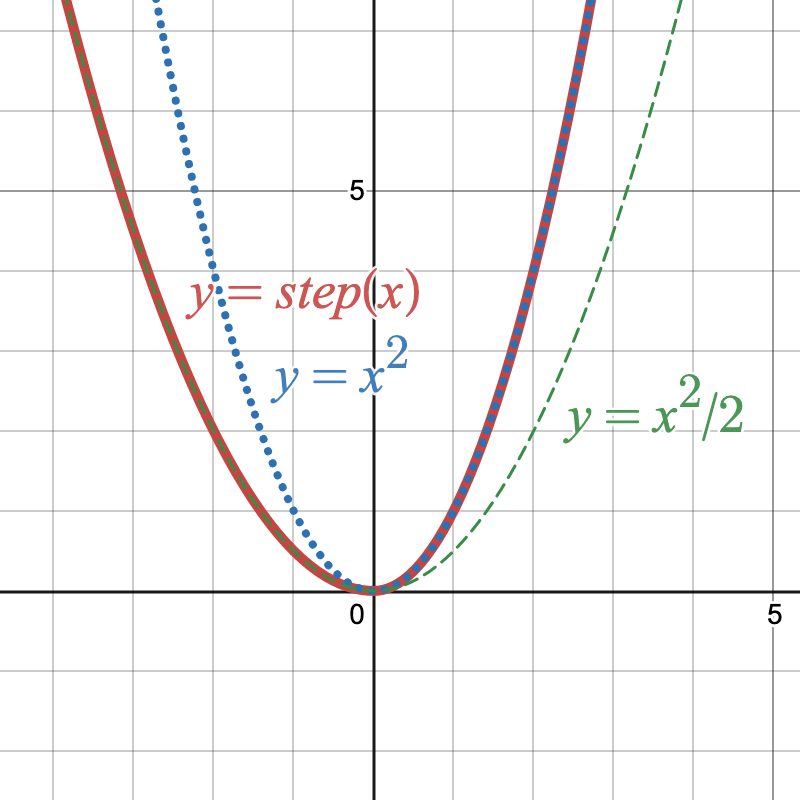} &     \includegraphics[width=5cm]{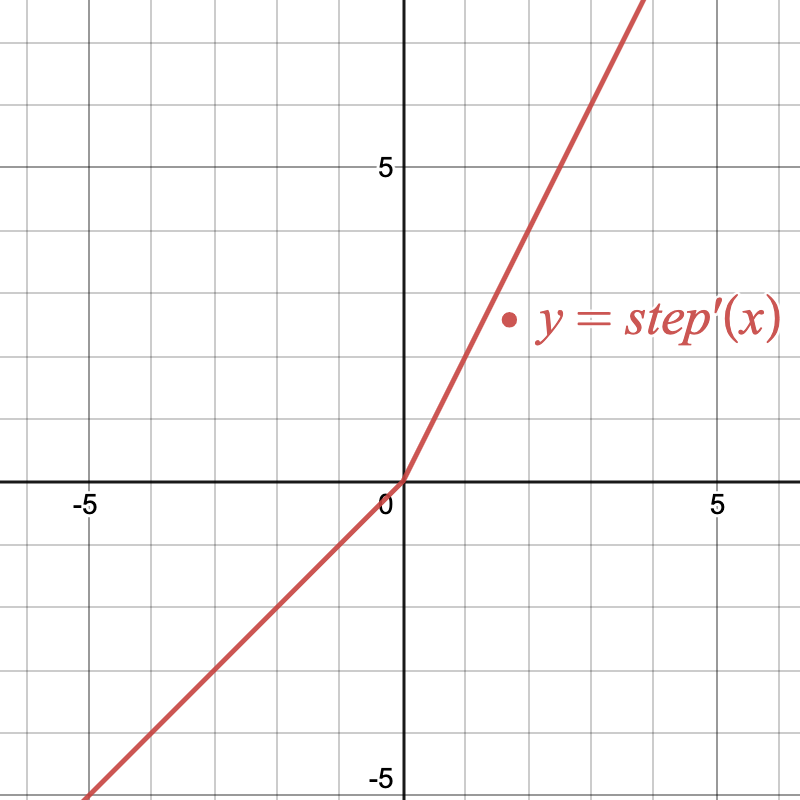} &     \includegraphics[width=5cm]{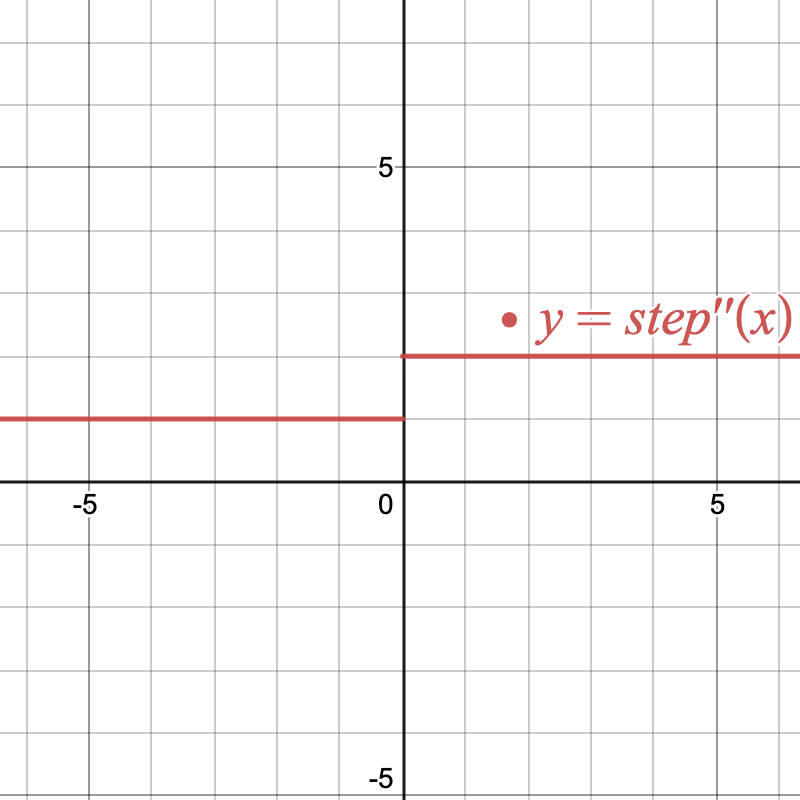} \\
        (a) & (b) & (c)  
    \end{tabular}
    \caption{The piecewise quadratic function and its first two derivatives.}
    \label{fig:step}
\end{figure}

Our main technique is comparing the iterates $x^{(0)}, x^{(1)}, \cdots$ from running SGD on the piecewise quadratic function $f_1(x; \xi)$ to the iterates $\{y^{(k)}\}$ and $\{z^{(k)}\}$ obtained from running SGD on the quadratic functions 
\begin{equation}
    f_{\ell}(x; \xi) := \frac{L}{4}x^2 + \xi x,
    \quad
    \text{and}
    \quad
    f_{u}(x; \xi) := \frac{L}{2}x^2 + \xi x,
\end{equation}
respectively. We will show in Lemma~\ref{mc_dominance} that if $x^{(0)} = y^{(0)} = z^{(0)}$, then the iterate $x^{(k)}$ is first-order stochastically dominated by both $y^{(k)}$ and $z^{(k)}$ (see Definition~\ref{def:stoc_dom} for the formal definition of first-order stochastic dominance). Fortunately, the iterates $y^{(k)}$ and $z^{(k)}$ are easy to analyze. A straightforward calculation in Lemma~\ref{closed} yields the closed form solutions
\begin{equation}
    y^{(k)} \sim \alpha_y^ky^{(0)} + \mathcal{N}(0, \sigma_y^2),
\quad
\text{and}
\quad
    z^{(k)} \sim \alpha_z^kz^{(0)} + \mathcal{N}(0, \sigma_z^2),
\end{equation}
where 
\begin{equation}\label{var_defs}
\alpha_y := 1 - \eta L/2 , \quad 
\alpha_z := 1 - \eta L, \quad
\sigma_y^2 := \frac{\eta^2 \sigma^2 (1 - \alpha_y^k)}{1 - \alpha_y}, \quad
\sigma_z^2 := \frac{\eta^2 \sigma^2 (1 - \alpha_z^k)}{1 - \alpha_z}.
\end{equation}

We can then bound the expectation of $x^{(k)}$ in the following way:
\begin{equation}
    \mathbb{E}[x^{(k)}] = -\int_{c = -\infty}^0{\Pr[x^{(k)} \leq c]} +  \int_{c = 0}^\infty{\Pr[x^{(k)} \geq c]} \leq -\int_{c = -\infty}^0{\Pr[y^{(k)} \leq c]} +  \int_{c = 0}^\infty{\Pr[z^{(k)} \geq c]}.
\end{equation}
This decomposition means that the higher variance of $y^{(k)}$ to contributes to the negative term, while the relatively lower variance of $z^{(k)}$ contributes to the positive term. 

To give some intuition, consider the case where $x^{(0)} = y^{(0)} = z^{(0)} = 0$, and $\eta L k \ll 1$. Here we have $y^{(k)} \sim \mathcal{N}(0, \sigma_y^2)$ and $z^{(k)} \sim \mathcal{N}(0, \sigma_z^2)$.
Plugging in the cdf of a Gaussian, we obtain 
\begin{equation}
\begin{split}
-\int_{c = -\infty}^0{\Pr[y^{(k)} \leq c]} +  \int_{c = 0}^\infty{\Pr[z^{(k)} \geq c]}  = -\frac{\sigma_y}{\sqrt{2\pi}} + \frac{\sigma_z}{\sqrt{2\pi}}.
\end{split}
\end{equation}

Using the fact that $\eta L k \ll 1$, we can approximate 
$$\sigma_y^2 \approx \frac{\eta^2\sigma^2(\eta L k/2 + (\eta L k)^2/8)}{\eta L/2} = \eta^2\sigma^2k(1 - \eta L k/4),$$ and  $$\sigma_z^2 \approx \frac{\eta^2\sigma^2(\eta L k/2 + (\eta L k)^2/2)}{\eta L} = \eta^2\sigma^2k(1 - \eta L k/2),$$ such that 
\begin{equation}
    \mathbb{E}[x^{(k)}] \leq -\frac{\sigma_y}{\sqrt{2\pi}} + \frac{\sigma_z}{\sqrt{2\pi}} \approx \frac{\eta \sigma \sqrt{k}}{\sqrt{2\pi}}\left(\frac{\eta L k}{8}\right).
\end{equation}
When $x^{(0)}$ is non-zero but sufficiently small, we can prove that this same negative iterate bias occurs in the expectation $\mathbb{E}[x^{(k)}] - x^{(0)}$. With slightly more effort, we can show that so long as the \em expectation \em   $\mathbb{E}[x^{(0)}]$ is sufficiently small, there is a negative drift in $\mathbb{E}[x^{(k)}] - \mathbb{E}[x^{(0)}]$. 

We formalize these observations in the following lemma. Note that this lemma also captures the case when $\eta L k \geq 1$, where $\sigma_y - \sigma_z = \Theta\left(\sigma \eta^{1/2}L^{-1/2}\right)$.

\begin{lemma}\label{full_lem:expected_step}
There exist universal constants $c_1$ and $c_2$ such that the following holds. Suppose we run SGD with step size $\eta$ on the function $f(x; \xi) = \frac{L}{2}\psi(x) + \xi x$ for $\xi \sim \mathcal{N}(0, \sigma^2)$ with step size $\eta \leq \frac{1}{6L}$, starting at a possibly random iterate $x^{(0)}$. If
$$-\sqrt{c_1}\frac{\sigma_y}{\alpha_y^k} \leq \mathbb{E}[x^{(0)}] \leq 0,$$ then for any $k$,
\begin{equation}
    \mathbb{E}[x^{(k)}] \leq \left(1 - \eta L/2\right)^k\mathbb{E}[x^{(0)}] - \frac{1}{2} c_2\sigma \eta^{1/2} L^{-1/2}\min(1, \eta L k)^{3/2},
\end{equation}
where $\sigma_y$ and $\alpha_y$ are defined in \cref{var_defs}.
In particular, we can choose $c_1 = 0.0005$ and $c_2 = 0.002$.
\end{lemma}
The proof of \cref{full_lem:expected_step} is relegated to \cref{sec:pf_exp_drift}.

Using Lemma~\ref{full_lem:expected_step} inductively, we can show that the bias accumulates over many rounds of \fedavg. Loosely speaking, the bias grows linearly with the number of rounds $R$ until the force of the gradient exceeds the drift from the difference $\sigma_y -\sigma_z$. 

\begin{lemma}\label{lem:steplb_1}
Suppose we run \fedavg for $R$ rounds with $K$ local steps and step size $\eta$ on the 1-dimensional function $f(x; \xi) = \frac{L}{2}\psi(x) + \xi x$ for $\xi \sim \mathcal{N}(0, \sigma^2)$. There exists a universal constant $c$ such that for $\eta \leq \frac{1}{6L}$, if $x^{(0, 0)} = 0$, then 
$$\mathbb{E}[x^{(R, 0)}] \leq -c\frac{\sqrt{\eta}\sigma}{\sqrt{L}}\min\left\{R(\eta L K)^{3/2}, 1, (\eta L K)^{1/2}\right\} $$
In particular, we can choose $c = 0.0005$. 
\end{lemma}

The proof of \cref{lem:steplb_1} is relegated to \cref{sec:pf_steplb_1}.

Now consider the \fedavg procedure on the 3-dimensional objective $f$ defined in \cref{lb:f}.
Since the trajectories of coordinates of $\x = (x_1, x_2, x_3)$ are completely decoupled, we prove Theorem~\ref{full_lb} by combining Lemma~\ref{lem:steplb_1} with bounds that relate the choice of $\eta$ to the suboptimality in the coordinates $x_2$ and $x_3$. This yields the first term, $O\left(\min\left(\frac{\sigma D}{\sqrt{KR}}, \frac{\sigma^{2/3}H^{1/3}D^{4/3}}{K^{1/3}R^{2/3}}\right)\right)$, in Theorem~\ref{full_lb}. To obtain the final term, we recall that any first order method which uses at most $MKR$ stochastic gradients has a lower bound of $O(\sigma D/\sqrt{MKR})$ in expected function error. It follows immediately that the function error of \fedavg is at least the maximum these two terms, which is on the same order as their sum. 
The details of the proof of \cref{full_lb} are provided in \cref{sec:pf_lb}.

\subsection{Proof of Lemma~\ref{full_lem:expected_step}}\label{sec:pf_exp_drift}

As outlined in the proof overview, we will compare the iterates of SGD on the piecewise quadratic function to the iterates of SGD on quadratic functions. The following lemma gives a closed form for the SGD iterates of a quadratic function. 
\begin{lemma}[Distribution of SGD on Quadratic objectives]\label{closed}
Let $x^{(0)} \cdots x^{(t)}$ be the iterates of SGD on the stochastic function $f(x; \xi) = \frac{L}{2}x^2 + \xi x$ with step size $\eta$ and $\xi \sim \mathcal{N}(0, \sigma^2)$. Then $x^{(t)} \sim (1 - \eta L)^tx^{(0)} + \mathcal{N}\left(0, \frac{(1 - (1 - \eta L)^t)\eta^2\sigma^2}{\eta L}\right)$.
\end{lemma}

\begin{proof}[Proof of \cref{closed}]
Let $\xi^{(i)} \sim \mathcal{N}(0, \sigma^2)$, such that
\begin{equation}
    x^{(i + 1)} = x^{(i)} - \eta (L x^{(i)} + \xi^{(i)}) = (1 - \eta L)x^{(i)} - \eta \xi^{(i)}.
\end{equation}
Recursing, we have
\begin{equation}
\begin{split}
    x^{(t)} &= (1 - \eta L)^tx^{(0)} + \sum_{i = 1}^{t} \left[(1 - \eta L)^{t - i}\eta^2\xi^{(i)} \right]\\
    &\sim (1 - \eta L)^tx^{(0)} + \mathcal{N}\left(0, \sum_{i = 1}^{t}(1 - \eta L)^{t - i}
    \eta^2\sigma^2\right)\\
    &\sim (1 - \eta L)^tx^{(0)} + \mathcal{N}\left(0, \frac{(1 - (1 - \eta L)^t)\eta^2\sigma^2}{\eta L}\right).
\end{split}
\end{equation}
\end{proof}

We introduce the following definition to facilitate the proof.
\begin{definition}\label{def:stoc_dom}
A random variable $Y$ \textbf{first-order stochastically dominates} a random variable $X$ if for all values $c \in \reals$,
\begin{equation}
    \Pr[Y \geq c] \geq \Pr[X \geq c].
\end{equation}
\end{definition}

We will use the following lemma.
\begin{lemma}[Markov Chain Stochastic Dominance]\label{mc_dominance}
    Let $X_t$ and $Y_t$ be time-homogeneous discrete-time Markov chains on $\mathbb{R}$ such that for any $z$, the random variable $Y_1 | Y_0 = z$ first-order stochastically dominates $X_1|X_0 = z$. Then for any $c$ and any $t > 0$, $Y_t| Y_0 = c$ first-order stochastically dominates $X_t| X_0 = c$. 
\end{lemma}
\begin{proof}[Proof of \cref{mc_dominance}]
We prove this by induction on $t$. Note that it holds trivially for $t = 0$. 

Let $p_x$ be the distribution of $X_{t - 1} | X_0 = c$ and $p_y$ be the distribution of $Y_{t - 1} | Y_0 = c$, such that by our inductive hypothesis, $p_y$ stochastically dominates $p_x$.

Then for any $c$, we have
\begin{equation}
\begin{split}
    \Pr[Y_t \geq c| Y_{t - 1} \sim p_y] &= \Pr[Y_1 \geq c| Y_{0} \sim p_y]\\
    &\geq \Pr[X_1 \geq c| X_{0} \sim p_y]\\
    &\geq \Pr[X_1 \geq c | X_{0} \sim p_x] \\
    &= \Pr[X_t \geq c | X_{t-1} \sim p_x].    
\end{split}
\end{equation}
Here the equalities follow from the fact that $X_t$ and $Y_t$ are Markov chains. The first inequality follows from our assumption that for any $z$, $Y_1 | Y_0 = z$ first-order stochastically dominates $X_1|X_0 = z$. The second inequality follows from the fact that the function $f(z) := \Pr[X_1 \leq c| X_0 = z]$ is increasing in $z$, so its expectation is at least as large under $z \sim p_y$ as under $z \sim p_x$. 
\end{proof}

Lemma~\ref{full_lem:expected_step}, the more general form of Lemma~\ref{lemma:expected_step},  gives the bias of SGD on the piecewise quadratic function if the expectation of the starting iterate is bounded. The most important part of its proof is the following weaker lemma, which gives the bias is SGD if the first iterate is deterministic and bounded.

Recall the variables introduced in the proof overview \cref{var_defs}, which we restate here for ease of reference:
\begin{equation}
\alpha_y := 1 - \eta L /2, \quad 
\alpha_z := 1 - \eta L, \quad 
\sigma_y^2 := \frac{\eta^2 \sigma^2 (1 - \alpha_y^k)}{1 - \alpha_y}, \quad
\sigma_z^2 := \frac{\eta^2 \sigma^2 (1 - \alpha_z^k)}{1 - \alpha_z}.
\label{eq:key-def}
\end{equation}

\begin{lemma}\label{lemma:masterlb}
If $\eta L \leq 1/6$ and
\begin{equation}
    - \sqrt{c_1} \frac{\sigma_y}{\alpha_y^k} \leq x^{(0)} \leq \sqrt{c_1} \frac{\sigma_y}{\alpha_y^k},
\end{equation}
then
\begin{equation}
    \mathbb{E}[x^{(k)} | x^{(0)}] \leq \max\left\{\left(1 - \eta L\right)^k x^{(0)}, \left(1 - \eta L/2\right)^k x^{(0)}\right\}- c_2 \frac{\sigma \sqrt{\eta}}{\sqrt{L}}\min\left(\eta L k, 1\right)^{3/2},
\end{equation}
where $c_1 = 0.0005$ and $c_2 = 0.002$.
\end{lemma}
The proof of Lemma~\ref{lemma:masterlb} is deferred to \cref{sec:proof:lemma:masterlb}.

The following lemma covers the edge cases when $|x^{(0)}|$ is large. 
\begin{lemma}\label{lemma:lblarge}
For all $x^{(0)} \in \reals$, 
\begin{equation}
    \mathbb{E}[x^{(k)}| x^{(0)}] \leq \left(1 - \eta L\right)^k x^{(0)},
\end{equation} and 
\begin{equation}
    \mathbb{E}[x^{(k)}| x^{(0)}] \leq \left(1 - \frac{\eta L}{2}\right)^k x^{(0)}.
\end{equation}

\end{lemma}

We begin by proving \cref{lemma:lblarge}. 
\begin{proof}[Proof of \cref{lemma:lblarge}]
This follows immediately from Lemma~\ref{mc_dominance}, since $x^{(k)}$ is stochastically dominated by $k$ steps the SGD processes $y^{(k)}$  and $z^{(k)}$ on the functions $f_{\ell}(x, \xi) := \frac{L}{4}x^2 + \xi x$  and $f_{u}(x; \xi) := \frac{L}{2}x^2 + \xi x$ respectively, with $y^{(0)} = z^{(0)} = x^{(0)}$. Indeed by Lemma~\ref{closed}, $$\mathbb{E}[x^{(k)}] \leq \mathbb{E}[z^{(k)}] = \left(1 - \eta L\right)^k x^{(0)},$$ and 
$$\mathbb{E}[x^{(k)}] \leq \mathbb{E}[y^{(k)}] = \left(1 - \frac{\eta L}{2}\right)^k x^{(0)}.$$
\end{proof}

We now prove Lemma~\ref{full_lem:expected_step} from Lemmas~\ref{lemma:masterlb} and \ref{lemma:lblarge}.

\begin{proof}[Proof of Lemma~\ref{full_lem:expected_step}]
We divide the proof into two cases. Let $B := \sqrt{c_1}\frac{\sigma_y}{\alpha_y^k}$ and let $\delta := c_2\sigma \eta^{1/2} L^{-1/2}\min(1, \eta L k)^{3/2}$.

\textbf{Case 1:} $\Pr[-B \leq x^{(0)} \leq B] \geq \frac{1}{2}$.
In this case, using the second statement of \cref{lemma:lblarge} in the first inequality and \cref{lemma:masterlb} in the second inequality, we achieve
\begin{equation}
\begin{split}
\mathbb{E}[x^{(k)}] &= \Pr[|x^{(0)}| \geq B]\mathbb{E}[x^{(k)} \mid |x^{(0)}| \geq B] + \Pr[|x^{(0)}| \leq B]\mathbb{E}\left[x^{(k)} \middle| |x^{(0)}| \leq B \right] \\
&\leq \Pr[|x^{(0)}| \geq B]\mathbb{E}[\left(1 - \eta L/2\right)^k x^{(0)} \mid |x^{(0)}| \geq B] + \Pr[|x^{(0)}| \leq B]\mathbb{E}\left[x^{(k)} \middle| |x^{(0)}| \leq B \right] \\
&\leq \Pr[|x^{(0)}| \geq B]\mathbb{E}[\left(1 - \eta L/2\right)^k x^{(0)} \mid |x^{(0)}| \geq B] + \Pr[|x^{(0)}| \leq B]\mathbb{E}[\left(1 - \eta L\right)^k x^{(0)} - \delta \mid |x^{(0)}| \leq B] \\
&= \left(1 - \eta L/2\right)^k\mathbb{E}[x^{(0)}] - \delta \Pr[|x^{(0)}| \leq B] \\
& \leq \left(1 - \eta L/2\right)^k \mathbb{E}[x^{(0)}] - \frac{\delta}{2}.
\end{split}
\end{equation}
This gives the desired result.

\textbf{Case 2:} $\Pr[-B \leq x^{(0)} \leq B] \leq \frac{1}{2}$.
In this case, using the first statement of Lemmas~\ref{lemma:lblarge} in the first inequality, and the second statement of Lemma~\ref{lemma:lblarge} in the second inequality, we have
\begin{equation}\label{eq:case2}
\begin{split}
    \mathbb{E}[x^{(k)}] &= \Pr[x^{(0)} \leq B]\mathbb{E}[x^{(k)} \mid x^{(0)} \leq B] + \Pr[x^{(0)} > B]\mathbb{E}[x^{(k)} \mid x^{(0)} > B] \\
    &\leq \Pr[x^{(0)} \leq B]\mathbb{E}\left[\left(1 - \frac{\eta L}{2}\right)^k x^{(0)} \mid x^{(0)} \leq B\right] + \Pr[x^{(0)} > B]\mathbb{E}\left[x^{(k)} \mid x^{(0)} > B\right]\\
    &\leq \Pr[x^{(0)} \leq B]\mathbb{E}\left[\left(1 - \frac{\eta L}{2}\right)^k x^{(0)} \mid x^{(0)} \leq B\right] + \Pr[x^{(0)} > B]\mathbb{E}\left[\left(1 - \eta L\right)^k x^{(0)} \mid x^{(0)} > B\right] \\
    &= \left(1 - \frac{\eta L}{2}\right)^k\mathbb{E}[x^{(0)}] + \Pr[x^{(0)} > B]\mathbb{E}\left[\left(\left(1 - \eta L\right)^k - \left(1 - \frac{\eta L}{2}\right)^k\right)x^{(0)} \mid x^{(0)} > B\right]. \\
\end{split}
\end{equation}
Now
\begin{equation}
\begin{split}
    (1 - \eta L/2)^k - (1 - \eta L)^k &\geq \begin{cases}
        1 - \eta L k /2 - (1 - \eta L k + (\eta L k)^2/2) & \eta L k \leq 1/2; \\
        \alpha_y^k\left(1 - (1 - \eta L/2)^k\right) & \eta L k \geq 1/2,
    \end{cases} \\
    &\geq \begin{cases}
        \eta L k/4 & \eta L k \leq 1/2; \\
        \alpha_y^k(1 - e^{-\eta L k/2})& \eta L k \geq 1/2.
    \end{cases} \\
        &\geq \begin{cases}
        \eta L k/4 & \eta L k \leq 1/2; \\
        \frac{\alpha_y^k}{5} & \eta L k \geq 1/2.
    \end{cases} \\
\end{split}
\end{equation}

Plugging this in to the previous equation, it follows that 
\begin{equation}\label{eq:exp_drift}
    \mathbb{E}[x^{(k)}] \leq \left(1 - \frac{\eta L}{2}\right)^k\mathbb{E}[x^{(0)}] + \begin{cases}
        \Pr[x^{(0)} > B]\mathbb{E}[x^{(0)} | x^{(0)} \geq B]\frac{\eta L k}{4} & \eta L k \leq 1/2;\\
        \Pr[x^{(0)} > B]\mathbb{E}[x^{(0)} | x^{(0)} \geq B]\frac{\alpha_y^k}{5} & \eta L k \geq 1/2.
    \end{cases}
\end{equation}

Now we can bound 
\begin{equation}
\begin{split}
\Pr[x^{(0)} > B]\mathbb{E}\left[x^{(0)} \mid x^{(0)} > B\right] &= \mathbb{E}[x^{(0)}] - \Pr[|x^{(0)}| \leq B]\mathbb{E}\left[x^{(0)} \mid |x^{(0)}| \leq B\right] - \Pr[x^{(0)} < -B]\mathbb{E}\left[x^{(0)} \mid x^{(0)} < -B\right] \\
& \geq \mathbb{E}[x^{(0)}] - \frac{1}{2}\left(B\right) - 0 \geq B - \frac{B}{2} \geq \frac{B}{2}.
\end{split}
\end{equation}
Plugging this calculation into the result of Equation~\ref{eq:exp_drift} yields
\begin{equation}
\begin{split}
\mathbb{E}[x^{(k)}] &\leq \left(1 - \eta L/2\right)^k\mathbb{E}[x^{(0)}] - \begin{cases}
    \frac{\sqrt{c_1}}{2}\frac{\sigma_y}{\alpha_y^k}\frac{\eta L k}{4} & \eta L k \leq 1/2 \\
     \frac{\sqrt{c_1}}{2}\frac{\sigma_y}{\alpha_y^k}\frac{\alpha_y^k}{5} & \eta L k \geq 1/2
\end{cases}\\
&\leq \left(1 - \eta L/2\right)^k\mathbb{E}[x^{(0)}] - \begin{cases}
    \frac{\sqrt{c_1}}{2}\frac{\eta \sigma \sqrt{k}}{2}\frac{\eta L k}{4} & \eta L k \leq 1/2 \\
     \frac{\sqrt{c_1}}{2}\sigma_y\frac{1}{5} & \eta L k \geq 1/2
\end{cases}\\
&\leq \left(1 - \eta L/2\right)^k\mathbb{E}[x^{(0)}] - \begin{cases}
    \frac{\sqrt{c_1}}{2}\frac{\eta \sigma \sqrt{k}}{2}\frac{\eta L k}{4} & \eta L k \leq 1/2 \\
     \frac{\sqrt{c_1}}{2}\frac{2\sqrt{\eta}\sigma}{\sqrt{L}}\frac{1}{5} & \eta L k \geq 1/2
\end{cases}\\
&\leq \left(1 - \eta L/2\right)^k\mathbb{E}[x^{(0)}] - \begin{cases}
    \frac{\sqrt{c_1}}{16c_2}\delta & \eta L k \leq 1/2 \\
     \frac{\sqrt{c_1}}{5c_2}\delta & \eta L k \geq 1/2
\end{cases} \\
&\leq \left(1 - \eta L/2\right)^k\mathbb{E}[x^{(0)}] - \frac{\delta}{2}.
\end{split}
\end{equation}
This proves the lemma.
\end{proof}

\subsubsection{Deferred proof of \cref{lemma:masterlb}}\label{sec:proof:lemma:masterlb}
\begin{proof}[Proof of Lemma~\ref{lemma:masterlb}]
Let $f_{\ell}(x, \xi) := \frac{L}{4}x^2 + \xi x$ and $f_{u}(x, \xi) := \frac{L}{2}x^2 + \xi x$. Let $y^{(k)}$ be the iterates of SGD on $f_{\ell}$ and let $z^{(k)}$ be the iterates of SGD on $f_u$, both initialized at $y^{(0)} = z^{(0)} = x^{(0)}$, with $\xi \sim \mathcal{N}(0, \sigma^2)$.

Then by Lemma~\ref{mc_dominance}, 
\begin{equation}
\begin{split}
    \mathbb{E}[x^{(k)}] &= -\int_{c = -\infty}^0{\Pr[x^{(k)} \leq c]} +  \int_{c = 0}^\infty{\Pr[x^{(k)} \geq c]} \leq -\int_{c = -\infty}^0{\Pr[y^{(k)} \leq c]} +  \int_{c = 0}^\infty{\Pr[z^{(k)} \geq c]}\\
\end{split}    
\end{equation}

By Lemma~\ref{closed}, we have $$y^{(k)} \sim \alpha_y^k x^{(0)} + \mathcal{N}\left(0, \frac{(1 - \alpha_y^k)\eta^2\sigma^2}{1 - \alpha_y}\right), \quad \text{and} \quad z^{(k)} \sim \alpha_z^k x^{(0)} + \mathcal{N}\left(0, \frac{(1 - \alpha_z^k)\eta^2\sigma^2}{1 - \alpha_z}\right).$$

Now with $Y \sim \mathcal{N}\left(0, \sigma^2_y\right)$ for $\sigma^2_y = \frac{(1 - \alpha_y^k)\eta^2\sigma^2}{1 - \alpha_y}$, we have
\begin{equation}
\begin{split}
     \Pr[y^{(k)} \leq c] = \Pr\left[Y \leq c - \alpha_y^kx^{(0)}\right],
\end{split}
\end{equation}     
so
\begin{equation}
\begin{split}
\int_{c = -\infty}^0\Pr\left[Y \leq c - \alpha_y^kx^{(0)}\right] &= \mathbb{E}\left[-\alpha_y^kx^{(0)} - Y| Y \leq -\alpha_y^kx^{(0)}\right]\Pr\left[Y \leq - \alpha_y^kx^{(0)}\right]\\
&= \left(-\alpha_y^kx^{(0)} - \mathbb{E}\left[Y| Y \leq -\alpha_y^kx^{(0)}\right]\right)\Pr\left[Y \leq - \alpha_y^kx^{(0)}\right].
\end{split}
\end{equation}

Now for any $a$,
\begin{equation}
    \mathbb{E}[Y |Y \leq a]\Pr[Y \leq a] = \frac{1}{\sigma_y\sqrt{2\pi}}\int_{t = -\infty}^{a}te^{-\frac{t^2}{2\sigma_y^2}}dt = -\frac{\sigma_y^2}{\sigma_y\sqrt{2\pi }} e^{-\frac{t^2}{2\sigma_y^2}} \Bigg|_{t = -\infty}^{t = a} = -\sigma_y\frac{e^{-\frac{a^2}{2\sigma^2_y}}}{\sqrt{2\pi}}.
\end{equation}
Hence we have 
\begin{equation}
\int_{c = -\infty}^0{\Pr[y^{(k)} \leq c]} = -\alpha_y^kx^{(0)}\Pr\left[\mathcal{N}(0, \sigma_y^2) \leq -a_y^kx^{(0)}\right] + \sigma_y\frac{e^{-\frac{(\alpha_y^{k}x^{(0)})^2}{2\sigma_y}}}{\sqrt{2\pi}}.
\end{equation}

Similarly, with $Z \sim \mathcal{N}\left(0, \sigma_z^2\right)$ with $\sigma_z^2 = \frac{(1 - \alpha_z^k)\eta^2\sigma^2}{1 - \alpha_z}$, we have
\begin{equation}
\begin{split}
\int_{c = 0}^\infty{\Pr[z^{(k)} \geq c]} &= \mathbb{E}\left[Z + \alpha_z^k x^{(0)} | Z \geq -\alpha_z^k x^{(0)}\right]\Pr\left[Z \geq -\alpha_z^k x^{(0)}\right]\\
&= \left(\alpha_z^k x^{(0)} + \mathbb{E}\left[Z | Z \geq -\alpha_z^k x^{(0)}\right]\right)\Pr\left[Z \geq -\alpha_z^k x^{(0)}\right] \\
&= \alpha_z^k x^{(0)}\Pr\left[\mathcal{N}(0, \sigma_z^2) \geq -\alpha_z^k x^{(0)}\right] + \sigma_z\frac{e^{-\frac{(\alpha_z^{k}x^{(0)})^2}{2\sigma_z}}}{\sqrt{2\pi}}.
\end{split}
\end{equation}
Summing, we have 
\begin{equation}\label{eq:bias}
\begin{split}
    \mathbb{E}[x^{(k)}]
    &= x^{(0)}\left(\alpha_z^k\Pr\left[\mathcal{N}(0, \sigma_z^2) \geq -\alpha_z^k x^{(0)}\right] + \alpha_y^k\Pr\left[\mathcal{N}(0, \sigma_y^2) \leq -\alpha_y^k x^{(0)}\right]\right)\\
    & \qquad + \sigma_z\frac{e^{-\frac{(\alpha_z^{k}x^{(0)})^2}{2\sigma^2_z}}}{\sqrt{2\pi}} - \sigma_y\frac{e^{-\frac{(\alpha_y^{k}x^{(0)})^2}{2\sigma_y^2}}}{\sqrt{2\pi }}
\end{split}
\end{equation}

We bound the terms in Equation~\ref{eq:bias} with the following three claims. The proofs of \cref{claim:bias1,claim:bias2,sigma_diff} are deferred to the end of the present subsubsection.
\begin{claim}\label{claim:bias1}
\begin{equation}
    \sigma_z\frac{e^{-\frac{(\alpha_z^{k}x^{(0)})^2}{2\sigma^2_z}}}{\sqrt{2\pi}} - \sigma_y\frac{e^{-\frac{(\alpha_y^{k}x^{(0)})^2}{2\sigma_y^2}}}{\sqrt{2\pi }} \leq \frac{\sigma_z - \sigma_y}{\sqrt{2\pi}}\left(e^{-c_1}\right) + \begin{cases}c_1\eta^2 \sigma L k^{3/2} & \eta L k \leq 1/2; \\
    \frac{\eta^{1/2} \sigma c_1}{L^{1/2}\sqrt{2\pi}} & \eta L k \geq 1/2.
    \end{cases}
\end{equation}
\end{claim}
\begin{claim}\label{claim:bias2}
\begin{equation}\label{eq:1}
\begin{split}
    x^{(0)}&\left(\alpha_z^k\Pr\left[\mathcal{N}(0, \sigma_z^2) \geq -\alpha_z^k x^{(0)}\right] + \alpha_y^k\Pr\left[\mathcal{N}(0, \sigma_y^2) \leq -\alpha_y^k x^{(0)}\right]\right)\\  
    &\leq x^{(0)}\alpha_y^k +  c_1\left(\sigma_y - \sigma_x\right) + \begin{cases}
         \sqrt{c_1}\eta^2 \sigma L k^{3/2} & \eta L k \leq 1/2; \\
         \frac{\sqrt{c_1}}{2}\frac{\sqrt{\eta} \sigma }{\sqrt{L}} & \eta L k \geq 1/2.
    \end{cases}
\end{split}
\end{equation}
\end{claim}
\begin{claim}\label{sigma_diff}
\begin{equation}
\sigma_y - \sigma_z \geq \begin{cases}
 \frac{1}{24}\eta^2 L \sigma k^{3/2} & \eta L k \leq 1/2; \\
 \frac{0.12 \eta \sigma}{\sqrt{\eta L}} & \eta L k \geq 1/2.
\end{cases}
\end{equation}
\end{claim}

Combining the results of Claims~\ref{claim:bias1} and \ref{claim:bias2} with Equation~\ref{eq:bias}, we obtain
\begin{equation}
\begin{split}
    \mathbb{E}[x^{(k)}] &\leq x^{(0)}\alpha_y^k + \left(\sigma_z - \sigma_y\right)\left(\frac{e^{-c_1}}{\sqrt{2\pi}} - c_1\right) + \begin{cases}
        (c_1 + \sqrt{c_1}/2)\eta^2 \sigma L k^{3/2} & \eta L k \leq 1/2; \\
        \frac{\sqrt{\eta} \sigma}{\sqrt{L}}\left(\frac{c_1}{\sqrt{2\pi}} + 2\sqrt{c_1}\right) & \eta L k \geq 1/2.
    \end{cases} \\
\end{split}
\end{equation}
Plugging in Claim~\ref{sigma_diff}, we obtain for $c_1 \leq 0.0005$,
\begin{equation}
\begin{split}
    \mathbb{E}[x^{(k)}] &\leq x^{(0)}\alpha_y^k + \begin{cases}
        -\left(\frac{e^{-c_1}}{\sqrt{2\pi}} - c_1\right)\left(\frac{1}{24}\eta^2L\sigma k^{3/2}\right) + (c_1+ \sqrt{c_1}/2)\eta^2 \sigma L k^{3/2} & \eta L k \leq 1/2; \\
        -\left(\frac{e^{-c_1}}{\sqrt{2\pi}} - c_1\right)\left(\frac{0.12\sqrt{\eta} \sigma}{\sqrt{L}}\right) + \frac{\sqrt{\eta} \sigma}{\sqrt{L}}\left(\frac{c_1}{\sqrt{2\pi}} + 2\sqrt{c_1}\right) & \eta L k \geq 1/2.
    \end{cases} \\
    &= x^{(0)}\alpha_y^k + \begin{cases}
        -\left(\frac{e^{-c_1}}{\sqrt{2\pi}} - c_1 - 24c_1 - 12\sqrt{c_1}\right)\left(\frac{1}{24}\eta^2L\sigma k^{3/2}\right) & \eta L k \leq 1/2; \\
        -\left(\frac{e^{-c_1}}{\sqrt{2\pi}} - c_1 -\frac{c_1}{0.12\sqrt{2\pi}} - \frac{2\sqrt{c_1}}{0.12} \right)\left(\frac{0.12\sqrt{\eta} \sigma}{\sqrt{L}}\right) & \eta L k \geq 1/2.
    \end{cases}\\
    &\leq x^{(0)}\alpha_y^k + \begin{cases}
        -0.117\left(\frac{1}{24}\eta^2L\sigma k^{3/2}\right) & \eta L k \leq 1/2; \\
        -0.023\left(\frac{0.12\sqrt{\eta} \sigma}{\sqrt{L}}\right) & \eta L k \geq 1/2.
    \end{cases}\\
    &\leq x^{(0)}\alpha_y^k - \frac{0.002\sqrt{\eta}\sigma \min(\eta L k, 1)^{3/2}}{\sqrt{L}}.
\end{split}
\end{equation}
This concludes the proof of the lemma aside from the proof of the three claims. To prove these, observe that $\alpha_z \leq \alpha_y$, and $\sigma_z \leq \sigma_y$. 
\end{proof}

\begin{proof}[Deferred Proof of Claim~\ref{claim:bias1}]
\begin{equation}
\begin{split}
    \sigma_z\frac{e^{-\frac{(\alpha_z^{k}x^{(0)})^2}{2\sigma^2_z}}}{\sqrt{2\pi}} - \sigma_y\frac{e^{-\frac{(\alpha_y^{k}x^{(0)})^2}{2\sigma_y^2}}}{\sqrt{2\pi }} &= \frac{\sigma_z - \sigma_y}{\sqrt{2\pi}}\left(e^{-\frac{(\alpha_z^{k}x^{(0)})^2}{2\sigma^2_z}}\right) + \frac{\sigma_y}{\sqrt{2\pi}}\left(e^{-\frac{(\alpha_z^{k}x^{(0)})^2}{2\sigma^2_z}} - e^{-\frac{(\alpha_y^{k}x^{(0)})^2}{2\sigma_y^2}}\right) \\
    &\leq \frac{\sigma_z - \sigma_y}{\sqrt{2\pi}}\left(e^{-c_1/2}\right) + \frac{\sigma_y}{\sqrt{2\pi}}\left(e^{-\frac{(\alpha_z^{k}x^{(0)})^2}{2\sigma^2_z}} - e^{-\frac{(\alpha_y^{k}x^{(0)})^2}{2\sigma_y^2}}\right) \\
    &\leq  \frac{\sigma_z - \sigma_y}{\sqrt{2\pi}}\left(e^{-c_1/2}\right) + \frac{\sigma_y}{\sqrt{2\pi}}\left|\frac{(\alpha_z^{k}x^{(0)})^2}{2\sigma^2_z} - \frac{(\alpha_y^{k}x^{(0)})^2}{2\sigma_y^2} \right|\max\left(e^{-\frac{(\alpha_z^{k}x^{(0)})^2}{2\sigma^2_z}}, e^{-\frac{(\alpha_y^{k}x^{(0)})^2}{2\sigma_y^2}}\right)\\
    &\leq \frac{\sigma_z - \sigma_y}{\sqrt{2\pi}}\left(e^{-c_1/2}\right) + \frac{\sigma_y}{\sqrt{2\pi}}\left|\frac{(\alpha_z^{k}x^{(0)})^2}{2\sigma^2_z} - \frac{(\alpha_y^{k}x^{(0)})^2}{2\sigma_y^2} \right| \\
    &= \frac{\sigma_z - \sigma_y}{\sqrt{2\pi}}\left(e^{-c_1/2}\right) + \frac{\sigma_y(x^{(0)})^2}{2\sqrt{2\pi}}\left|\frac{\alpha_z^{2k}}{\sigma^2_z} - \frac{\alpha_y^{2k}}{\sigma_y^2} \right| \\
    &\leq \frac{\sigma_z - \sigma_y}{\sqrt{2\pi}}\left(e^{-c_1/2}\right) + \frac{\sigma_y(x^{(0)})^2}{2\sqrt{2\pi}}\left|\frac{\alpha_z^{2k}- \alpha_y^{2k}}{\sigma_y^2} \right|.\\
\end{split}
\end{equation}

By the assumption on $(x^{(0)})^2$ in the lemma, we have 
\begin{equation}\label{eq36}
    \frac{\sigma_y(x^{(0)})^2}{2\sqrt{2\pi}}\left|\frac{\alpha_z^{2k}- \alpha_y^{2k}}{\sigma_y^2} \right| \leq \frac{\sigma_y(x^{(0)})^2}{2\sqrt{2\pi}}\frac{\alpha_y^{2k}}{\sigma_y^2} \leq \frac{\sigma_y c_1}{2\sqrt{2\pi}} \leq \frac{\sqrt{\eta} \sigma c_1}{\sqrt{L}\sqrt{2\pi}}.
\end{equation}

Now if $\eta L k \leq 1/2$, %
\begin{equation}\label{eq37}
\begin{split}
\frac{\sigma_y(x^{(0)})^2}{2\sqrt{2\pi}}\left|\frac{\alpha_z^{2k}- \alpha_y^{2k}}{\sigma_y^2} \right| &= \frac{(x^{(0)})^2\alpha_y^{2k}}{2\sigma_y\sqrt{2\pi}}\left|\left(\frac{(1 - \eta L)}{ (1 - \eta L/2)}\right)^{2k} - 1\right| \\
&= \frac{(x^{(0)})^2\alpha_y^{2k}}{2\sigma_y\sqrt{2\pi}}\left|\left(1 - \eta L\right)^{2k} - 1\right| \\
&\leq \frac{(x^{(0)})^2\alpha_y^{2k}}{2\sigma_y\sqrt{2\pi}}\left(2\eta L k\right) \\
&\leq \sigma_y c_1\eta L k, \\
&= \left(\sqrt{\frac{(1 - (1 - \eta L/2)^k)\eta^2 \sigma^2}{\eta L/2}}\right)c_1\eta L k\\
&\leq \left(\eta \sigma \frac{1 - (1 - \eta L k/2)}{\eta L/2}\right)c_1\eta L k \\
&= c_1 \eta^2 \sigma L k^{3/2}.
\end{split}
\end{equation}
Here in the second inequality we used the condition on $(x^{(0)})^2$. This proves the claim.
\end{proof}

\begin{proof}[Deferred Proof of Claim~\ref{claim:bias2}]
Observe that
\begin{equation}
\begin{split}
x^{(0)}\Pr\left[\mathcal{N}(0, \sigma_z^2) \geq -\alpha_z^k x^{(0)}\right] \leq x^{(0)}\Pr\left[\mathcal{N}(0, \sigma_z^2) \geq -\alpha_y^k x^{(0)}\right]
\end{split}
\end{equation}
and 
\begin{equation}\label{ineq:cdf}
\begin{split}
\alpha_y^k x^{(0)}\Pr\left[\mathcal{N}(0, \sigma_y^2) \leq -\alpha_y^k x^{(0)}\right] &=\alpha_y^k x^{(0)}\Pr\left[\mathcal{N}(0, \sigma_z^2) \leq -\alpha_y^k x^{(0)}\right] - \alpha_y^k x^{(0)}\int_{c = -\alpha_y^k x^{(0)}}^0{\mu_{\sigma_z}(c) - \mu_{\sigma_y}(c)}dc \\
&\leq \alpha_y^k x^{(0)}\Pr\left[\mathcal{N}(0, \sigma_z^2) \leq -\alpha_y^k x^{(0)}\right] + \alpha_y^{2k} (x^{(0)})^2 \max_{-\infty \leq c \leq \infty}\left|\mu_{\sigma_z}(c) - \mu_{\sigma_y}(c)\right|,
\end{split}
\end{equation}
as shown in Figure~\ref{fig:cdfs}.
\begin{figure}
    \centering
\includegraphics[width=16cm]{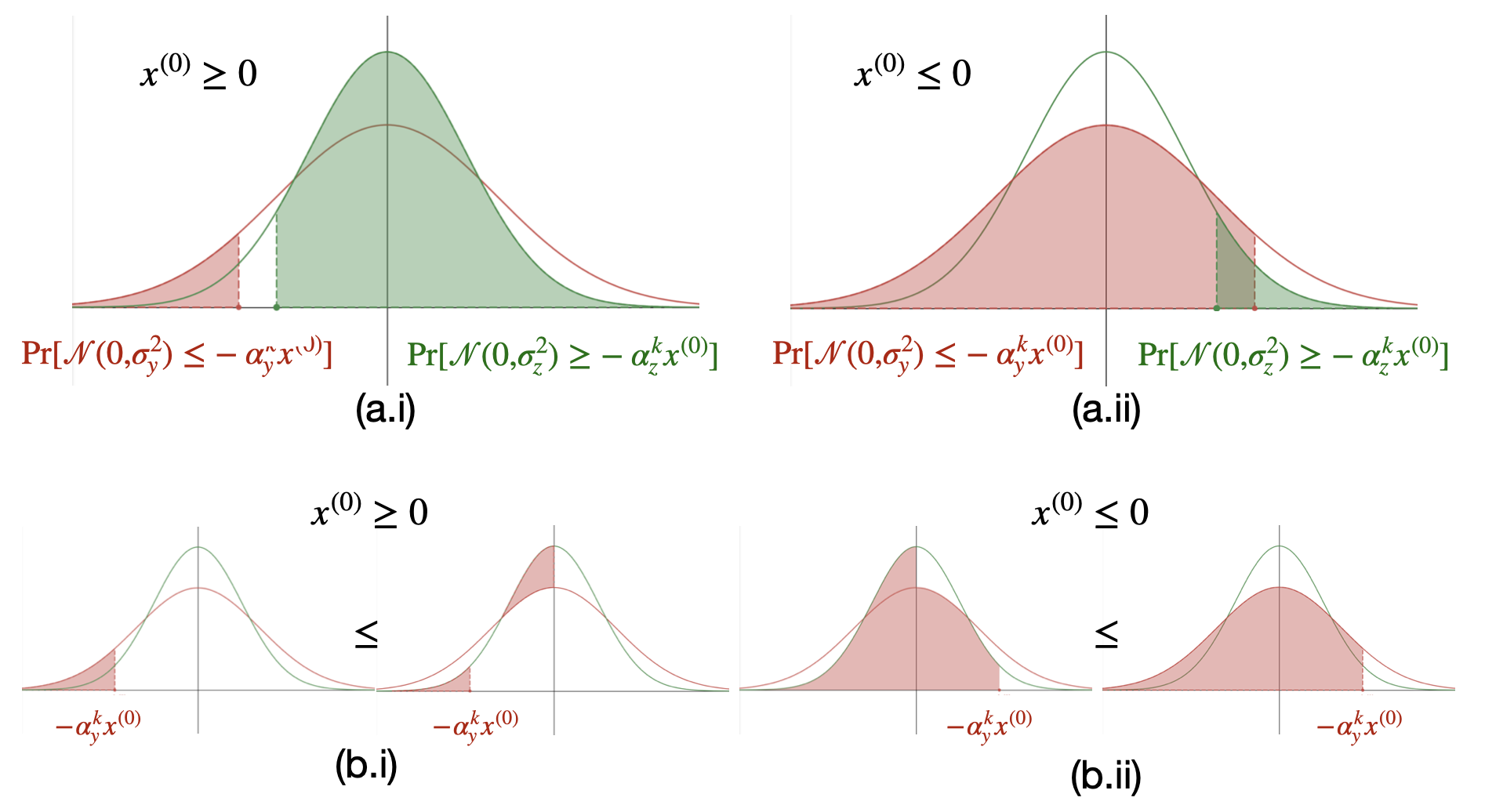}
    \caption{Comparison of CDFs. (a) and (b): the regions in the first line of Equation~\ref{eq:bias} for positive (a.i) or negative $x^{(0)}$ (a.ii). (b) The inequality used in Equation~\ref{ineq:cdf} for both positive (b.i) or negative $x^{(0)}$ (b.ii). In each of (b.i) and (b.ii), the shaded region on the right is larger than the shaded region on the left.}
    \label{fig:cdfs}
\end{figure}

Now
\begin{equation}\label{eq:2}
\begin{split}
\max_{-\infty \leq c \leq \infty} \left|\frac{1}{\sqrt{2\pi \sigma_z^2}}e^{-c^2/(2\sigma_z^2)} - \frac{1}{\sqrt{2\pi \sigma_y^2}}e^{-c^2/(2\sigma_y^2)} \right|
&= \left(\frac{1}{\sqrt{2\pi \sigma_z^2}}- \frac{1}{\sqrt{2\pi \sigma_y^2}} \right) \\
&= \left(\frac{1}{\sqrt{2\pi \sigma_z^2}}\left(1 - \frac{1}{\sqrt{1 + (\sigma_y^2 - \sigma_z^2)/\sigma_z^2}} \right)\right) \\
&\leq \left(\frac{1}{\sqrt{2\pi \sigma_z^2}}\frac{\sigma_y^2 - \sigma_z^2}{2\sigma_z^2}\right)\\
&= \left(\frac{(\sigma_y - \sigma_z)\sigma_y}{2\sigma^2_z\sqrt{2\pi \sigma_z^2}}\right)
\end{split}
\end{equation}

To bound this term, observe that 
\begin{equation}
     \frac{\sigma_y^2}{\sigma_z^2} = \frac{(1 - \alpha_y^k)(1 - \alpha_z)}{(1 - \alpha_z^k)(1 - \alpha_y)} = 2\frac{(1 - \alpha_y^k)}{(1 - \alpha_z^k)} \leq 2,
\end{equation}
so 
\begin{align}
\alpha_y^{2k}(x^{(0)})^2 \frac{\sigma_y(\sigma_y - \sigma_z)}{2\sigma_z^2\sqrt{2\pi\sigma_z^2}} &\leq  \frac{\alpha_y^k(x^{(0)})^2}{\sigma_y^2}\frac{\sigma_y(\sigma_y - \sigma_z)}{\sqrt{2\pi\sigma_z^2}} \tag{$\sigma_y^2 \leq 2\sigma_z^2$}\\
&\leq c_1\frac{\sigma_y(\sigma_y - \sigma_z)}{\sqrt{2\pi\sigma_z^2}} \tag{$(x^{(0)})^2 \leq \frac{\sigma_y^2}{\alpha_y^{2k}}$}\\
&\leq c_1\left(\sigma_y - \sigma_z\right). \tag{$\sigma_y^2 \leq 2\sigma_z^2$}
\end{align}

Plugging Equation~\ref{eq:2} into Equation~\ref{ineq:cdf}, and then plugging in this last equation, we obtain
\begin{equation}\label{eq:3}
\alpha_y^k x^{(0)}\Pr\left[\mathcal{N}(0, \sigma_y^2) \leq -\alpha_y^k x^{(0)}\right]  \leq \alpha_y^k x^{(0)} \Pr\left[\mathcal{N}(0, \sigma_z^2) \leq -\alpha_y^k x^{(0)}\right] + c_1\left(\sigma_y - \sigma_x\right).
\end{equation}

We combine Equations~\ref{eq:1} with \ref{eq:3} to yield
\begin{equation}
\begin{split}
x^{(0)}&\left(\alpha_z^k\Pr\left[\mathcal{N}(0, \sigma_z^2) \geq -\alpha_z^k x^{(0)}\right] + \alpha_y^k\Pr\left[\mathcal{N}(0, \sigma_y^2) \leq -\alpha_y^k x^{(0)}\right]\right)\\
&\leq   x^{(0)}\left(\alpha_z^k\Pr\left[\mathcal{N}(0, \sigma_z^2) \geq -\alpha_y^k x^{(0)}\right] + \alpha_y^k\Pr\left[\mathcal{N}(0, \sigma_z^2) \leq -\alpha_y^k x^{(0)}\right]\right) + c_1\left(\sigma_y - \sigma_x\right)\\
&\leq \max(x^{(0)}\alpha_y^k, x^{(0)}\alpha_z^k) + c_1\left(\sigma_y - \sigma_x\right).
\end{split}
\end{equation}

Now by similar calculations as in Equation~\ref{eq36} and \ref{eq37}, we have
\begin{equation}
\begin{split}
    \max(x^{(0)}\alpha_y^k, x^{(0)}\alpha_z^k) &\leq x^{(0)}\alpha_y^k + (\alpha_y^k - \alpha_z^k)|x^{(0)}|\\
    &\leq x^{(0)}\alpha_y^k + \begin{cases}
        \sqrt{c_1}\sigma_y \frac{\eta L k}{2} & \eta L k \leq 1/2; \\
        \sqrt{c_1}\sigma_y & \eta L k \geq 1/2.
    \end{cases} \\
    & \leq x^{(0)}\alpha_y^k + \begin{cases}
         \frac{\sqrt{c_1}}{2}\eta^2 \sigma L k^{3/2} & \eta L k \leq 1/2; \\
         \sqrt{c_1}\frac{2\sqrt{\eta} \sigma }{\sqrt{L}} & \eta L k \geq 1/2.
    \end{cases} \\
\end{split}
\end{equation}

Plugging this in to the previous equation concludes the proof of the claim.
\end{proof}
\begin{proof}[Deferred Proof of Claim~\ref{sigma_diff}]
If $\eta L k \leq 1/2$, we have the following:
\begin{align}
    \sigma_y - \sigma_z &= \eta \sigma \left(\sqrt{\frac{1 - \alpha_y^k}{1 - \alpha_y}} - \sqrt{\frac{1 - \alpha_z^k}{1 - \alpha_z}} \right)\\
    &= \eta \sigma \left(\sqrt{\frac{1 - (1 - \eta L/2)^k}{\eta L/2}} - \sqrt{\frac{1 - (1 - \eta L)^k}{\eta L}} \right) \\
    &\geq \eta \sigma \left(\sqrt{\frac{\eta Lk/2 - (\eta L)^2k(k - 1)/8}{\eta L/2}} - \sqrt{\frac{\eta Lk - (\eta L)^2k(k - 1)/2 + (\eta L)^3\binom{k}{3}}{\eta L}} \right)\\
\end{align}
because for any integer $r \geq 2$ and $0 \leq x \leq 1$, 

$$1 - rx + \binom{r}{2}x^2 - \binom{r}{3}x^3 \leq (1 - x)^r \leq 1 - rx + \binom{r}{2}x^2.$$

Continuing, we have

\begin{align}
    &= \eta \sigma \left(\sqrt{\frac{\eta Lk/2 - (\eta L)^2k(k - 1)/8}{\eta L/2}} - \sqrt{\frac{\eta Lk - (\eta L)^2k(k - 1)/2 + (\eta L)^3\binom{k}{3}}{\eta L}} \right)\\
    &= \eta \sigma \sqrt{k} \left(\sqrt{1 - \eta L (k-1)/4} - \sqrt{1 - \eta L (k-1)/2 + (\eta L)^2(k -1)(k - 2)/6} \right)\\
    & \geq \frac{\eta \sigma \sqrt{k}}{2}\left(\eta L (k-1)/4 - (\eta L)^2(k - 1)(k - 2)/6\right) \\
    & \geq  \frac{\eta \sigma \sqrt{k}}{2}\left(\eta L (k - 1)(1/4 - 1/12)\right) \tag{$\eta L k \leq 1/2$}\\
    & \geq \frac{\eta^2 L \sigma k^{3/2}}{24} \tag{$k \geq 2$}.
\end{align}
Here the first inequality follows from the fact that the derivative of $\sqrt{x}$ is at least $\frac{1}{2}$ for $0 \leq x \leq 1$.

If $\eta L k \geq 1/2$, we have 
\begin{align}
    \sigma_y - \sigma_z &= \eta \sigma \left(\sqrt{\frac{1 - \alpha_y^k}{1 - \alpha_y}} - \sqrt{\frac{1 - \alpha_z^k}{1 - \alpha_z}} \right)\\
    &= \eta \sigma \left(\sqrt{\frac{1 - (1 - \eta L/2)^k}{\eta L/2}} - \sqrt{\frac{1 - (1 - \eta L)^k}{\eta L}} \right) \\
    &\geq \eta \sigma \left(\sqrt{\frac{1 - e^{-\eta L k/2}}{\eta L/2}} - \sqrt{\frac{1 - e^{-1.1\eta L k/2}}{\eta L}} \right) \tag{$\eta L \leq 1/6$}\\
\end{align}
For $q \geq 1/2$, 
$$1 - e^{-q/2} \geq 0.52(1 - e^{-1.1q}),$$
so letting $q = \eta L k$, we have 
\begin{equation}
\begin{split}
    \sigma_y - \sigma_z &\geq \eta \sigma \left(\sqrt{2(0.52)-1} \sqrt{\frac{1 - e^{-\eta L k}}{\eta L}} \right) \geq \frac{0.12 \eta \sigma}{\sqrt{\eta L}}. 
\end{split}
\end{equation}
\end{proof}

\subsection{Proof of \cref{lem:steplb_1}}\label{sec:pf_steplb_1}
We now use Lemma~\ref{full_lem:expected_step} to prove Lemma~\ref{lem:steplb_1}, which we restate for the reader's convenience. Recall that we use the notation $x^{(r, k)}_m$ to denote the $k$ iterate of \fedavg at the client $m$ in the $r$th round. Recall that we use the notation $x^{(r, 0)}$ to denote the starting iterate at all clients in round $r$.

\begin{lemma}[Same as Lemma~\ref{lem:steplb_1}]\label{lem:steplb}
Suppose we run \fedavg for $R$ rounds with $K$ local steps per round and step size $\eta$ on the function $f(x; \xi) = \frac{L}{2}\psi(x) + \xi x$ for $\xi \sim \mathcal{N}(0, \sigma^2)$. There exists a universal constant $c$ such that for $\eta \leq \frac{1}{6L}$, if $x^{(0, 0)} = 0$, then 
$$\mathbb{E}[x^{(R, 0)}] \leq -c\frac{\sqrt{\eta}\sigma}{\sqrt{L}}\min\left(R(\eta L K)^{3/2}, 1, (\eta L K)^{1/2}\right).$$
In particular, we can choose $c = 0.0005$. 
\end{lemma}

\begin{proof}[Proof of Lemma~\ref{lem:steplb}]
Let $c_1$ and $c_2$ be the constants in Lemma~\ref{full_lem:expected_step}. For simplicity, let $q := \min(1, \eta L K)$. 
By Lemma~\ref{full_lem:expected_step}, for all $r \leq R$, if $$-\sqrt{c_1}\frac{\sigma_y}{\alpha_y^K} \leq \mathbb{E}[x^{(r, 0)}_m] \leq 0,$$
then for any client $m \in [M]$,
\begin{equation}
    \mathbb{E}[x^{(r + 1, 0)}] 
    = \mathbb{E}[x^{(r, K)}_m]
    \leq (1 - \eta L/2)^K\mathbb{E}[x^{(r , 0)}_m] - \frac{c_2\sqrt{\eta}\sigma q^{3/2}}{2\sqrt{L}} 
    = (1 - \eta L/2)^K\mathbb{E}[x^{(r , 0)}] - \frac{c_2\sqrt{\eta}\sigma q^{3/2}}{2\sqrt{L}}.
\end{equation}

If $\mathbb{E}[x^{(r, 0)}]  \leq -\sqrt{c_1}\frac{\sigma_y}{\alpha_y^K}$,
then by Lemma~\ref{mc_dominance} and comparison to SGD on the quadratic $\frac{L}{4}x^2$, we have
$$\mathbb{E}[x^{(r + 1, 0)}] \leq (1 - \eta L/2)^K\mathbb{E}[x^{(r , 0)}] \leq -\sqrt{c_1}\sigma_y \leq -\sqrt{c_1}q^{1/2}\frac{\sqrt{\eta}\sigma}{\sqrt{L}}.$$

We will prove the lemma by induction on $r$. Notice that it holds for $r = 0$. Then for all $r \geq 1$, assuming it holds for $r$, we have 

\begin{equation}\label{inductive}
    \mathbb{E}[x^{(r + 1, 0)}] \leq \begin{cases}
        -\sqrt{c_1}q^{1/2}\frac{\sqrt{\eta}\sigma}{\sqrt{L}} & \mathbb{E}[x^{(r, 0)}] \leq -\sqrt{c_1}\frac{\sigma_y}{\alpha_y^K} \\
        -(1 - \eta L/2)^Kc\min\left(r\eta^2K^{3/2}\sigma, \frac{\sqrt{\eta}\sigma q^{1/2}}{\sqrt{L}}\right) - \frac{c_2\sqrt{\eta}\sigma q^{3/2}}{2\sqrt{L}} & -\sqrt{c_1}\frac{\sigma_y}{\alpha_y^K} \leq \mathbb{E}[x^{(r, 0)}] \leq 0 \\
    \end{cases}
\end{equation}

Now for $c = \frac{c_2}{4} = 0.0005$, we have
\begin{equation}
\begin{split}
    -(1 - \eta L/2)^K&c\min\left(r\eta^2K^{3/2}\sigma, \frac{\sqrt{\eta}\sigma q^{1/2}}{\sqrt{L}}\sigma\right) - \frac{c_2\sqrt{\eta}\sigma q^{3/2}}{2\sqrt{L}} \\
    &\leq -(1 - q) c\min\left(r\eta^2K^{3/2}\sigma, \frac{\sqrt{\eta}\sigma q^{1/2}}{\sqrt{L}}\sigma\right) - \frac{c_2\sqrt{\eta}\sigma q^{3/2}}{2\sqrt{L}} \\
    &\leq -c\min\left(\left((1 - q)r + \frac{c_2}{2c}\right)\frac{\sqrt{\eta}\sigma q^{3/2}}{\sqrt{L}}, \frac{\sqrt{\eta}\sigma q^{1/2}}{\sqrt{L}} + \left(\frac{c_2}{2c} - 1\right)\frac{\sqrt{\eta}\sigma q^{3/2}}{\sqrt{L}}\right)\\
    &= -c\min\left(\left((1 - q)r + 2\right)\frac{\sqrt{\eta}\sigma q^{3/2}}{\sqrt{L}}, \frac{\sqrt{\eta}\sigma q^{1/2}}{\sqrt{L}}\right)\\
    &\leq -c\min\left(\left(r + 1\right)\frac{\sqrt{\eta}\sigma q^{3/2}}{\sqrt{L}}, \frac{\sqrt{\eta}\sigma q^{1/2}}{\sqrt{L}}\right).
\end{split}
\end{equation}

The last inequality follows from the fact that
$((1 - q)r + 2)q \geq \min((r + 1)q, 1)$, for any $0 \leq q \leq 1$.

Hence returning to \cref{inductive}, we have 
\begin{equation}\label{inductive2}
\begin{split}
    \mathbb{E}[x^{(r + 1, 0)}] &\leq \begin{cases}
        -\sqrt{c_1}q^{1/2}\frac{\sqrt{\eta}\sigma}{\sqrt{L}} & \mathbb{E}[x^{(r, 0)}] \leq -\sqrt{c_1}\frac{\sigma_y}{\alpha_y^K} \\
        -c\min\left(\left(r + 1\right)\frac{\sqrt{\eta}\sigma q^{3/2}}{\sqrt{L}}, \frac{\sqrt{\eta}\sigma q^{1/2}}{\sqrt{L}}\right) & 0 \geq \mathbb{E}[x^{(r, 0)}] \geq -\sqrt{c_1}\frac{\sigma_y}{\alpha_y^K} \\
    \end{cases}\\
    &\leq -c\min\left(\left(r + 1\right)\frac{\sqrt{\eta}\sigma q^{3/2}}{\sqrt{L}}, \frac{\sqrt{\eta}\sigma q^{1/2}}{\sqrt{L}}\right)
\end{split}
\end{equation}
since $\sqrt{c_1} \geq \frac{c_2}{4}$.
This proves the lemma.

\end{proof}

\subsection{Proof of \cref{lem:hetero_short}: Lower Bound on Bias of \fedavg with Heterogeneous Distribution}\label{sec:hetero}
We restate Lemma~\ref{lem:hetero_short} for the reader's convenience.
\begin{lemma}[Same as Lemma~\ref{lem:hetero_short}]\label{full_lemma:hetero}
Consider \fedavg with $M$ clients with $$f^{(3)}(x; (\xi_1, \xi_2)) = \begin{cases}Hx^2 - x\xi_2 & \xi_1 = 1\\ \frac{H}{2}x^2 - x\xi_2& \xi_1 = 2\\ \end{cases},$$ and for all the odd $m \in [M]$, we have $(\xi_1, \xi_2) = (1, \zeta_*)$ always, while for all the even $m \in [M]$ we have $(\xi_1, \xi_2) = (2, -\zeta_*)$. There exists a universal constant $c_h$ such that for $\eta \leq \frac{1}{H}$, if $x^{(0, 0)} \leq 0$, then \fedavg with $R$ rounds and $K$ steps per round results in 
$$x^{(R, 0)} \leq -\frac{c_h}{H} \min(1, \eta H K, (\eta H K)^2R)\zeta_*.$$ 
In particular, we can choose $c_h = 0.07$. 
\end{lemma}

\begin{proof}[Proof of \cref{full_lemma:hetero}]
As derived in \cite{Woodworth.Patel.ea-NeurIPS20}, we have the following SGD dynamics, where $\mu := H/2$:
For $0 \leq k < K$, we have
\begin{equation}
    x_1^{(r, k + 1)} = x_1^{(r, k)}(1  - \eta H) + \eta\zeta_* = (1  - \eta H)\left(x_1^{(r, 0)} - \zeta_*/H\right) + \zeta_*/H,
\end{equation}
and 
\begin{equation}
    x_2^{(r, k + 1)} = x_2^{(r, k)}(1  - \eta \mu) - \eta\zeta_* = (1  - \eta \mu)\left(x_2^{(r, 0)} - \zeta_*/\mu\right) - \zeta_*/\mu.
\end{equation}

Recursing, we have 
\begin{equation}
    x_1^{(r, K)} = (1  - \eta H)^K\left(x_1^{(r, 0)} - \zeta_*/H\right) + \zeta_*/H,
\qquad
    x_2^{(r, K)} = (1  - \eta \mu)^K\left(x_2^{(r, 0)} - \zeta_*/H\right) - \zeta_*/\mu.
\end{equation}

Since $x_1^{(r + 1, 0)} = x_2^{(r + 1, 0)} = \frac{1}{2}\left(x_1^{(r, K)} + x_2^{(r, K)} \right)$, we have for $i \in \{1, 2\}$,
\begin{equation}
    x_i^{(r + 1, 0)} = a x_i^{(r, 0)} +  b\zeta_*,
\end{equation}
where 
\begin{equation}
    a = \frac{1}{2}\left((1 - \eta H)^K + (1 - \eta \mu)^K\right),
\qquad
    b = \frac{1}{2}\left(1/H - (1 - \eta H)^K - 1/\mu + (1 - \eta \mu)^K\right).
\end{equation}

We defer the proof of the following claim to the end of the present subsection.
\begin{claim}\label{claim:b}
\begin{equation}
b \leq - \begin{cases}
-0.4\eta^2 K^2H & \eta H K \leq 1/2;\\
-0.015/H &\eta H K > 1/2.
\end{cases}
\end{equation}
\end{claim}

Recursing, we have for $x^{(0, 0)} \leq 0$ and $\mu \leq H$,
\begin{equation}
\begin{split}
    x^{(R, 0)} &= a^R x^{(0, 0)} +  \sum_{j = 0}^{R - 1}a^jb\zeta_* 
    = a^R x^{(0, 0)} + \frac{1 - a^R}{1 - a}b\zeta_* 
    \leq \frac{1 - a^R}{1 - a}b\zeta_* 
    \leq \frac{1 - a^R}{1 - (1 - \eta \mu)^K}b\zeta_*,
\end{split}
\end{equation}
where we have used the fact that $b \leq 0$ from the Claim.

Now if $\eta H K \leq 1/2$, we have
\begin{equation}
1 - (1 - \eta H)^K \leq 1 - (1 - \eta H K) = \eta H K.
\end{equation}
Trivially, 
\begin{equation}
1 - (1 - \eta H)^K \leq 1.
\end{equation}
Finally 
\begin{equation}
a \leq (1 - \eta \mu)^K \leq e^{-\eta \mu K},
\end{equation}
so
\begin{equation}
1 - a^R \geq 1 - e^{-\eta \mu KR} \geq \min(1/2, \eta \mu KR/2).
\end{equation}

Putting together these approximations, we have
\begin{equation}
\begin{split}
    x^{(R, 0)} \leq \frac{1 - a^R}{1 - (1 - \eta H)^K}b\zeta_* \leq -\frac{c}{H} \min(1, \eta H K, (\eta H K)^2R)\zeta_*,
\end{split}
\end{equation}
for $c = 0.07$. 
\end{proof}

\subsubsection{Deferred Proof of \cref{claim:b}}
\begin{proof}[Proof of Claim~\ref{claim:b}]
We follow the same steps as the proof of Claim~\ref{sigma_diff}.

If $\eta H K \leq 1/2$, we have the following:
\begin{align}
    b &= 1/H - (1 - \eta H)^K - 1/\mu + (1 - \eta \mu)^K\\
    &\leq \frac{\eta HK - (\eta H)^2K(K- 1)/2 + (\eta H)^3\binom{K}{3}}{H} -  \frac{\eta \mu K- (\eta \mu)^2K(K - 1)/2}{\mu}\\ 
    &= \eta\left(K - \eta H K(K - 1)/2+ (\eta H)^2K(K - 1)(K - 2)/6\right) - \eta \left(K - \eta \mu K(K - 1)/2\right) \\
    &= - \eta^2 (H - \mu)K(K - 1)/2+ (\eta H)^2K(K - 1)(K - 2)/6 \\
    & = -\eta^2 K(K - 1)\left(H - \mu - \eta H^2(K - 2)/6\right),
\end{align}
where we used the fact that for any integer $r \geq 2$ and $0 \leq x \leq 1$, 

$$1 - rx + \binom{r}{2}x^2 - \binom{r}{3}x^3 \leq (1 - x)^r \leq 1 - rx + \binom{r}{2}x^2.$$

Continuing, we have
\begin{equation}
b \leq -\eta^2 K(K - 1)\left(H - \mu - H/12\right) \leq -\eta^2 K^2\left(H - \mu - H/12\right),
\end{equation}

if $\eta H K \leq 1/2$ and $H(11/12) \geq \mu$.

If $\eta H K \geq 1/2$, then
\begin{equation}
    b = 1/H - (1 - \eta H)^K - 1/\mu + (1 - \eta \mu)^K \leq \frac{1 - e^{-1.1\eta H K/2}}{H} -  \frac{1 - e^{-\eta H K/2}}{H/2}.
\end{equation}
For $q \geq 1/2$, 
$$1 - e^{-q/2} \geq 0.52(1 - e^{-1.1q}),$$
so letting $q = \eta H k$, we have 
\begin{equation}
    b \leq \left((1 - 2(0.52))\frac{1 - e^{-\eta H k}}{H} \right) \leq -\frac{0.04(1 - e^{-1/2})}{H}
    \leq -0.015/H.
\end{equation}

\end{proof}

\subsection{Proof of Theorem~\ref{full_lb}: Lower bound of \fedavg convergence}\label{sec:pf_lb}
In this section, we prove our main \cref{full_lb}.
We use the following lemma.

\begin{lemma}\label{f2}
Let $f(x) = \mu x^2$, where $$\mu =  \frac{1}{D^2}\max\left(\min\left(\frac{\sigma D}{\sqrt{K}\sqrt{R}}, \frac{\sigma^{2/3}H^{1/3}D^{4/3}}{K^{1/3}R^{2/3}}\right), \frac{\zeta_*^{2/3}H^{1/3}D^{4/3}}{R^{2/3}}, \frac{HD^2}{KR}\right).$$
Suppose we run gradient descent starting at $x^0 = D/2$ with step size $\eta$. Then if $\eta \leq \frac{1}{\mu KR},$
after $KR$ iterations, we have $$f(x^{(KR)}) > \frac{1}{30}\max\left(\min\left(\frac{\sigma D}{\sqrt{K}\sqrt{R}}, \frac{\sigma^{2/3}H^{1/3}D^{4/3}}{K^{1/3}R^{2/3}}\right), \frac{\zeta_*^{2/3}H^{1/3}D^{4/3}}{R^{2/3}}, \frac{HD^2}{KR}\right).$$
\end{lemma}
\begin{proof}[Proof of \cref{f2}]
For any iteration $t$, we have
\begin{equation}
    x^{(t)} = \left(1 - 2\eta \mu\right)^tx^{(0)},
\end{equation}
such that 
\begin{equation}
\begin{split}
    f(x^{(KR)}) &= \left(1 - 2\eta \mu\right)^{2KR}f(x^0)\\
    &= \left(1 - 2\eta \mu\right)^{2KR} \frac{D^2\mu}{4}\\
    & \geq e^{-2\eta \mu K R}\frac{D^2\mu}{4}\\
    &> \frac{D^2}{30}\frac{1}{D^2}\max\left(\min\left(\frac{\sigma D}{\sqrt{K}\sqrt{R}}, \frac{\sigma^{2/3}H^{1/3}D^{4/3}}{K^{1/3}R^{2/3}}\right), \frac{\zeta_*^{2/3}H^{1/3}D^{4/3}}{R^{2/3}}, \frac{HD^2}{KR}\right)\\
    &=  \frac{1}{30}\max\left(\min\left(\frac{\sigma D}{\sqrt{K}\sqrt{R}}, \frac{\sigma^{2/3}H^{1/3}D^{4/3}}{K^{1/3}R^{2/3}}\right), \frac{\zeta_*^{2/3}H^{1/3}D^{4/3}}{R^{2/3}}, \frac{HD^2}{KR}\right).
\end{split}
\end{equation}
\end{proof}

\begin{proof}[Proof of Theorem~\ref{full_lb}]
We will use the following stochastic functions for our lower bound. Note that in the homogeneous case this reduces to the functions introduced in equation~\ref{lb:f}.
$$f(\x; (\xi_1, \xi_2, \xi_3)) = f^{(1)}(x_1; \xi_1) + f^{(2)}(x_2) + f^{(3)}(x_3; (\xi_2, \xi_3)),$$
where $$f^{(1)}(x; \xi) = \frac{L}{2}\psi(x) + x\xi, \quad f^{(2)}(x) = \mu x^2, \quad f^{(3)}(x; (\xi_2, \xi_3)) = \begin{cases}Hx^2 - x\xi_3 & \xi_2 = 1\\ \frac{H}{2}x^2 - x\xi_3& \xi_2 = 2\\ \end{cases},$$ and 
$$\xi_1 \sim \mathcal{N}(0, \sigma^2),\quad (\xi_2, \xi_3) = \begin{cases}
(1, 0) & \text{homogeneous case,}\\
(1, \zeta_*) & \text{for odd $m$,}\\
(2,-\zeta_*) & \text{for even $m$.}\\
\end{cases}$$
with $L := \frac{H}{12}$, and $$\mu := \frac{1}{D^2}\max\left(\min\left(\frac{\sigma D}{\sqrt{K}\sqrt{R}}, \frac{\sigma^{2/3}H^{1/3}D^{4/3}}{K^{1/3}R^{2/3}}\right), \frac{\zeta_*^{2/3}H^{1/3}D^{4/3}}{R^{2/3}}, \frac{HD^2}{KR}\right).$$

Suppose we run \fedavg starting at $\left(0, D/2, D/2, 0\right)$. Suppose the number of machines $M$ is even. We make the following three observations.

\begin{enumerate}
    \item If $\eta > \frac{1}{H}$, then $x_3$ diverges, and hence it suffices to consider the case when $\eta \leq \frac{2}{H}$.
    \item By Lemma~\ref{f2}, if $\eta \leq \frac{1}{\mu K R}$, we have $$f^{(2)}(x) > \frac{1}{30}\max\left(\min\left(\frac{\sigma D}{\sqrt{K}\sqrt{R}}, \frac{\sigma^{2/3}H^{1/3}D^{4/3}}{K^{1/3}R^{2/3}}\right), \frac{\zeta_*^{2/3}H^{1/3}D^{4/3}}{R^{2/3}}, \frac{HD^2}{KR}\right).$$ Furthermore, if $\frac{HD^2}{KR} \geq \max\left(\min\left(\frac{\sigma D}{\sqrt{K}\sqrt{R}}, \frac{\sigma^{2/3}H^{1/3}D^{4/3}}{K^{1/3}R^{2/3}}\right), \frac{\zeta_*^{2/3}H^{1/3}D^{4/3}}{R^{2/3}}\right)$, then $\frac{1}{\mu K R} \geq \frac{1}{H}$, and so $\eta \leq \frac{1}{H}$ implies that $$f^{(2)}(x) \geq \frac{1}{30}\max\left(\min\left(\frac{\sigma D}{\sqrt{K}\sqrt{R}}, \frac{\sigma^{2/3}H^{1/3}D^{4/3}}{K^{1/3}R^{2/3}}\right), \frac{\zeta_*^{2/3}H^{1/3}D^{4/3}}{R^{2/3}}, \frac{HD^2}{KR}\right).$$
    \item By Lemma~\ref{lem:steplb_1}, for some constant $c$, for $\eta \leq \frac{1}{6L} = \frac{2}{H}$, we have \begin{equation}
    \begin{split}
        \mathbb{E}[f^{(1)}(x^{(R, 0)}_1; \xi_1)] &\geq \frac{L}{4}\mathbb{E}[(x_1^{(R, 0)})^2] \\
        &\geq \frac{L}{4}\left(\mathbb{E}[x_1^{(R, 0)}]\right)^2 \\
        &\geq \frac{L}{4}\left(c\frac{\sqrt{\eta}\sigma}{\sqrt{L}}\min\left(R(\eta L K)^{3/2}, 1, (\eta L K)^{1/2}\right)\right)^2 \\
        &= \frac{c^2 \eta \sigma^2}{4}\min\left(R^2(\eta L K)^{3}, 1, \eta L K\right).
    \end{split}
    \end{equation}
    \item By Lemma~\ref{full_lemma:hetero}, if $\eta < 1/H$, for some constant $c_h$, we have 
    \begin{equation}
    \mathbb{E}_m\mathbb{E}_{\xi_2 \sim \mathcal{D}_m}[f^{(3)}(x^{(R, 0)}_3; \xi_2)] \geq \frac{3H}{4}\mathbb{E}[(x_3^{(R, 0)})^2]
    \geq \frac{3H}{4}\left(\mathbb{E}[x_3^{(R, 0)}]\right)^2
    \geq \frac{3c_h}{4H} \min(1, (\eta H K)^2, (\eta H K)^4R^2)\zeta_*^2.
    \end{equation}

\end{enumerate}

By items (1) and (2), its suffices to consider the case when $\eta > \frac{1}{\mu KR}$ and when $$\frac{HD^2}{KR} < \max\left(\min\left(\frac{\sigma D}{\sqrt{K}\sqrt{R}}, \frac{\sigma^{2/3}H^{1/3}D^{4/3}}{K^{1/3}R^{2/3}}\right), \frac{\zeta_*^{2/3}H^{1/3}D^{4/3}}{R^{2/3}}\right).$$  Otherwise, we immediately recover the theorem. 

We consider two cases depending of the relative order of $\min\left(\frac{\sigma D}{\sqrt{K}\sqrt{R}}, \frac{\sigma^{2/3}H^{1/3}D^{4/3}}{K^{1/3}R^{2/3}}\right)$ and $\frac{\zeta_*^{2/3}H^{1/3}D^{4/3}}{R^{2/3}}$. 

\textbf{Case 1: $\min\left(\frac{\sigma D}{\sqrt{K}\sqrt{R}}, \frac{\sigma^{2/3}H^{1/3}D^{4/3}}{K^{1/3}R^{2/3}}\right) \geq \frac{\zeta_*^{2/3}H^{1/3}D^{4/3}}{R^{2/3}}$.}

For $$\eta \geq \frac{1}{\mu K R} = \frac{D^2}{KR\min\left(\frac{\sigma D}{\sqrt{K}\sqrt{R}}, \frac{\sigma^{2/3}H^{1/3}D^{2/3}}{K^{1/3}R^{2/3}}\right)},$$ by the third item above, we have for some universal constant $C$:

\begin{equation}
\begin{split}
   \mathbb{E}\left[f^{(1)}(x^{(R, 0)}_1; \xi_1)\right] &\geq \frac{c^2D^2\sigma^2}{4KR\min\left(\frac{\sigma D}{\sqrt{K}\sqrt{R}}, \frac{\sigma^{2/3}H^{1/3}D^{4/3}}{K^{1/3}R^{2/3}}\right)}\min\left(R^2\left(\eta L K\right)^{3}, 1, \eta L K\right)\\
   &\geq \frac{c^2D^2\sigma^2}{4KR\min\left(\frac{\sigma D}{\sqrt{K}\sqrt{R}}, \frac{\sigma^{2/3}H^{1/3}D^{4/3}}{K^{1/3}R^{2/3}}\right)}\min\left(1, \frac{D^2L}{R\min\left(\frac{\sigma D}{\sqrt{K}\sqrt{R}}, \frac{\sigma^{2/3}H^{1/3}D^{4/3}}{K^{1/3}R^{2/3}}\right)}\right)\\
   &= \min\left(\max\left(\frac{c^2D\sigma}{4\sqrt{KR}}, \frac{c^2D^{2/3}\sigma^{4/3}H^{1/3}}{4K^{2/3}R^{1/3}}\right), \max\left(\frac{c^2D^2L}{4R}, \frac{c^2D^{4/3}\sigma^{2/3} L}{4K^{1/3}R^{2/3}H^{2/3}}\right)\right)\\
   &\geq C\min\left(\max\left(\frac{D\sigma}{\sqrt{KR}}, \frac{D^{4/3}\sigma^{4/3}H^{1/3}}{K^{2/3}R^{1/3}}\right), \max\left(\frac{D^2L}{R}, \frac{D^{4/3}\sigma^{2/3} H^{1/3}}{K^{1/3}R^{2/3}}\right)\right) \\
   &\geq C\min\left(\frac{D\sigma}{\sqrt{KR}}, \frac{D^{4/3}\sigma^{2/3} H^{1/3}}{K^{1/3}R^{2/3}}\right), 
\end{split}
\end{equation}

where we have used the fact that $$R\eta L K \geq R\eta \mu K \geq 1$$ to get rid of the first of the three terms in the minimum.

\textbf{Case 2: $\min\left(\frac{\sigma D}{\sqrt{K}\sqrt{R}}, \frac{\sigma^{2/3}H^{1/3}D^{4/3}}{K^{1/3}R^{2/3}}\right) < \frac{\zeta_*^{2/3}H^{1/3}D^{4/3}}{R^{2/3}}$.}

For $$\eta \geq \frac{1}{\mu K R} = \frac{D^2}{KR\frac{H^{1/3}\zeta_*^{2/3}D^{4/3}}{R^{2/3}}},$$ by the fourth item above, we have
\begin{equation}
\begin{split}
\mathbb{E}_m\mathbb{E}_{\xi_2 \sim \mathcal{D}_m}[f^{(3)}(x^{(R, 0)}_3; \xi_2)] &\geq \frac{3c_h}{4H} \min(1, (\eta H K)^2, (\eta H K)^4R^2)\zeta_*^2\\
 &=  \frac{3c_h\zeta_*^2}{4H} \min(1, (\eta H K)^2)\\
&\geq \frac{3c_h\zeta_*^2}{4H} \min\left(1, \left(\frac{D^2 HK}{KR\frac{\zeta_*^{2/3}H^{1/3}D^{4/3}}{R^{2/3}}}\right)^2\right)\\
&= \frac{3c_h\zeta_*^2}{4H} \min\left(1, \frac{D^{4/3} H^{4/3}}{R^{2/3}\zeta_*^{4/3}}\right)\\
&= \frac{3c_h}{4}\min\left(\frac{\zeta_*^2}{H}, \frac{\zeta_*^{2/3}H^{1/3}D^{4/3}}{R^{2/3}}\right).
\end{split}
\end{equation}
where we have used in the first equality the fact that $$R\eta H K \geq R\eta \mu K \geq 1$$ to get rid of the last of the three terms in the minimum.

Combining these two cases proves Theorem~\ref{full_lb}.

\end{proof}

%% file: appendices/ub_3o.tex
\ifdefined\conf\section{PROOF OF THEOREMS~\ref{thm:3o_convex} AND \ref{thm:non_convex}: UPPER BOUNDS OF FEDAVG UNDER 3RD-ORDER SMOOTHNESS}\else \section{Proof of \cref{thm:3o_convex,thm:non_convex}: Upper Bounds of \fedavg Under 3rd-order Smoothness}\fi

In this section, we state and prove the formal theorems on the upper bounds of \fedavg under third-order smoothness, including both convex and non-convex cases.
\begin{theorem}[Upper bound for \fedavg under third-order smoothness, complete version of \cref{thm:3o_convex}]
    \label{full_thm:ub:3o}
Suppose $f(\x; \xi)$ satisfies \cref{asm:main,asm:third_order}. Then for step size $\eta = \min\left\{\frac{1}{H}, \frac{\sqrt{BM}}{\sigma\sqrt{HRK}}, \frac{B^{1/5}}{K^{3/5}R^{1/5}Q^{2/5}\sigma^{4/5}}\right\}$ , \fedavg satisfies
\begin{align}
  \mathbb{E}\left[\left\|\nabla F(\hat{\x})\right\|^2\right]\leq 
  O\left(\frac{HB}{KR} 
  +  \frac{\sigma \sqrt{BH}}{\sqrt{MKR}} 
  +  \frac{B^{\frac{4}{5}}\sigma^{\frac{4}{5}}Q^{\frac{2}{5}}}{K^{\frac{2}{5}}R^{\frac{4}{5}}}\right).
\end{align}
where $\hat{\x} := \frac{1}{M}\sum_m{\x^{(r, k)}}$ for a uniformly random choice of $k \in [K]$, and $r \in [R]$, and $B = F(\x^{(0, 0)}) - \min_{\x} F(\x)$.
\end{theorem}

In the non-convex case, we prove our results under the following assumption, which is slightly weaker than Assumption~\ref{asm:bounded_grad} in the main body.
\begin{assumption}[Universal gradient bound of expected objective $F$]\label{asm:bounded_grad_2}
For any $\x$, $$\|\nabla F(\x)\| \leq G.$$
\end{assumption}

\begin{theorem}[Upper bound for \fedavg for non-convex objectives under third-order smoothness, complete version of \cref{thm:non_convex}]\label{full_thm:3o_non_convex}
Suppose $F(\x)$ is $H$-smooth and $f(\x; \xi)$ satisfies \cref{asm:third_order,asm:bounded_grad_2}. Then for step size $\min\left\{\frac{1}{H}, \frac{\sqrt{BM}}{\sigma\sqrt{HKR}}, \frac{B^{1/5}}{KR^{1/5}Q^{2/5}(\sigma + G)^{4/5}}\right\}$, we have
\begin{equation}
  \mathbb{E}\left[\left\|\nabla F(\hat{\x})\right\|^2\right]\leq
  \bigo \left(
  \frac{HB}{KR}
  + 
  \frac{\sigma\sqrt{BH}}{\sqrt{MKR}}
  +
  \frac{B^{\frac{4}{5}}(G + \sigma)^{\frac{4}{5}}Q^{\frac{2}{5}}}{R^{\frac{4}{5}}}
  \right),
\end{equation}
where $\hat{\x} := \frac{1}{M}\sum_m{\x_m^{(r, k)}}$ for a  random choice of $k \in [K]$, and $r \in [R]$, and $B := F(\x^{(0, 0)}) - \inf_{\x} F(\x)$.
\
\end{theorem}
\begin{remark}
Assumption~\ref{asm:bounded_grad} implies Assumption~\ref{asm:bounded_grad_2}. Note also that we must have $G \geq \sigma$ in Assumption~\ref{asm:bounded_grad}, since the second moment always exceeds the variance. It follows that the this Theorem~\ref{full_thm:3o_non_convex} implies \cref{thm:non_convex} in the main body.
\end{remark}

We prove both theorems using the following lemma.
Define the shadow iterate $\overline{\x}^{(r, k)} := \frac{1}{M}\sum_{i = 1}^M\x_m^{(r, k)}$. The following claim bounds the expected difference $F(\overline{\x}^{(r, k + 1)}) - F(\overline{\x}^{(r, k)})$. In what follows, all expectations are conditional on $\overline{\x}^{(r, 0)}$.
\begin{lemma}
    \label{lem:general_step}
For $\eta \leq \frac{1}{2H}$,
\begin{equation}
     \mathbb{E}\left[F(\overline{\x}^{(r, k + 1)})\right] \leq \mathbb{E}\left[F(\overline{\x}^{(r, k)})\right] - \frac{\eta}{2}\mathbb{E}\left[\left\|\nabla F(\overline{\x}^{(r, k)})\right\|^2\right] + \frac{\eta Q^2}{4}\mathbb{E}\left[\left(\frac{1}{M}\sum_m \left\|\x_m^{(r, k)} - \overline{\x}^{(r, k)}\right\|^2\right)^2\right] + \frac{H\eta^2 \sigma^2}{M}.
\end{equation}
\end{lemma}
\begin{proof}[Proof of \cref{lem:general_step}]
By $H$-smoothness, we have (recall $\g_m^{(r,k)}$ stands for the stochastic gradient of the $m$-th client taken at the $k$-th local step of the $r$-th round)
\begin{equation}
\begin{split}
    \mathbb{E}\left[F(\overline{\x}^{(r, k + 1)})\right] &= \mathbb{E}\left[F\left(\overline{\x}^{(r, k)} - 
    \eta\frac{1}{M}\sum_m\g_m^{(r, k)}\right) \right]\\
    &\leq \mathbb{E}\left[F(\overline{\x}^{(r, k)})\right] - \eta \mathbb{E}\left[\langle{\nabla F(\overline{\x}^{(r, k)}), \frac{1}{M}\sum_m\g_m^{(r, k))}}\rangle\right] + \frac{H\eta^2}{M^2}\mathbb{E}\left[\left\|\sum_m\g_m^{(r, k))}\right\|^2\right] \\
    &\leq \mathbb{E}\left[F(\overline{\x}^{(r, k)})\right] - \eta \mathbb{E}\left[\langle{\nabla F(\overline{\x}^{(r, k)}), \frac{1}{M}\sum_m\nabla F(\x_m^{(r, k)})}\rangle\right] + \frac{H\eta^2}{M^2}\mathbb{E}\left[\left\|\sum_m\nabla F(\x_m^{(r, k)})\right\|^2\right] + \frac{H\eta^2 \sigma^2}{M}.\\
\end{split}
\end{equation}
 Observe that for any real vectors, $\a$ and $\b$, we have $\langle{\a, \b}\rangle \geq \frac{1}{2}\|\a\|^2 + \frac{1}{2}\|\b\|^2 - \|\a - \b\|^2$.
Letting $\a := \nabla F(\overline{\x}^{(r, k)})$, and $\b := \frac{1}{M}\sum_m\nabla F(\x_m^{(r, k)})$, we obtain
\begin{equation}\label{eq:smooth_step}
\begin{split}
    \mathbb{E}\left[F(\overline{\x}^{(r, k + 1)})\right] &\leq \mathbb{E}\left[F(\overline{\x}^{(r, k)})\right] - \frac{\eta}{2}\mathbb{E}\left[\left\|\nabla F(\overline{\x}^{(r, k)})\right\|^2\right] - \frac{\eta}{2}\mathbb{E}\left[\left\|\frac{1}{M}\sum_m\nabla F(\x_m^{(r, k)})\right\|^2\right]\\
    &\quad + \eta\mathbb{E}\left[\left\|\nabla F(\overline{\x}^{(r, k)}) - \frac{1}{M}\sum_m\nabla F(\x_m^{(r, k)})\right\|^2\right] + \frac{H\eta^2}{M^2}\mathbb{E}\left[\left\|\sum_m\nabla F(\x_m^{(r, k)})\right\|^2\right] + \frac{H\eta^2 \sigma^2}{M}\\
    &\leq \mathbb{E}\left[F(\overline{\x}^{(r, k)})\right] - \frac{\eta}{2}\mathbb{E}\left[\left\|\nabla F(\overline{\x}^{(r, k)})\right\|^2\right] + \eta\mathbb{E}\left[\left\|\nabla F(\overline{\x}^{(r, k)}) - \frac{1}{M}\sum_m\nabla F(\x_m^{(r, k)})\right\|^2\right]  + \frac{H\eta^2 \sigma^2}{M},
\end{split}
\end{equation}
where the last inequality follows because $\eta \leq \frac{1}{H}$.

We will use third order smoothness to bound $\mathbb{E}\left[\left\|\nabla F(\overline{\x}^{(r, k)}) - \frac{1}{M}\sum_m\nabla F(\x_m^{(r, k)})\right\|^2\right]$.
By definition of $Q$-third order smoothness, we have
\begin{equation}
    \nabla F(\x_m^{(r, k)}) = \nabla F(\overline{\x}^{(r, k)}) + \nabla^2 F(\overline{\x}^{(r, k)})(\x_m^{(r, k)} - \overline{\x}^{(r, k)}) + \e_m,
\end{equation}
where $\|\e_m\|_2 \leq \frac{Q}{2}\left\|\x_m^{(r, k)} - \overline{\x}^{(r, k)}\right\|^2$. It follows that 
\begin{equation}
    \left\|\nabla F(\overline{\x}^{(r, k)}) - \frac{1}{M}\sum_m\nabla F(\x_m^{(r, k)})\right\|^2 = \left\|\frac{1}{M}\sum_m \e_m\right\|^2 \leq \left(\frac{1}{M}\sum_m \|\e_m\|_2\right)^2 \leq \frac{Q^2}{4}\left(\frac{1}{M}\sum_m \left\|\x_m^{(r, k)} - \overline{\x}^{(r, k)}\right\|^2\right)^2.
\end{equation}
Plugging this into equation~\ref{eq:smooth_step}, we obtain the claim:
\begin{equation}
     \mathbb{E}\left[F(\overline{\x}^{(r, k + 1)})\right] \leq \mathbb{E}\left[F(\overline{\x}^{(r, k)})\right] - \frac{\eta}{2}\mathbb{E}\left[\left\|\nabla F(\overline{\x}^{(r, k)})\right\|^2\right] + \frac{\eta Q^2}{4}\mathbb{E}\left[\left(\frac{1}{M}\sum_m \left\|\x_m^{(r, k)} - \overline{\x}^{(r, k)}\right\|^2\right)^2\right] + \frac{H\eta^2 \sigma^2}{M}.
\end{equation}
\end{proof}

The following lemma bounds the term $\varterm$ when $f$ is convex. 
\begin{lemma}\label{lemma:varterm}
With $\y := \mathbb{E}[\x^{(r, k)}_m]$, we have
\begin{equation}
    \varterm \leq \mathbb{E}[\|\x^{(r, k)}_m - \y\|^4].
\end{equation}
\end{lemma}
\begin{proof}[Proof of \cref{lemma:varterm}]
By definition of the mean, we have 
\begin{equation}
    \frac{1}{M}\sum_m \left\|\x_m^{(r, k)} - \overline{\x}^{(r, k)}\right\|^2 \leq \frac{1}{M}\sum_m \left\|\x_m^{(r, k)} - \y\right\|^2
\end{equation}
By the independence of the $\x_m^{(r, k)}$, we have
\begin{equation}
    \varterm = \frac{1}{M}\mathbb{E}\left[\|\x_m^{(r, k)} - \y\|^4\right] + \frac{M(M - 1)}{M^2}\left(\mathbb{E}\left[\|\x_m^{(r, k)} - \y\|^2\right]\right)^2.
\end{equation}
By Jensen's inequality, we can move the square in the second term inside the expectation, so this is less than $\mathbb{E}[\|\x^{(r, k)}_m - \y\|^4]$.
\end{proof}

In the convex case, we bound this term using a result from \cite{Yuan.Ma-NeurIPS20}.
\begin{lemma}[Proposition D.6 in \cite{Yuan.Ma-NeurIPS20}]
    \label{lem:highervariance}
Under Assumptions~\ref{asm:main} and~\ref{asm:third_order},
\begin{equation}
 \mathbb{E}\left[\left\|\x_m^{(r, k)} - \y\right\|^4\right] \leq 200 k^2\eta^4\sigma^4,
\end{equation}
\end{lemma}

In the non-convex, we bound the term in Lemma~\ref{lemma:varterm} in the following lemma.
\begin{lemma}\label{lem:ncvx_variance}
Under Assumptions~\ref{asm:third_order} and \ref{asm:bounded_grad_2}, with $\y := \mathbb{E}[\x^{(r, k)}_m]$, we have
\begin{equation}
   \mathbb{E}[\|\x^{(r, k)}_m - \y\|^4] \leq 8\eta^4(G + \sigma)^4k^4.
\end{equation}
\end{lemma}
\begin{proof}[Proof of \cref{lem:ncvx_variance}]
First note that $\mathbb{E}[\|\x^{(r, k)}_m - \z\|^4]$ is minimized over all $\z$ by the expectation $\y$, hence we have 
\begin{equation}
    \mathbb{E}[\|\x^{(r, k)}_m - \y\|^4] \leq \mathbb{E}[\|\x^{(r, k)}_m - \x^{(r, 0)}_m\|^4]
\end{equation}
We prove this by induction on $k$ with the following inductive hypothesis:
\begin{equation}
    \mathbb{E}[\|\x^{(r, k)}_m - \x^{(r, 0)}_m\|^4] \leq 8\eta^4(G + \sigma)^4k^4.
\end{equation}
Clearly this holds in the base case for $k = 0$. Suppose it holds for $k$ and we want to prove it for $k = 1$. Then we can expand
\begin{equation}
\begin{split}
   \mathbb{E}[\|\x^{(r, k + 1)}_m - \x^{(r, 0)}_m\|^4] &= \mathbb{E}[\|\x^{(r, k)}_m - \eta\g^{(r, k)}_m - \x^{(r, 0)}_m\|^4]\\
   &\leq \left(\frac{k + 1}{k}\right)^3\mathbb{E}[\|\x^{(r, k)}_m - \x^{(r, 0)}_m\|^4] + (k + 1)^3\mathbb{E}[\|\eta\g^{(r, k)}_m\|^4]\\
   &= \left(\frac{k + 1}{k}\right)^3\mathbb{E}[\|\x^{(r, k)}_m - \x^{(r, 0)}_m\|^4] + (k + 1)^3\mathbb{E}[\|\eta \nabla F(\x^{(r, k)}) + \eta(\g^{(r, k)}_m - \nabla F(\x^{(r, k)})\|^4]\\
   &\leq \left(\frac{k + 1}{k}\right)^3\mathbb{E}[\|\x^{(r, k)}_m - \x^{(r, 0)}_m\|^4] + 8(k + 1)^3\left(\eta^4\|\nabla F(\x^{(r, k)})\|^4 + \eta^4\mathbb{E}[\g^{(r, k)}_m - \nabla F(\x^{(r, k)})]\right)\\
   &\leq \left(\frac{k + 1}{k}\right)^3\mathbb{E}[\|\x^{(r, k)}_m - \x^{(r, 0)}_m\|^4] + 8(k + 1)^3\left(\eta^4G^4 + \eta^4\sigma^4\right)\\
   &\leq  \left(\frac{k + 1}{k}\right)^3\mathbb{E}[\|\x^{(r, k)}_m - \x^{(r, 0)}_m\|^4] + 8(k + 1)^3\eta^4(G + \sigma)^4\\
   &\leq  8\left(\frac{k + 1}{k}\right)^3\eta^4(G + \sigma)^4k^4 + 8(k + 1)^3\eta^4(G + \sigma)^4\\
   &= 8\eta^4(G + \sigma)^4(k + 1)^4.
\end{split}
\end{equation}
where the first inequality following from Jenson's inequality applied to the random variable $X$, where
\begin{equation}
    X = \begin{cases} \frac{k + 1}{k}(\x^{(r, k)}_m - \x^{(r, 0)}_m) & \text{ with probability } \frac{k}{k + 1}\\
    (k+1)\eta\g^{(r, k)}_m & \text{ with probability } \frac{1}{k + 1}.
    \end{cases}
\end{equation}
\end{proof}
\paragraph{Finishing the Proof of \cref{full_thm:ub:3o}.}
For the convex case, we now put Lemma~\ref{lem:general_step} together with the moment bounds in Lemmas~\ref{lemma:varterm} and \ref{lem:highervariance}. Telescoping,  we achieve the following for the convex case:
\begin{lemma}
Under Assumptions~\ref{asm:main} and \ref{asm:third_order}, for $\eta \leq \frac{1}{H}$, we have
\begin{equation}
    \frac{1}{KR} \sum_{r = 1}^R \sum_{k = 1}^K \mathbb{E}\left[\left\|\nabla F(\overline{\x}^{(r, k)})\right\|^2\right] \leq \frac{2(F(\x^{(0, 0)}) - F(\x^{\star}))}{\eta K R} + 50Q^2\eta^4\sigma^4K^2 + \frac{H\eta \sigma^2}{M}.
\end{equation}
\end{lemma}

Choosing \begin{equation}
    \eta = \min\left\{\frac{1}{H}, \frac{\sqrt{BM}}{\sigma\sqrt{HKR}}, \frac{B^{1/5}}{K^{3/5}R^{1/5}Q^{2/5}\sigma^{4/5}}\right\},
\end{equation}
we achieve from this lemma the convergence bound in Theorem~\ref{full_thm:ub:3o}.

\paragraph{Finishing the Proof of \cref{full_thm:3o_non_convex}.}
For the non-convex case, we now put Lemma~\ref{lem:general_step} together with the moment bounds in Lemmas~\ref{lemma:varterm} and \ref{lem:ncvx_variance}. Telescoping,  we achieve the following for the non-convex case:
\begin{lemma}
Under \cref{asm:third_order,asm:bounded_grad_2}, for $\eta \leq \frac{1}{H}$, we have
\begin{equation}
    \frac{1}{KR} \sum_{r = 1}^R \sum_{k = 1}^K \mathbb{E}\left[\left\|\nabla F(\overline{\x}^{(r, k)})\right\|^2\right] \leq \frac{2(F(\x^{(0, 0)}) - F(\x^{\star}))}{\eta K R} + 8Q^2\eta^4(G + \sigma)^4K^4 + \frac{H\eta \sigma^2}{M}.
\end{equation}
\end{lemma}
Choosing \begin{equation}
    \eta = \min\left\{\frac{1}{H}, \frac{\sqrt{BM}}{\sigma\sqrt{HKR}}, \frac{B^{1/5}}{KR^{1/5}Q^{2/5}(\sigma + G)^{4/5}}\right\},
\end{equation}
we achieve from this lemma the convergence bound in Theorem~\ref{full_thm:3o_non_convex}.

\subsection{Upper Bounds of \fedavg Under Second-Order Smoothness}
\label{sec:proof:fedavg:2o:ub}
For completeness, we also prove the following theorem, which can be obtained from \citet{yu2019parallel}.
\begin{theorem}[Upper bound for \fedavg for non-convex objectives under second-order smoothness]\label{thm:2o_non_convex}
Suppose $F(\x)$ is $H$-smooth and $f(\x; \xi)$ satisfies \cref{asm:bounded_grad_2}, and the variance of the gradients is bounded $\sigma^2$. Then for step size $\eta = \min\left\{\frac{1}{H}, \frac{\sqrt{BM}}{\sigma\sqrt{HKR}}, \frac{B^{\frac{1}{3}}}{KR^{\frac{1}{3}}H^\frac{2}{3}(G + \sigma)^{\frac{2}{3}}}\right\}$,, we have
\begin{equation}
  \mathbb{E}\left[\left\|\nabla F(\hat{\x})\right\|^2\right]\leq
  \bigo \left(
  \frac{HB}{KR}
  + 
  \frac{\sigma\sqrt{BH}}{\sqrt{MKR}}
  +
  \frac{B^{\frac{2}{3}}(G + \sigma)^{\frac{2}{3}}H^{\frac{2}{3}}}{R^{\frac{2}{3}}}
  \right),
\end{equation}
where $\hat{\x} := \frac{1}{M}\sum_m{\x_m^{(r, k)}}$ for a  random choice of $k \in [K]$, and $r \in [R]$, and $B := F(\x^{(0, 0)}) - \inf_{\x} F(\x)$.
\
\end{theorem}
\begin{proof}[Proof of \cref{sec:proof:fedavg:2o:ub}]
The proof is similar to the $Q$-third order smooth case. Following the proof of Lemma~\ref{lem:general_step} up to \cref{eq:smooth_step}, we obtain from $H$-smoothness:
\begin{equation}
    \mathbb{E}\left[F(\overline{\x}^{(r, k + 1)})\right] \leq \mathbb{E}\left[F(\overline{\x}^{(r, k)})\right] - \frac{\eta}{2}\mathbb{E}\left[\left\|\nabla F(\overline{\x}^{(r, k)})\right\|^2\right] + \eta\mathbb{E}\left[\left\|\nabla F(\overline{\x}^{(r, k)}) - \frac{1}{M}\sum_m\nabla F(\x_m^{(r, k)})\right\|^2\right]  + \frac{H\eta^2 \sigma^2}{M},
\end{equation}

Now by Jensen's inequality and $H$-smoothness, we have 
\begin{equation}
\begin{split}
    \mathbb{E}\left[\left\|\nabla F(\overline{\x}^{(r, k)}) - \frac{1}{M}\sum_m\nabla F(\x_m^{(r, k)})\right\|^2\right] &\leq \frac{1}{M}\sum_m\mathbb{E}\left[\left\|\nabla F(\overline{\x}^{(r, k)}) - \nabla F(\x_m^{(r, k)})\right\|^2\right] \\
    &\leq \frac{H^2}{M}\sum_m\mathbb{E}\left[\left\|\overline{\x}^{(r, k)} - \x_m^{(r, k)}\right\|^2\right]\\
    &\leq H^2\mathbb{E}\left[\left\|\x^{(r, 0)} - \x_1^{(r, k)}\right\|^2\right].
\end{split}
\end{equation}
The following claim shows bounds this expectation.
\begin{claim}
For any $k$,
\begin{equation}
   \mathbb{E}[\|\x^{(r, k)}_1 - \x^{(r, 0)}\|^2] \leq 2\eta^2(G + \sigma)^2k^2.
\end{equation}
\end{claim}
The proof of this claim is by induction, and it is nearly identical to the proof of Lemma~\ref{lem:ncvx_variance}.

Plugging in this claim, we have 
\begin{equation}
    \mathbb{E}\left[F(\overline{\x}^{(r, k + 1)})\right] \leq \mathbb{E}\left[F(\overline{\x}^{(r, k)})\right] - \frac{\eta}{2}\mathbb{E}\left[\left\|\nabla F(\overline{\x}^{(r, k)})\right\|^2\right] + 2\eta^3(G + \sigma)^2k^2  + \frac{H\eta^2 \sigma^2}{M},
\end{equation}
Telescoping, we obtain
\begin{equation}
    \frac{1}{KR} \sum_{r = 1}^R \sum_{k = 1}^K \mathbb{E}\left[\left\|\nabla F(\overline{\x}^{(r, k)})\right\|^2\right] \leq \frac{2(F(\x^{(0, 0)}) - F(\x^{\star}))}{\eta K R} + 2H^2\eta^2(G + \sigma)^2K^2 + \frac{H\eta \sigma^2}{M}.
\end{equation}
Choosing $\eta = \min\left\{\frac{1}{H}, \frac{\sqrt{BM}}{\sigma\sqrt{HKR}}, \frac{B^{\frac{1}{3}}}{KR^{\frac{1}{3}}H^\frac{2}{3}(G + \sigma)^{\frac{2}{3}}}\right\}$, we obtain
\begin{equation}
    \frac{1}{KR} \sum_{r = 1}^R \sum_{k = 1}^K \mathbb{E}\left[\left\|\nabla F(\overline{\x}^{(r, k)})\right\|^2\right] \leq \bigo \left(
  \frac{HB}{KR}
  + 
  \frac{\sigma\sqrt{BH}}{\sqrt{MKR}}
  +
  \frac{B^{\frac{2}{3}}(G + \sigma)^{\frac{2}{3}}H^{\frac{2}{3}}}{R^{\frac{2}{3}}}
  \right).
\end{equation}
This proves the theorem.
\end{proof}